\documentclass[12pt]{article}
\usepackage[margin=3cm]{geometry}
\usepackage{fourier}
\DeclareMathAlphabet{\mathcal}{OMS}{cmsy}{m}{n}
\usepackage{amsmath, amssymb, amsthm, mathrsfs,bm}
\usepackage{booktabs}
\usepackage{url}
\usepackage{graphicx}
\usepackage{enumitem}
\usepackage{algorithm}
\usepackage{algpseudocode}
\usepackage{comment}
\newtheorem{theorem}{Theorem}[section]
\theoremstyle{definition}

\usepackage{mathtools}
\graphicspath{{graph}}
\usepackage[dvipsnames]{xcolor}
\usepackage{psfrag}

\usepackage[numbers]{natbib} 

\usepackage{graphicx}

\usepackage{setspace}
\title{Senior Thesis CS}
\author{yunjin.tong.22 }
\date{January 2024}

\begin{document}

\begin{titlepage}
    \centering
    \vspace*{\stretch{1}}
    {\LARGE Data-Driven Computing Methods for Nonlinear Physics Systems with Geometric Constraints\par}
    \vspace*{\stretch{1}}
    {\Large Yunjin Tong\par}
    \vspace{\stretch{0.5}}
    {\Large Undergraduate Computer Science Thesis\par}
    \vspace{\stretch{0.5}}
    {\large Advised by\par}
    \vspace{\stretch{0.1}}
    {\Large Professor Bo Zhu\par}
    \vspace*{\stretch{1}}
    % Include the image here
    \includegraphics[width=0.2\linewidth]{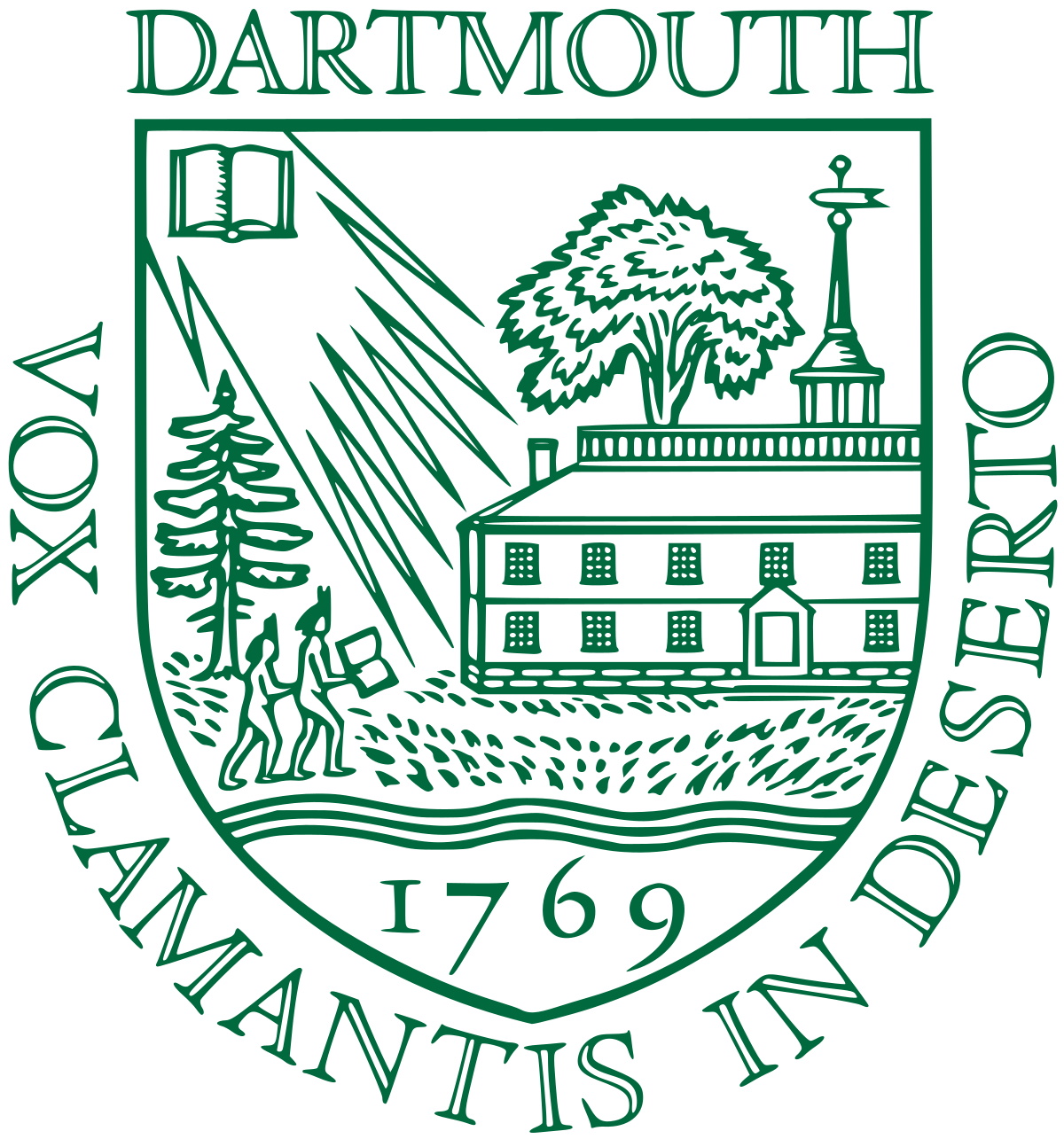}\par
    \vspace*{\stretch{0.1}}
    {\Large Dartmouth College\par}
    {\Large Hanover, New Hampshire\par}
    {\Large June, 2024\par}

    % \today
\end{titlepage}

\section*{Abstract}
In a landscape where scientific discovery is increasingly driven by data, the integration of machine learning (ML) with traditional scientific methodologies has emerged as a transformative approach. This paper introduces a novel, data-driven framework that synergizes physics-based priors with advanced ML techniques to address the computational and practical limitations inherent in first-principle-based methods and brute-force machine learning methods. Our framework showcases four algorithms, each embedding a specific physics-based prior tailored to a particular class of nonlinear systems, including separable and nonseparable Hamiltonian systems, hyperbolic partial differential equations, and incompressible fluid dynamics. The intrinsic incorporation of physical laws preserves the system's intrinsic symmetries and conservation laws, ensuring solutions are physically plausible and computationally efficient. The integration of these priors also enhances the expressive power of neural networks, enabling them to capture complex patterns typical in physical phenomena that conventional methods often miss. As a result, our models outperform existing data-driven techniques in terms of prediction accuracy, robustness, and predictive capability, particularly in recognizing features absent from the training set, despite relying on small datasets, short training periods, and small sample sizes.
\newpage

\section*{Acknowledgements}

I am deeply grateful for the generous financial support from Dartmouth Undergraduate Advising \& Research, which provided me with the Presidential Scholarship, Sophomore and Junior Research Scholarships, the Leave Term Research Grant, and support through the Women in Science Project. Additionally, I wish to acknowledge the Neukom Scholarship program from Neukom Institute for Computational Science.

I extend my heartfelt thanks to my supervisor, Professor Bo Zhu, for his unwavering support and profound inspiration. I am especially grateful for the opportunities he offered me as a first-year student, which opened my career as a researcher. I wish him a prolific and successful career at Georgia Tech.

I also recognize the invaluable assistance of the team at the Dartmouth Visual Computing Lab. Special thanks to Dr. Shiying Xiong, now an Assistant Professor at Zhejiang University, for his extensive help  in various aspects of my research. He is not only a super talented researchers in computational physics but also a remarkable coworker. Additional thanks go to Xingzhe He for his assistance with deep learning algorithms, among many others in the lab. Without their collective effort and collaboration, these works would not have been possible.

I also want to thank Professor Deeparnab Chakrabarty for his support and inspiration, as well as Professor Soroush Vosoughi and Professor Yaoqing Yang for being on my thesis committee. Additionally, I am grateful to the other professors and students who have taught and helped me at Dartmouth.

\newpage
\section*{List of Publications}

The following papers were published during the completion of my undergraduate studies and will be introduced in this thesis (listed in chronological order):
\begin{enumerate}
    \item Tong, Y., Xiong, S., He, X., Pan, G., \& Zhu, B. (2021). Symplectic Neural Networks in Taylor Series Form for Hamiltonian Systems. \textit{Journal of Computational Physics, 437}, 110325.
    \item Xiong, S., Tong, Y., He, X., Yang, S., Yang, C., \& Zhu, B. (2021). Nonseparable Symplectic Neural Networks. In \textit{Proceedings of the International Conference on Learning Representations}.
    \item Xiong, S., He, X., Tong, Y., Deng, Y., \& Zhu, B. (2023). Neural Vortex Method: from Finite Lagrangian Particles to Infinite Dimensional Eulerian Dynamics. \textit{Computers and Fluids, 258}, 105811.
    \item Tong, Y., Xiong, S., He, X., Yang, S., Wang, Z., Tao, R., Liu, R., \& Zhu, B. (2024). RoeNet: Predicting Discontinuity of Hyperbolic Systems from Continuous Data. \textit{International Journal for Numerical Methods in Engineering, 125}, e7406.
\end{enumerate}

\section*{Author's Contribution}

The work presented in this thesis is the product of scientific collaboration. Here I detail my specific contributions to each project.
For the project listed first \citep{tong2021symplectic}, my responsibilities 
include the initial generation of research ideas, implementation of the methodologies, conducting experiments, and writing the research paper. For the remaining three projects \citep{xiong2020nonseparable,xiong2023neural,tong2024roenet}, my primary roles involved conducting experiments and writing the respective papers. Additionally, for the fourth project \citep{tong2024roenet}, I was involved in idea generation and was responsible of the paper writing and revision. 
\newpage

\tableofcontents
% 1.引言
% 2.研究方法
% 2.1机器学习简介
% 介绍机器学习基本流程、lossfunction、数据集、训练器等
% 2.2数值格式简介
% 介绍ode的格式，如欧拉格式、龙格库塔、保辛格式、不可分保辛格式
% 介绍pde的格式，如roe solver，拉格朗日涡粒子方法
% 3.研究内容：嵌入数值格式的机器学习
% 3.1Taylor Net
% 3.2不可分那个辛格式
% 3.3RoeNet
% 3.4VortexNet
% 4.讨论
% 主要和PINN相比有什么优点和缺点
% 5.总结
% 5.1研究总结
% 5.2创新点
% 5.3未来展望
% 6.论文发表情况、获奖情况

\newpage
\section{Introduction}
From the time of Newton, two principal paradigms have shaped the methodologies of scientific research: the Keplerian paradigm, or the data-driven approach, and the Newtonian paradigm, or the first-principle-based approach \citep{weinan2021dawning}. The first-principle-based approach is fundamental and elegant, but the dilemma we often face is its practicality. There are many time-dependent problems in science, where the equations of motion are too complex for full solution, either because the equations are not certain or because the computational cost is too high. Additionally, for a dynamic system governed by some unknown mechanics, it is challenging to identify governing equations by directly observing the system's state, especially when such observation is partial and the sample data is sparse.

Now, the data-driven approach has become a very powerful tool with the advancement of statistical methods and machine learning (ML). This approach enables us to handle physical systems by statistically exploring their underlying structures. Data-driven approaches have proven their efficacy in uncovering the underlying governing equations of a variety of physical systems, ranging from fluid mechanics \citep{Brunton2020} and wave physics \citep{Hughes2019} to quantum physics \citep{Sellier2019}, thermodynamics \citep{hernandez2020}, and materials science \citep{Teicherta2019}. Moreover, various ML methods have significantly advanced the numerical simulation of complex and high-dimensional dynamical systems. These methods integrate learning paradigms with simulation infrastructures, enhancing the modeling of ordinary differential equations \citep{regazzoni2019}, linear and nonlinear partial differential equations \citep{tong2024roenet,Raissi2018}, high-dimensional partial differential equations \citep{Sirignano2018}, and inverse problems \citep{Raissi2019}, among others.

Despite these advancements, data-driven methods like neural networks, which exhibit remarkable generalization abilities across various fields, face significant challenges. These methods require large, clean datasets and depend heavily on complex, black-box network structures that are highly sensitive to input variations. Additionally, brute-force machine learning with conventional toolkits such as deep neural networks often struggles with the high dimensionality of input-output spaces, the cost of data acquisition, the production of physically implausible results, and the inability to handle extrapolation robustly. These factors make it difficult to predict long-term dynamical behaviors accurately.

To address these challenges, we introduce a novel, data-driven framework designed to make accurate, long-term predictions in a computationally efficient manner. The key innovation lies in incorporating physics-based priors into the learning algorithms so that the physics structure of the underlying system is intrinsically preserved. As a result, our models outperform other state-of-the-art data-driven methods in terms of prediction accuracy, robustness, and predictive capability, particularly in recognizing features absent from the training set. This superior performance is achieved despite relying on smaller datasets, shorter training periods, and limited sample sizes. At the same time, our models are significantly more computationally efficient than traditional first-principles-based methods, while achieving a similar level of accuracy.

This thesis details four algorithms we have developed over time, each incorporating a distinct physics-based prior relevant to a specific type of nonlinear system. The algorithm names, the associated physics priors, and the systems they address are as follows:
\begin{enumerate}
    \item Symplectic Taylor Neural Networks (Taylor-nets): The symplectic structure in separable Hamiltonian systems \citep{tong2021symplectic},
    \item Nonseparable Symplectic Neural Networks (NSSNNs): The symplectic structure in nonseparable Hamiltonian systems \citep{xiong2020nonseparable},
    \item Roe Neural Networks (RoeNet): Hyperbolic Conservation Law in hyperbolic partial differential equations (PDEs) \citep{tong2024roenet},
    \item Neural Vortex Method (NVM): Helmholtz’s Theorems in incompressible fluid dynamics \citep{xiong2023neural}.
\end{enumerate}

Overall, the key advantages and contributions of our methodologies are as follows:
\begin{itemize}
    \item \textbf{Preservation of Intrinsic Symmetries and Conservation Laws:} Our methodologies integrate physics-based priors within the learning algorithms, which significantly narrows the solution space. This reduction not only streamlines the computational demands but also preserves the mathematical symmetries and physical conservation laws inherent in the systems being modeled. Such an approach ensures that the generated solutions are not only efficient but also robust and aligned with physical reality, enhancing both the reliability and validity of the predictions.
    \item \textbf{Enhanced Expressive Power of Neural Networks:} By embedding physics-based structures into our models, we expand the network’s capacity to capture and reproduce complex patterns that are typical in solutions to physical phenomena. Conventional deep neural networks often struggle to identify such patterns when they are not represented within the training dataset. Our approach supports generalized solutions to PDEs and expands the solution space, allowing for a more comprehensive encapsulation of the potential physical behaviors, significantly improving the model's applicability and predictive accuracy.

\end{itemize}

The thesis will be organized into several key sections: an introductory section and a related work section that outline the research background; a methodology section that elaborates on the mathematical foundations, including an introduction to the supervised learning and numerical methods we used to develop our methodologies; an implementation section that details the algorithm design and proofs of four methodologies respectively; and a results section that summarizes the key implementation and experimental findings. The paper will conclude with a discussion on the implications of the results and potential avenues for future research.

\begin{table}[h]
  \caption{Overview of the key concepts related to four methodologies.}
  \centering
  \renewcommand{\arraystretch}{1.4} 
  \begin{tabular}{p{1.9cm} p{2.5cm} p{2.5cm} p{2.8cm} p{2.8cm}}

  \hline
   & \textbf{Taylor-nets} & \textbf{NSSNNs} & \textbf{RoeNet} & \textbf{NVM} \\
  \hline
  Physics System & Separable Hamiltonians & Nonseparable Hamiltonians & Hyperbolic PDEs & Incompressible Fluid Dynamics \\
    \hline
  \parbox[t]{2cm}{Prior\\ Embedded} & Symplectic Structure & Symplectic Structure & Hyperbolic Conservation Law & Helmholtz's Theorems \\
    \hline
  Solver & Separable Symplectic Integrator & Nonseparable Symplectic Integrator & Roe Solver & Lagrangian Vortex Method \\
  \hline
\parbox[t]{2cm}{Key\\ Advantages} & \multicolumn{2}{p{5.2cm}}{Accurately approximate the continuous-time evolution over a long term} & Predict future discontinuities with short-term continuous data & Reconstruct continuous vortex dynamics with a small number of vortex particles \\
  \hline
  \end{tabular}
  \label{tab:concepts}
\end{table}

Table \ref{tab:concepts} summarizes the key concepts related to the four methods, including the specific physics systems they model, the type of physical principles or priors they embed, the integrative techniques they employ, and the primary advantages each method offers. These comparative insights provide an at-a-glance understanding of the distinct capabilities and applications of each method. The details will be addressed comprehensively in Section \ref{num} and Section \ref{imp}.

\section{Related Work}

\paragraph{Neural Networks for Hamiltonian Systems.}
Greydanus et al. introduce Hamiltonian Neural Networks (HNNs) to preserve the Hamiltonian energy of systems by reformulating the loss function \citep{Greydanus2019}. Inspired by HNNs, a series of methods that intrinsically embed a symplectic integrator into the recurrent neural network architecture were proposed, including SRNN \citep{Chen2020}, and SSINN \citep{dipietro2020sparse}. Methods like HNN face two primary challenges: they require the temporal derivatives of system momentum and position to compute the loss function, which are hard to obtain from real-world systems, and they do not strictly preserve the symplectic structure as their symplectomorphism is governed by the loss function. Our model, Taylor-net \citep{tong2021symplectic}, addresses these limitations by integrating a solver into the network architecture to avoid the need for time derivatives and by embedding a symmetrical structure directly within the neural networks, rather than adjusting the loss function. Moreover, these methods have been extended, via combination with graph networks \citep{sanchez2019hamiltonian,battaglia2016interaction}, to address large-scale N-body problems where interactions are driven by forces between particle pairs.

While the above methods are all designed to solve separable Hamiltonian systems, Jin et al. proposed SympNet, which constructs symplectic mappings of system variables across neighboring time steps to handle both separable and nonseparable Hamiltonian systems \citep{Jin2020}. However, the parameter scalability of SympNet, growing quadratically with the system size $O(N^2)$, poses challenges for application to high-dimensional N-body problems. Our model, NSSNN, addresses these issues with a novel network architecture tailored for nonseparable systems, which significantly reduces the complexity of parameter scaling \citep{xiong2020nonseparable}. Additionally, Hamiltonian-based neural networks have been adapted for broader applications. Toth et al. developed the Hamiltonian Generative Network (HGN) to infer Hamiltonian dynamics from high-dimensional observations, such as image data \citep{Toth2020}. Furthermore, Zhong et al. introduced Symplectic ODE-Net (SymODEN), which incorporates an external control term into the standard Hamiltonian framework, enhancing the model's applicability to controlled dynamical systems \citep{Zhong2020Symplectic}.

\paragraph{Neural Networks for Discontinuous Functions.}

The use of deep learning networks to approximate discontinuous functions is well-supported theoretically, as highlighted in various studies on H\"{o}lder spaces \cite{Yarosky2016}, piecewise smooth functions \cite{Petersen2017}, linear estimators \cite{Imaizumi2019}, and highly adaptive, spatially anisotropic target functions \cite{Suzuki2018}. Building on these foundations, Physics-Informed Neural Networks (PINNs) were introduced by Raissi et al. as a data-driven approach to solving nonlinear problems \cite{Raissi2017}, leveraging the well-kown capability of deep neural networks to act as universal function approximators \cite{Hornik1989}. Among their key attributes, PINNs ensure the preservation of symmetry, invariance, and conservation principles that are inherent in the physical laws governing the observed data \cite{Zhang2019}. Michoski et al. demonstrated that PINNs could capture irregular solutions to PDEs without the need for any regularization \cite{Michoski2019}. Additionally, Mao et al. utilized PINNs to approximate solutions for high-speed flows by integrating the Euler equations with initial and boundary conditions into the loss function \cite{Mao2020}. However, while these studies demonstrate the robust capabilities of PINNs, they often do not address extrapolation beyond the training set, a critical aspect for ensuring the generalizability of the models to a wider range of scenarios.

\paragraph{Neural Networks for Fluid Dynamics.}
Recent advancements in fluid dynamics analysis have increasingly leveraged data-driven approaches powered by machine learning \citep{duraisamy2019turbulence,xie2018tempogan,chu2017data}. Recognizing the limitations in traditional brute-force machine learning methods, current research efforts are increasingly focused on integrating physical priors into learning algorithms, aiming to equip neural networks with a foundational understanding of physical laws, rather than approaching the data naively \citep{anderson1996comparison,crutchfield1987equations, daniels2015automated,wang2017physics,hammond2022machine,xu2022towards}. Significant efforts have been made to encode these physical constraints efficiently, such as incorporating the Navier-Stokes (NS) equations \citep{Raissi2018}, modeling incompressibility constraints \citep{mohan2020}, and mapping dynamics of wave phenomena onto recurrent neural network computations \citep{Hughes2019}. Moreover, understanding complex fluid dynamics through machine learning involves embedding the structure of partial differential equations (PDEs) within neural network architectures \citep{yang2019predictive,raissi2020hidden,belbute2020combining,lye2020deep,white2019neural,mohan2020embedding}. Ideally, these machine learning models designed to solve PDEs should be able to evolve the flow fields independently, obttaining initial-condition invariance without the need for a specific solver. However, the high dimensionality of the problems and insufficient supervisory data continue to pose significant challenges.

\section{Methodology}
\subsection{Supervised learning}
We used supervised learning for all of our models. Supervised learning is a subset of machine learning where an algorithm learns a function that maps an input to an output based on example input-output pairs. It infers a function from labeled training data consisting of a set of training examples. Each example is a pair consisting of an input object and a desired output value. The supervised learning algorithm analyzes the training data and produces an inferred function, which can be used for mapping new examples. Sequential steps involved in developing a supervised learning model, from determining the type of training dataset to evaluating the model's accuracy are:

\begin{enumerate}
    \item \textbf{Determine the Type of Training Dataset:} Identify whether the problem is a classification or regression to select the appropriate type of training dataset.
    
    \item \textbf{Collect/Gather the Labelled Training Data:} Assemble a dataset where each instance is tagged with the correct answer or outcome.
    
    \item \textbf{Split the Training Dataset:} Divide the dataset into three parts:
    \begin{itemize}
        \item Training dataset: used to train the model.
        \item Test dataset: used to test the model's predictions.
        \item Validation dataset: used to tune the model's hyperparameters.
    \end{itemize}
    
    \item \textbf{Determine the Input Features:} Select the features of the training dataset that contain sufficient information for the model to accurately predict the output.
    
    \item \textbf{Determine the Suitable Algorithm:} Choose an appropriate algorithm for the model based on the problem type.
    
    \item \textbf{Execute the Algorithm on the Training Dataset:} Train the model using the selected algorithm on the training dataset. Utilize the validation set to adjust control parameters as needed.
    
    \item \textbf{Evaluate the Model's Accuracy:} Test the model using the test dataset to assess its accuracy. A model that correctly predicts the output indicates high accuracy.
\end{enumerate}

\subsubsection{Optimizer}

In the context of neural networks, optimizers are crucial for minimizing the loss function, i.e., the difference between the actual and predicted outputs. One of the popular optimizers is the Adam optimizer \citep{kingma2014adam}, which combines the advantages of two other extensions of stochastic gradient descent, namely Adaptive Gradient Algorithm and Root Mean Square Propagation. The Adam optimizer's update equations are given by:

\begin{align*}
m_t &= \beta_1 m_{t-1} + (1 - \beta_1) g_t \\
v_t &= \beta_2 v_{t-1} + (1 - \beta_2) g_t^2 \\
\hat{m}_t &= \frac{m_t}{1 - \beta_1^t} \\
\hat{v}_t &= \frac{v_t}{1 - \beta_2^t} \\
\theta_{t+1} &= \theta_t - \frac{\alpha \hat{m}_t}{\sqrt{\hat{v}_t} + \epsilon}
\end{align*}

where $\theta$ represents the parameters of the model, $g_t$ is the gradient of the loss function with respect to the parameters at timestep $t$, $m_t$ and $v_t$ are estimates of the first and the second moments of the gradients, respectively. $\alpha$ is the learning rate, $\beta_1, \beta_2$, and $\epsilon$ are hyperparameters. 

\subsubsection{Loss Functions}

The choice of loss function is pivotal in guiding the training of the model towards its objective. In our methods, we use several common loss functions in supervised learning, including:

\paragraph{L1 Loss (Absolute Loss)} Defined as $L(y, \hat{y}) = \sum |y - \hat{y}|$, where $y$ is the true value and $\hat{y}$ is the predicted value. 

\paragraph{L2 Loss (Squared Loss)} Given by $L(y, \hat{y}) = \sum (y - \hat{y})^2$. This loss function is sensitive to outliers as it squares the differences, hence penalizing larger errors more.
\paragraph{Cross-Entropy Loss}

The Cross-Entropy Loss is widely used in classification tasks to measure the performance of a classification model whose output is a probability value between 0 and 1.
The Cross-Entropy Loss formula is given by:
\[
L(y, \hat{y}) = -\frac{1}{N} \sum_{i=1}^{N} \left[ y_i \log(\hat{y}_i) + (1 - y_i) \log(1 - \hat{y}_i) \right]
\]
where $L(y, \hat{y})$ is the loss function, $N$ is the number of samples, $y_i$ is the actual label of the $i$-th sample, and $\hat{y}_i$ is the predicted probability for the $i$-th sample.

For multi-class classification, the generalized formula is:
\[
L(y, \hat{y}) = -\frac{1}{N} \sum_{i=1}^{N} \sum_{c=1}^{M} y_{ic} \log(\hat{y}_{ic})
\]
where $M$ is the number of classes, $y_{ic}$ indicates whether class $c$ is the correct classification for observation $i$, and $\hat{y}_{ic}$ is the predicted probability that observation $i$ is of class $c$.

\paragraph{Focal Loss}
Focal Loss is an adapted version of Cross-Entropy Loss, which addresses the problem of class imbalance by focusing more on hard-to-classify examples. It is particularly useful in scenarios where there is a large class imbalance. The formula for Focal Loss is given by:
\[
L(y, \hat{y}) = -\alpha_t (1 - \hat{y}_t)^{\gamma} \log(\hat{y}_t)
\]
where $\alpha_t$ is a weighting factor for the class $t$ to counteract class imbalance, $\gamma$ is a focusing parameter that adjusts the rate at which easy examples are down-weighted, $\hat{y}_t$ is the predicted probability of the class with label $t$, and $(1 - \hat{y}_t)^{\gamma}$ reduces the loss for well-classified examples, putting more focus on hard, misclassified examples.
Focal Loss is particularly useful for training on datasets where some classes are much more frequent than others, helping to improve the robustness and performance of classification models in imbalanced datasets.

\subsubsection{Activation Functions}

Activation functions are non-linear functions applied to the output of a neuron in a neural network. They decide whether a neuron should be activated or not, helping the network learn complex patterns in the data.

\paragraph{ReLU}
One of the most popular activation functions is the Rectified Linear Unit (ReLU). It is defined as:
\[
f(x) = \max(0, x)
\]
where $x$ is the input to the neuron. ReLU is favored for its simplicity and efficiency, promoting faster convergence in training due to its linear, non-saturating form.

Next, we will outline the various models that were employed in the development of our model.

\subsubsection{Residual Networks (ResNets) and Residual Blocks (ResBlocks)}
ResNets are designed to enable training of very deep neural networks through the introduction of Residual Blocks (ResBlocks), which use skip connections or shortcuts to jump over some layers \citep{he2016resnet}. ResNets have been proven in numerous research studies to be a neural network architecture highly suitable for deep learning and computer vision. It offers distinctive advantages in mitigating problems like gradient vanishing during network training.

\paragraph{ResBlocks}
A Residual Block allows the gradient to flow through the network directly, without passing through non-linear activations, by using skip connections. This is mathematically represented as:
\[
\mathbf{h}_{\text{out}} = \mathcal{F}(\mathbf{h}_{\text{in}}, \{\theta_i\}) + \mathbf{h}_{\text{in}}
\]
where $\mathbf{h}_{\text{in}}$ is the input to the ResBlock, $\mathcal{F}(\mathbf{h}_{\text{in}}, \{\theta_i\})$ represents the residual mapping to be learned by layers of the ResBlock, and $\mathbf{h}_{\text{out}}$ is the output of the ResBlock. The addition operation is element-wise, allowing the network to learn identity mappings efficiently, which is crucial for training deep networks.

\subsubsection{Neural Ordinary Differential Equations (Neural ODEs) and the Adjoint Method} \label{NeuralODE}

Neural ODEs are a class of models that represent the continuous dynamics of hidden states using differential equations \citep{chen2018neural}. Unlike traditional neural networks that apply a discrete sequence of transformations, Neural ODEs model the derivative of the hidden state as a continuous transformation:
\begin{equation}
\frac{d\mathbf{h}(t)}{dt} = f(\mathbf{h}(t), t, \theta)\label{ODE}
\end{equation}
where $\mathbf{h}(t)$ is the hidden state at time $t$, $f$ is a neural network parameterized by $\theta$ defining the time derivative of the hidden state, making the model capable of learning continuous-time dynamics.

At the heart of the model is that
under the perspective of viewing a neural network as a dynamic system, we can treat the chain of residual blocks in a neural network as the solution of an ordinary differential equation (ODE) with the Euler method. Given a residual network that consists of sequence of transformations
\begin{equation}
  \bm{h}_{t+1}=\bm{h}_{t}+f(\bm{h}_{t}, \theta_t),
  \label{eq:neural1}
\end{equation}
the idea is to parameterize the continuous dynamics using an ODE specified by the neural network specified in~\eqref{ODE}.

In a Neural ODE framework, the evolution of the hidden state \( z \) is governed by an ODE parameterized by a neural network:
\begin{equation}
    \frac{d z(t)}{dt} = f(z(t), t, \theta),
\end{equation}
where \( t \) is time, \( \theta \) represents the parameters of the neural network, and \( f \) is a function approximated by the neural network defining the dynamics of \( z \).

To optimize Neural ODEs, the adjoint method is utilized, providing an efficient means for calculating gradients with respect to the parameters \( \theta \) during backpropagation \citep{pontryagin2018mathematical}. Rather than differentiating through the ODE solver, we solve the adjoint ODE defined as:
\begin{equation}
\frac{d \mathbf{a}(t)}{dt} = -\mathbf{a}(t)^\top \frac{\partial f}{\partial \mathbf{h}(t)},
\end{equation}
where \( \mathbf{a}(t) = \frac{dL}{d\mathbf{h}(t)} \) is the gradient of the loss function \( L \) with respect to the hidden state.

The gradient of the loss with respect to the parameters is then obtained by integrating:
\begin{equation}
\frac{dL}{d\theta} = \int_{t_1}^{t_0} \mathbf{a}(t)^\top \frac{\partial f}{\partial \theta} \, dt,
\end{equation}
over the interval from \( t_0 \) to \( t_1 \), the duration of the forward pass. The adjoint state \( \mathbf{a}(t) \) is initialized at the end of the forward pass and integrated backward in time to obtain the necessary gradients for parameter updates.

\subsection{Numerical Methods}\label{num}

First, we present four methods for solving ordinary differential equations (ODEs), which include the Euler method, Runge-Kutta method, Symplectic Integrator, and Non-separable Symplectic Integrator.

\subsubsection{Euler Method}
The Euler method represents one of the most straightforward numerical strategies for approximating solutions to ODEs. As a first-order numerical method, it provides an initial approach for solving initial value problems defined by \( \frac{d\bm{y}}{dt} = \bm{f}(t, \bm{y}) \) with the initial condition \( \bm{y}(t_0) = \bm{y}_0 \). Despite its simplicity, the Euler method is fundamental in the introduction to more sophisticated numerical methods for differential equations.

This method calculates the next state vector \( \bm{y} \) by proceeding in the direction of the derivative \( \bm{f}(t, \bm{y}) \), scaled by the timestep \( dt \). The updated state \( \bm{y} \) at time \( t + dt \) is given by:

\[
\bm{y}(t + dt) = \bm{y}(t) + \bm{f}(t, \bm{y}(t)) \cdot dt
\]

As a consequence of its first-order accuracy, the local truncation error for the Euler method is of the order \(O(dt^2)\), while the global error is of the order \(O(dt)\). This relatively large error suggests that while the Euler method can be beneficial for straightforward problems and educational purposes, it may not be the best choice for scenarios that demand high precision over extended durations.

\subsubsection{Runge-Kutta Method}
The Runge-Kutta methods are a prominent family of iterative techniques for the numerical resolution of ODEs. The fourth-order Runge-Kutta method, commonly referred to as RK4, is particularly renowned for its balance between computational efficiency and accuracy. This method is applied to approximate the solution of an initial value problem defined by the ODE \( \frac{d\bm{y}}{dt} = \bm{f}(t, \bm{y}) \) with the initial condition \( \bm{y}(t_0) = \bm{y}_0 \).

RK4 progresses the solution by computing a weighted average of four increments, where each increment evaluates the derivative \( \bm{f}(t, \bm{y}) \)at various points within the timestep \( dt \). The solution \( \bm{y} \) at a subsequent time \( t + dt \) is determined using the formula:

\[
\bm{y}(t + dt) = \bm{y}(t) + \frac{1}{6}(\bm{k}_1 + 2\bm{k}_2 + 2\bm{k}_3 + \bm{k}_4)
\]

with the increments given by:

\[
\begin{aligned}
\bm{k}_1 &= \bm{f}(t, \bm{y}) \cdot dt, \\
\bm{k}_2 &= \bm{f}\left(t + \frac{dt}{2}, \bm{y} + \frac{\bm{k}_1}{2}\right) \cdot dt, \\
\bm{k}_3 &= \bm{f}\left(t + \frac{dt}{2}, \bm{y} + \frac{\bm{k}_2}{2}\right) \cdot dt, \\
\bm{k}_4 &= \bm{f}(t + dt, \bm{y} + \bm{k}_3) \cdot dt.
\end{aligned}
\]

As a fourth-order method, the RK4 achieves a local truncation error of the order \(O(dt^5)\) and a global error of the order \(O(dt^4)\). This substantial accuracy renders the RK4 method highly effective for a broad spectrum of applications, offering an excellent trade-off between the computational demands and the precision of the solution.

\subsubsection{Separable Symplectic Integrator} \label{syminte}

Symplectic integrators are a class of numerical integration schemes specifically designed for simulating Hamiltonian systems.

A Hamiltonian system is characterized by $N$ pairs of canonical coordinates, denoted by generalized positions $\bm{q}=(q_1,q_2,\cdots,q_N)$ and generalized momenta $\bm{p}=(p_1,p_2,...p_N)$. The evolution of these coordinates over time is governed by Hamilton's equations, expressed as

\begin{equation}
\begin{dcases}
\frac{\textrm{d} \bm{q}}{\textrm{d} t} = \frac{\partial \mathcal {H}}{\partial \bm{p}} ,\\
\frac{\textrm{d} \bm{p}}{\textrm{d} t} =-\frac{\partial \mathcal {H}}{\partial \bm{q}},
\end{dcases}
\label{eq:Hamilton}
\end{equation}

with the initial condition

\begin{equation}
(\bm{q}(t_0),\bm{p}(t_0)) = (\bm q_0,\bm p_0).
\label{eq:intH}
\end{equation}

In a general setting, $\bm{q}=(q_1,q_2,\cdots,q_N)$ represents the positions and $\bm{p}=(p_1,p_2,...p_N)$ denotes their momentum. Function $\mathcal H = \mathcal H(\bm q, \bm p)$ is the Hamiltonian, which corresponds to the total energy of the system.

In a seperable Hamiltonian system, the Hamiltonian \(\mathcal {H}\) can be split into a kinetic energy part \(T(\bm{p})\) and a potential energy part \(V(\bm{q})\).
Consequently, the Hamiltonian  of a separable Hamiltonian system can e expressed in this form:

\begin{equation}
   \mathcal {H}(\bm{q}, \bm{p})= T(\bm{p}) + V(\bm{q}).
   \label{eq:Hpq}
 \end{equation}

The Symplectic integrators are distinguished by their ability to preserve the symplectic structure of phase space, an essential property for ensuring the long-term stability and accuracy of the simulation. By conserving quantities analogous to energy, these methods avoid the numerical dissipation typical of other numerical schemes, making them particularly well-suited for simulating dynamical systems over extended periods.

The specific Symplectic integrators we use is the fourth-order symplectic integrator, as described in the context of Hamiltonian systems and notably referenced in works by Forest and Ruth \citep{Forest1990} and Yoshida \cite{Yoshida1990}. It operates by applying a sequence of operations that integrate the system's equations of motion over a timestep \(dt\) while preserving the symplectic geometry of phase space. This preservation is crucial for accurately simulating the long-term behavior of Hamiltonian systems. The integrator is specifically designed for separable Hamiltonian systems shown in eqaution \eqref{eq:Hpq}. The fourth-order symplectic integrator updates the system's state over a time step \(dt\) by applying a sequence of operations that preserve the symplectic structure. The procedure is as follows:

1. Initialize with \((\bm{q}_0, \bm{p}_0)\) at \(t = t_0\).

2. For each time step \(dt\), update \((\bm{q}, \bm{p})\) through the following sequence of operations:
\begin{enumerate}
\item For each step \(j\) from 1 to 4, execute the following updates:
  \begin{itemize}
    \item Update momentum \(\bm{p}\) by a fraction of the time step: 
    \begin{equation}
       \bm{p} = \bm{p} - d_j \nabla V(\bm{q}) \cdot dt. 
    \end{equation}
    \item Update position \(\bm{q}\) by a fraction of the time step:
     \begin{equation}
       \bm{q} = \bm{q} + c_j \nabla T(\bm{p}) \cdot dt.
    \end{equation}

  \end{itemize}
\end{enumerate}

The coefficients \(c_j\) and \(d_j\) are chosen to eliminate lower-order error terms, ensuring fourth-order accuracy. These coefficients are typically defined as \citep{Forest1990,Yoshida1990,Candy1991}:

\begin{equation}
\begin{aligned}
c_{1}&=c_{4}={\frac {1}{2(2-2^{1/3})}},&c_{2}&=c_{3}={\frac {1-2^{1/3}}{2(2-2^{1/3})}},&\\
d_{1}&=d_{3}={\frac {1}{2-2^{1/3}}},&d_{2}&=-{\frac {2^{1/3}}{2-2^{1/3}}},& d_{4}=0.
\end{aligned}
\label{coeff}
\end{equation}

Repeat these steps for each time step \(dt\), iteratively advancing the system from \((\bm{q}_0, \bm{p}_0)\) at \(t_0\) to \((\bm{q}_n, \bm{p}_n)\) at \(t_0 + n \cdot dt\), where \(n\) is the number of time steps.

The fourth-order symplectic integrator is characterized by its fourth-order accuracy in the numerical simulation of Hamiltonian systems. This indicates that the local truncation error of the method is of the order \(O(dt^5)\), implying that the error introduced in a single timestep decreases as the fifth power of the timestep size. Consequently, the global error, or the cumulative error over a fixed interval of time, is of the order \(O(dt^4)\). Such high-order accuracy is especially beneficial for simulations requiring long-term stability and precision, as it permits the use of relatively large timestep sizes while maintaining a low overall numerical error.

\subsubsection{Nonseparable Symplectic Integrator}

Given a Hamiltonian system described in \eqref{eq:Hamilton} with initial condition \eqref{eq:intH}, we now consider a more genral case, an arbitrary separable and nonseparable Hamiltonian system.
In the original research of \cite{Tao2016} in computational physics, a generic, high-order, explicit and symplectic time integrator was proposed to solve (\ref{eq:Hamilton}) of an arbitrary separable and nonseparable Hamiltonian $\mathcal{H}$. This is implemented by considering an augmented Hamiltonian
\begin{equation}
    \overline{\mathcal{H}}(\bm q, \bm p, \bm x, \bm y) :=\mathcal{H}_A + \mathcal{H}_B + \omega \mathcal{H}_C
    \label{eq:overlineH}
\end{equation}
with
\begin{equation}
    \mathcal{H}_A = \mathcal{H} (\bm q, \bm y),~~\mathcal{H}_B = \mathcal{H} (\bm x, \bm p),~~\mathcal{H}_C = \frac{1}{2} \left(\|\bm q - \bm x\|_2^2+\|\bm p - \bm y\|_2^2\right)
\end{equation}
in an extended phase space with symplectic two form $\textrm{d} \bm q \wedge
\textrm{d} \bm p +\textrm{d} \bm x \wedge
\textrm{d} \bm y$, where $\omega$ is a constant that controls the
binding of the original system and the artificial restraint.

Notice that the Hamilton's equations for $\overline{\mathcal H}$
\begin{equation}
\begin{dcases}
\frac{\textrm{d} \bm{q}}{\textrm{d} t} = \frac{\partial \overline{\mathcal {H}}}{\partial \bm{p}} =\frac{\partial \mathcal {H}(\bm x, \bm p)}{\partial \bm{p}}+\omega (\bm p - \bm y),\\
\frac{\textrm{d} \bm{p}}{\textrm{d} t} =-\frac{\partial \overline{\mathcal {H}}}{\partial \bm{q}}=-\frac{\partial \mathcal {H}(\bm q, \bm y)}{\partial \bm{q}}-\omega (\bm q - \bm x),\\
\frac{\textrm{d} \bm{x}}{\textrm{d} t} = \frac{\partial \overline{\mathcal {H}}}{\partial \bm{y}} =\frac{\partial \mathcal {H}(\bm q, \bm y)}{\partial \bm{y}}-\omega (\bm p - \bm y),\\
\frac{\textrm{d} \bm{y}}{\textrm{d} t} =-\frac{\partial \overline{\mathcal {H}}}{\partial \bm{x}}=-\frac{\partial \mathcal {H}(\bm x, \bm p)}{\partial \bm{x}}+\omega (\bm q - \bm x),\\
\end{dcases}
\label{eq:overlineHamilton}
\end{equation}
with the initial condition $(\bm{q},\bm{p},\bm{x},\bm{y})|_{t=t_0} = (\bm q_0,\bm p_0,\bm q_0,\bm p_0)$ have the same exact solution as (\ref{eq:Hamilton}) in the sense that $(\bm{q},\bm{p},\bm{x},\bm{y}) = (\bm q,\bm p,\bm q,\bm p)$. Hence, we can get the solution of (\ref{eq:Hamilton}) by solving (\ref{eq:overlineHamilton}). The coefficient $\omega$ acts as a regularizer, which stabilizes the numerical results. 

It is possible to construct high-order symplectic integrators for $\overline{\mathcal {H}}$ with explicit updates. Denote respectively by
$\bm \phi_{1}^{\delta}(\bm q,\bm p,\bm x,\bm y)$, $\bm \phi_{2}^{\delta}(\bm q,\bm p,\bm x,\bm y)$, and $\bm \phi_{3}^{\delta}(\bm q,\bm p,\bm x,\bm y)$, which are the time-$\delta$ flow of $\mathcal {H_A}$, $\mathcal {H_B}$, $\omega \mathcal {H_C}$.
$\bm \phi_{1}^{\delta}$, $\bm \phi_{2}^{\delta}$, and $\bm \phi_{3}^{\delta}$ are given by

\begin{equation}
    \begin{bmatrix}
    \bm q\\
    \bm p-\delta [\partial \mathcal{H}_{\theta}(\bm q,\bm y)/\partial \bm q ]\\
    \bm x+\delta [\partial \mathcal{H}_{\theta}(\bm q,\bm y)/\partial \bm p ]\\
    \bm y
    \end{bmatrix},~
    \begin{bmatrix}
    \bm q+\delta [\partial \mathcal{H}_{\theta}(\bm x,\bm p)/\partial \bm p ]\\
    \bm p\\
    \bm x\\
    \bm y-\delta [\partial \mathcal{H}_{\theta}(\bm x,\bm p)/\partial \bm q ]
    \end{bmatrix},~\textrm{and}~\frac12 \begin{bmatrix}
    \begin{pmatrix}
    \bm q+\bm x\\
    \bm p+\bm y\\
    \end{pmatrix}+\bm R^\delta \begin{pmatrix}
    \bm q-\bm x\\
    \bm p-\bm y\\
    \end{pmatrix}\\
    \begin{pmatrix}
   \bm q+\bm x\\
    \bm p+\bm y\\
    \end{pmatrix}-\bm R^\delta \begin{pmatrix}
    \bm q-\bm x\\
    \bm p-\bm y\\
    \end{pmatrix}\\
    \end{bmatrix},
    \label{eq:phi}
\end{equation}
respectively. Here
\begin{equation}
    \bm R^\delta := \begin{bmatrix}
    \cos (2 \omega \delta) \bm I & \sin (2 \omega \delta) \bm I\\
    -\sin (2 \omega \delta) \bm I&\cos (2 \omega \delta) \bm I
    \end{bmatrix},~~\textrm{where} ~\bm I~ \textrm{is a identity matrix}.
\end{equation}

We remark that $\bm x$ and $\bm y$ are just auxiliary variables, which are theoretically equal to $\bm q$ and $\bm p$.  

Then we construct a numerical integrator that approximates $\overline{\mathcal {H}}$  by composing these
maps: it is well known that
\begin{equation}
    (\bm q_i,\bm p_i,\bm x_i,\bm y_i) = \bm \phi_1^{\textrm{d} t/2}\circ \bm \phi_2^{\textrm{d} t/2}\circ  \bm \phi_3^{\textrm{d} t}\circ \bm \phi_2^{\textrm{d} t/2}\circ \bm \phi_1^{\textrm{d} t/2}\circ (\bm q_{i-1},\bm p_{i-1},\bm x_{i-1},\bm y_{i-1})
\end{equation}

commonly named as Strang splitting, has a 3rd-order local error (thus a 2nd-order
method), and is a symmetric method.

Next, we introduce two methods for solving partial differential equations (PDEs), which are the Roe solver and Lagrangian Vortex Method.
\subsubsection{Roe Solver}

In continuum mechanics, a one-dimensional hyperbolic conservation law is a first-order quasilinear hyperbolic PDE
\begin{equation}
\frac{\partial \bm u}{\partial t} + \frac{\partial \bm F(\bm u)}{\partial x} = 0,
\label{eq:conserv_u}
\end{equation}
with an initial condition
\begin{equation}
    \bm u(t=t_0,x) = \bm u_0(x),
    \label{eq:initial_u}
\end{equation}
and a proper boundary condition.
Here the $N_c$-component vector $\bm u = [u^{(1)},u^{(2)},\cdots,u^{(N_c)}]^T$ is the conserved quantity, $t\in T =[t_0,t_1]$ denotes the time variable, $x$ denotes the spatial coordinate in a computational domain $\Omega$, and $\bm F=[F^{(1)},F^{(2)},\cdots,F^{(N_c)}]^T$ is a $N_c$-component flux function. The conservation laws described by \eqref{eq:conserv_u} are fundamental in continuum mechanics, such as mass conservation, momentum conservation, and energy conservation in fluid mechanics \cite{WuMaZhou2015}.

Equation \eqref{eq:conserv_u} can also be expressed in a weak form, which extends the class of admissible solutions to include discontinuous solutions. Specifically, by defining an arbitrary test function $\phi(t,x)$ that is continuously differentiable both in time and space with compact support,
and integrating \eqref{eq:conserv_u}$\times\phi$ in the space-time domain $T\times\Omega$, the weak form of \eqref{eq:conserv_u} is derived as
\begin{equation}
    \int_{T\times\Omega}\left(\bm u\frac{\partial \phi}{\partial t}+ \bm F \frac{\partial \phi}{\partial x}\right)  \textrm{d}t \textrm{d}x = 0.
    \label{eq:phiint}
\end{equation}
We remark that, with generalized Stokes theorem, all the partial derivatives of $\bm u$ and $\bm F$ in \eqref{eq:conserv_u} have been passed on to the test function $\phi$ in \eqref{eq:phiint}, which with the former hypothesis is sufficiently smooth to admit these derivatives \cite{Evans2010}.
In the absence of ambiguity, we refer to the solution of \eqref{eq:conserv_u} below as a weak solution that satisfies \eqref{eq:phiint}.

In addition, \eqref{eq:conserv_u} can be written in a high dimensional form
\begin{equation}
\frac{\partial \bm u}{\partial t} + \sum_{i=1}^{N_d}\frac{\partial \bm F_i(\bm u)}{\partial x_i}=\bm 0,
\label{eq:conserv_u_highdim}
\end{equation}
where $x_1,x_2,\cdots,x_{N_d}$ denote the $N_d$-dimensional spatial coordinates. Since every dimension in the second term of \eqref{eq:conserv_u_highdim}, namely $\partial \bm F_i(\bm u)/\partial x_i$, has the same form $\partial \bm F(\bm u)/\partial x$ as the second term of \eqref{eq:conserv_u}, \eqref{eq:conserv_u_highdim} can be easily solved if given the solution of \eqref{eq:conserv_u}. Thus, we will only discuss the numerical method to solve \eqref{eq:conserv_u}.

Philip L. Roe proposed an approximated Riemann solver based on the Godunov scheme \cite{Roe1981} that constructs an estimation for the intercell numerical flux of $\bm F$ in \eqref{eq:conserv_u} on the interface of two neighboring computational cells in a discretized space-time computational domain \cite{Roe1981}.
In particular, the Roe solver discretizes \eqref{eq:conserv_u} as
\begin{equation}
\bm{u}_j^{n+1}=\bm{u}_j^{n} - \lambda_r\left(\hat{\bm F}_{j+\frac12}^n - \hat{\bm F}_{j-\frac12}^n\right),
\label{eq:un1}
\end{equation}
where $\lambda_r = \Delta t/\Delta x$ is the ratio of the temporal step size $\Delta t$ to the spatial step size $\Delta x$, $j=1,...,N_g$ is the grid node index, and
\begin{equation}
    \hat{\bm F}_{j+\frac12}^n = \hat{\bm F}(\bm u_j^n,\bm u_{j+1}^n)
    \label{eq:Fj}
\end{equation}
with
\begin{equation}
    \hat{\bm F}(\bm u, \bm v) = \frac{1}{2}\left[\bm F(\bm u)+ \bm F(\bm v)-|\tilde{\bm A}(\bm u, \bm v)|(\bm v - \bm u)\right].
    \label{eq:Fuv}
\end{equation}
Here, Roe matrix $\tilde{\bm A}$ that is assumed constant between two cells and must obey the following Roe conditions:
\begin{enumerate}
    \item Matrix $\tilde{\bm A}$ is a diagonalizable matrix with real eigenvalues, i.e.,
    matrix $\tilde{\bm A}(\bm u ,\bm v)$ can be diagonalized as
    \begin{equation}
        \tilde{\bm A} = \bm {L}^{-1}\bm {\Lambda} \bm L
    \label{eq:diag}
    \end{equation}
    with an invertible matrix $\bm {L}$ and a diagonal matrix $\bm \Lambda = \textrm{diag}(\Lambda_1,\cdots,\Lambda_{N_c})$.
    \item Matrix $\tilde{\bm A}$ is consistent with an exact Jacobian, that is
    \begin{equation}
        \lim_{\bm u_j,\bm u_{j+1} \rightarrow \bm u} \tilde{\bm A}(\bm u_j, \bm u_{j+1}) = \frac{\partial \bm F (\bm u)}{ \partial \bm u}.
    \end{equation}
    \item Physical quantity $\bm u$ is conserved on the interface between two computational cells as
    \begin{equation}
        \bm F_{j+1}-\bm F_{j}=\tilde{\bm  A}(\bm u_{j+1}-\bm u_{j}).
        \label{eq:Roe3}
    \end{equation}
\end{enumerate}

We denote the absolute value of $\tilde{\bm A}(\bm u, \bm v)$ as
\begin{equation}
    |\tilde{\bm A}| = \bm {L}^{-1}|\bm {\Lambda}| \bm L, \label{eq:absA}
\end{equation}
where $|\bm \Lambda|=\textrm{diag}(|\Lambda_1|,\cdots,|\Lambda_{N_c}|)$ is the absolute value of $\bm \Lambda$. Substituting \eqref{eq:Fj}, \eqref{eq:Fuv} and \eqref{eq:absA} into \eqref{eq:un1} along with the third Roe condition \eqref{eq:Roe3} yields
\begin{equation}
\begin{aligned}
    \bm{u}_j^{n+1}=&\bm{u}_j^{n} - \frac{1}{2}\lambda_r
    [(\bm L^{n}_{j+\frac{1}{2}})^{-1}(\bm \Lambda_{j+\frac{1}{2}}^n-|\bm \Lambda_{j+\frac{1}{2}}^n|)\bm L_{j+\frac{1}{2}}^n(\bm u_{j+1}^n-\bm u_{j}^n)\\
    &+(\bm L^{n}_{j-\frac{1}{2}})^{-1}(\bm \Lambda_{j-\frac{1}{2}}^n+|\bm \Lambda_{j-\frac{1}{2}}^n|)\bm L_{j-\frac{1}{2}}^n(\bm u_{j}^n-\bm u_{j-1}^n)],
\end{aligned}
 \label{eq:roeeq}
\end{equation}
with
\begin{equation}
    \bm L_{j+\frac{1}{2}}^n = \bm L(\bm u_{j}^n,\bm u_{j+1}^n),~~~~ \bm \Lambda_{j+\frac{1}{2}}^n = \bm \Lambda(\bm u_{j}^n,\bm u_{j+1}^n).
    \label{eq:LLambda}
\end{equation}
Equation \eqref{eq:roeeq} serves as a template of evolution from $\bm{u}_j^{n}$ to $\bm{u}_j^{n+1}$ in Roe solver.

The key to design an effective Roe solver is to find the Roe matrix $\tilde{\bm A}$ that satisfies the three Roe conditions.
In order to construct a Roe matrix $\tilde{\bm A}$ in \eqref{eq:diag}, Roe solver utilizes an analytical approach to solve $\bm L$ and $\bm \Lambda$ based on $\bm F(\bm u)$. The Roe matrix is then plugged into \eqref{eq:roeeq} to ultimately solve for $\bm u$ in \eqref{eq:conserv_u}.
The Roe solver linearizes Riemann problems, and such linearization recognizes the problem's nonlinear jumps, while remaining computationally efficient. %Compared with the other Riemann solvers, e.g., Godunov's method \cite{godunov1959difference}, and HLLC solver \cite{toro1994restoration}, it performs with less cost and less dissipation.

\subsubsection{Lagrangian Vortex Method (LVM)}\label{LVM}

Given a fluid velocity field $\bm u (\bm x,t)$ with an incompressible constraint, its underlying dynamics can be described by the NS equations
\begin{equation}
  \begin{dcases}
  \frac{D \bm{u}}{D t}=-\frac{1}{\rho}\bm{\nabla} p+\nu\nabla^2 \bm u +\bm f,\\
  \bm{\nabla}\cdot \bm{u}=0,
  \end{dcases}
  \label{eq:uNS}
\end{equation}
where $t$ denotes the time, $D/D t = \partial / \partial t + \bm u \cdot \bm \nabla $ is the material derivative, $p$ is the pressure, $\nu$ is the kinematic viscosity, $\rho$ is the density, and $\bm f$ is the body accelerations (per unit mass) acting on the continuum, for example, gravity, inertial accelerations, electric field acceleration, and so on.

The alternative form of the NS equations could be obtained by defining the vorticity field $\bm{\omega} = \bm{\nabla \times u}$, which leads to the following vorticity dynamical equation
\begin{equation}
\begin{dcases}
  \frac{D \bm{\omega}}{D t}=(\bm{\omega}\cdot \bm{\nabla})\bm{u} + \nu \bm{\nabla}^{2}\bm{\omega} + \bm{\nabla} \times \bm f,\\
  \nabla^2 \bm \Psi = -\bm \omega,~~\bm u = \bm \nabla \times \bm \Psi,
\end{dcases}
  \label{eq:wNS}
\end{equation}
where $\bm \Psi$ is a vector potential whose curl is the velocity field.
Although this form does not seem to bring any simplification, the key illumination of doing this transformation stems the Helmholtz's theorems \citep{Helmholtz1858}, which states that the dynamics of the vorticity field can be described by vortex surfaces/lines, which are Lagrangian surfaces/lines flowing with the velocity field in inviscid flows \citep{Yang2010b,Xiong2017}.

The LVM discretizes the vorticity dynamical equation \eqref{eq:wNS} with $N$ particles resulting in a set of ODEs for the particle strengths $\bm \Gamma = \{\bm \Gamma_i|i=1,\cdots,N\}$ and the particle positions $\bm X =\{\bm X_i|i=1,\cdots,N\}$ as
\begin{equation}
  \begin{dcases}
  \frac{\textrm{d} \bm \Gamma_i}{\textrm{d} t}= \bm \gamma_i,\\
  \frac{\textrm{d} \bm X_i}{\textrm{d} t}= \bm u_i + \bm v_i.
  \end{dcases}
  \label{eq:vortex}
\end{equation}
Here, the particle strength $\bm \Gamma_i$ is the integral of $\bm \omega$ over the $i^{\textrm{th}}$ computational element,
$\bm u_i$ is the induced velocity calculated by BS law
\begin{equation}
    \bm u_i = \frac{1}{2(n_d-1)\pi}\sum_{j\neq i}^{N} \frac{\bm \Gamma_j\times (\bm X_i - \bm X_j)}{|\bm X_i - \bm X_j|^{n_d}+ \mathcal{R}^{n_d}},
    \label{eq:BS}
\end{equation}
where $n_d$ is the dimension of the flow field.
In addition, $\bm \gamma_i$ and $\bm v_i$ are the change rate of the particle strength and the drift velocity \citep{Hao2019}, respectively. 
To avoid singularities in the BS law, we introduce the numerical regularization parameter $\mathcal{R}$ in the LVM as $\mathcal{R}=0.1$. The effect of the regularization parameter on the dynamics of the flow evolution of the simulated vortex particles is rather small because of the large spacing between the vortex particles.

In a two-dimensional ideal fluid flow, i.e., a strictly inviscid barotropic flow with conservative body forces, the movements of Lagrangian particles with conserved vorticity strength are determined by the velocity field they create, thus allowing us to advance the simulation temporally \citep{Cottet2000}. However, in the real three-dimensional flow, under the action of vortex stretching, vortex distortion, viscous dissipation, external forces, etc., the Lagrangian advection of vortex particles and their strength need to be corrected by $\gamma_i$ and $\bm v_i$ in \eqref{eq:vortex}.

We remark that the NS equations can be accurately modeled by the LVM with a large number of computational elements and a reasonable discrete distribution. However, the implementation of the LVM faces a major challenge which is to model the right-hand sides (r.h.s.) of the set of ordinary differential equations based on the NS equations. Firstly, the assumption
that the vortices are point-like largely limits the use of the continuous BS law. Second, the drift velocity due to the external force cannot be obtained using the LVM without knowing the function of the external force. Even given the function, the LVM still fails to capture the drift velocity accurately in most cases \citep{Hao2019}. Finally, when two particles are close enough, the singularity of the discrete BS law leads to a significant numerical error. The above problems make the LVM inaccurate and inapplicable in solving the underlying fluid dynamics under many situations \citep{Cottet2000}.

\section{Implementation}\label{imp}
\subsection{Symplectic Taylor Neural Networks (Taylor-nets)}

\subsubsection{Symplectomorphism in Hamiltonian Mechanics}

Given a separable Hamiltonian system described by \eqref{eq:Hamilton}, \eqref{eq:intH}, and \eqref{eq:Hpq}.
Substituting \eqref{eq:Hpq} into \eqref{eq:Hamilton} yields

\begin{equation}
\begin{dcases}
\frac{\textrm{d} \bm{q}}{\textrm{d} t} = \frac{\partial T(\bm p)}{\partial \bm{p}},\\
\frac{\textrm{d} \bm{p}}{\textrm{d} t} =-\frac{\partial V(\bm q)}{\partial \bm{q}}.
\end{dcases}
\label{eq:HpqVT}
\end{equation}
This set of equations is fundamental in designing our neural networks. Our model will learn the r.h.s. of \eqref{eq:HpqVT} under the framework of ODE-net.

One of the important features of the time evolution of Hamilton's equations is symplectomorphism, which represents a transformation of phase space that is volume-preserving. In the setting of canonical coordinates, symplectomorphism means the transformation of the phase flow of a Hamiltonian system conserves the symplectic two-form

\begin{equation}
    \textrm{d} \bm p\wedge \textrm{d} \bm q \equiv \sum_{j=1}^{N}\left(\textrm{d}p_j\wedge \textrm{d}q_j\right),
    \label{eq:dpq}
\end{equation}
where $\wedge$ denotes the wedge product of two differential forms. Inspired by the symplectomorphism feature, we aim to construct a neural network architecture that intrinsically preserves Hamiltonian structure.

\subsubsection{A symmetric network in Taylor expansion form}
\label{subsec:taylor}

In order to learn the gradients of the Hamiltonian with respect to the generalized coordinates, we propose the following underpinning mechanism, which is a set of symmetric networks that learn the gradients of the Hamiltonian with respect to the generalized coordinates.

\begin{equation}
\begin{dcases}
\bm T_p(\bm p,\bm \theta_p) \rightarrow \frac{\partial T(\bm p)}{\partial \bm p},\\
\bm V_q(\bm q,\bm \theta_q) \rightarrow \frac{\partial V(\bm q)}{\partial \bm q},
\end{dcases}
\label{eq:TpVq}
\end{equation}
with parameters $(\bm\theta_p,\bm\theta_q)$ that are designed to learn the r.h.s. of \eqref{eq:HpqVT}, respectively. Here, the ``$\rightarrow$" represents our attempt to use the left-hand side (l.h.s) to learn the r.h.s.
Substituting \eqref{eq:TpVq} into \eqref{eq:HpqVT} yields

\begin{equation}
\begin{dcases}
\frac{\textrm{d} \bm{q}}{\textrm{d} t} = \bm T_p(\bm p,\bm \theta_p),\\
\frac{\textrm{d} \bm{p}}{\textrm{d} t} = -\bm V_q(\bm q,\bm \theta_q).
\end{dcases}
\label{eq:HpqVT1}
\end{equation}
Therefore, under the initial condition \eqref{eq:intH}, the trajectories of the canonical coordinates can be integrated as

\begin{equation}
\begin{dcases}
\bm q(t) = \bm q_0 + \int_{t_0}^{t} \bm T_p(\bm p,\bm \theta_p) \textrm{d}t,\\
\bm p(t) = \bm p_0 - \int_{t_0}^{t} \bm V_q(\bm q,\bm \theta_q) \textrm{d}t.
\end{dcases}
\label{eq:TVint}
\end{equation}

From \eqref{eq:TpVq}, we obtain

\begin{equation}
\begin{dcases}
\frac{\partial \bm T_p(\bm p,\bm \theta_p)}{\partial \bm p} \rightarrow \frac{\partial^2 T(\bm p)}{\partial \bm p^2},\\
\frac{\partial \bm V_q(\bm q,\bm \theta_q)}{\partial \bm q} \rightarrow \frac{\partial^2 V(\bm q)}{\partial \bm q^2}.
\end{dcases}
\label{eq:dTpVq}
\end{equation}
The r.h.s. of \eqref{eq:dTpVq} are the Hessian matrix of $T$ and $V$ respectively, so we can design $\bm T_p(\bm p,\bm \theta_p)$ and $\bm V_q(\bm q,\bm \theta_q)$ as symmetric mappings, that are

\begin{equation}
\frac{\partial \bm T_p(\bm p,\bm \theta_p)}{\partial \bm p} = \left[\frac{\partial \bm T_p(\bm p,\bm \theta_p)}{\partial \bm p}\right]^T,
\label{eq:partial_T}
\end{equation}
and

\begin{equation}
\frac{\partial \bm V_q(\bm q,\bm \theta_q)}{\partial \bm q} = \left[\frac{\partial \bm V_q(\bm q,\bm \theta_q)}{\partial \bm q}\right]^T.
\label{eq:partial_V}
\end{equation}

Due to the multiple nonlinear layers in the construction of traditional deep neural networks, it is impossible for these deep neural networks to fulfill \eqref{eq:partial_T} and \eqref{eq:partial_V}. Therefore, we can only use a three-layer network with the form of \emph{linear-activation-linear}, where the weights of the two linear layers are the transpose of each other, and in order to still maintain the expressive power of the networks, we construct symmetric nonlinear terms, as same as the terms of a Taylor polynomial, and combine them linearly. Specifically, we construct a symmetric network $\bm T_p(\bm p,\bm \theta_p)$ as

\begin{equation}
   \bm T_p(\bm p,\bm \theta_p) =\left( \sum_{i = 1}^{M}\bm A_i^T \circ f_i \circ \bm A_i - \bm B_i^T \circ f_i \circ \bm B_i\right) \circ \bm p + \bm b,
   \label{eq:Tp_Taylor}
\end{equation}
where `$\circ$' denotes the function composition, $\bm A_i$ and $\bm B_i$ are fully connected layers with size $N_h\times N$, $\bm b$ is a $N$ dimensional bias, $M$ is the number of terms in the Taylor series expansion, and $f_i$ is an element-wise function, representing the $i^{\textrm{th}}$ order term in the Taylor polynomial

\begin{equation}
f_i(x) = \frac{1}{i!}x^i.
\label{eq:Taylor_ex}
\end{equation}
Figure \ref{fig:Taylor_net} plots a schematic diagram of $\bm T_p(\bm p,\bm \theta_p)$ in Taylor-net. The input of $\bm T_p(\bm p,\bm \theta_p)$ is $\bm p$, and $\bm \theta_p = (\bm A_i$, $\bm B_i, \bm b )$. We construct a negative term $\bm B_i^T \circ f_i \circ \bm B_i$ following a positive term $\bm A_i^T \circ f_i \circ \bm A_i$, since two positive semidefinite matrices with opposite signs can represent any symmetric matrix.

\begin{figure}
  \centering
  \includegraphics[width=.7\linewidth]{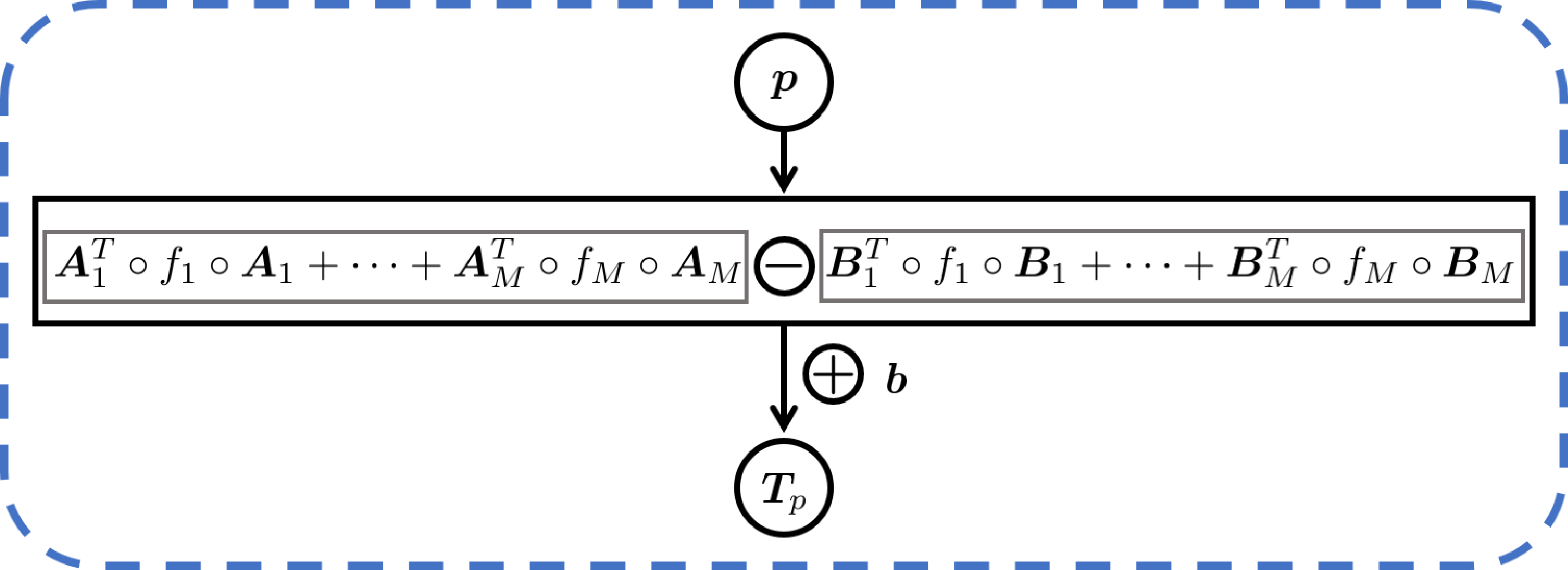}\\[0.5mm]
  \caption{The schematic diagram of $ \bm T_p(\bm p,\bm \theta_p)$ in Taylor-net. Source: \citep{tong2021symplectic}.}
  \label{fig:Taylor_net}
\end{figure}

To prove \eqref{eq:Tp_Taylor} is symmetric, that is it fulfills \eqref{eq:partial_T}, we introduce Theorem \ref{thm:Tp_sym}.
\begin{theorem}
The network \eqref{eq:Tp_Taylor} satisfies \eqref{eq:partial_T}.
\label{thm:Tp_sym}
\end{theorem}
\begin{proof}
From \eqref{eq:Tp_Taylor}, we have

\begin{equation}
    \frac{\partial \bm T_p(\bm p,\bm \theta_p)}{\partial \bm p} = \sum_{i = 1}^{M}\bm A_i^T \bm \Lambda_i^A \bm A_i - \bm B_i^T \bm \Lambda_i^B\bm B_i,
    \label{eq:Tp_d}
\end{equation}

with

\begin{equation}
    \Lambda_i^A = \textrm{diag}\left(\frac{\textrm{d}f}{\textrm{d}x}\Bigg|_{x= \bm A_i\circ \bm p}\right),
\end{equation}

and

\begin{equation}
    \Lambda_i^B = \textrm{diag}\left(\frac{\textrm{d}f}{\textrm{d}x}\Bigg|_{x= \bm B_i\circ \bm p}\right).
\end{equation}

It's easy to see that \eqref{eq:Tp_d} is a symmetric matrix that satisfies \eqref{eq:partial_T}.
\end{proof}
In fact, $\bm T_p(\bm p,\bm \theta_p)$ in \eqref{eq:partial_T} and $\bm V_q(\bm q,\bm \theta_q)$ in \eqref{eq:partial_V} satisfy the same property, so we construct $V_q$ with the similar form as

\begin{equation}
   \bm V_q(\bm q,\bm \theta_q) =\left( \sum_{i = 1}^{M}\bm C_i^T \circ f_i \circ \bm C_i - \bm D_i^T \circ f_i \circ \bm D_i\right) \circ \bm q + \bm d.
   \label{eq:Vq_Taylor}
\end{equation}
Here, $\bm C_i$, $\bm D_i$, and $\bm d$ have the same structure as \eqref{eq:Tp_Taylor}, and $(\bm C_i$, $\bm D_i, \bm d )= \bm \theta_q$.

\subsubsection{Symplectic Taylor neural networks}\label{subsec:Sym_ode}
Next, we substitute the constructed network \eqref{eq:Tp_Taylor} and \eqref{eq:Vq_Taylor} into \eqref{eq:TVint} to learn the Hamiltonian system \eqref{eq:HpqVT}.
%A popular idea that is developed recently in the area of deep learning is neural ordinary differential equation \cite{chen2018neural}.
We employ ODE-net \cite{chen2018neural} introduced in \ref{NeuralODE} as our computational infrastructure. 
Inspired by the idea of ODE-net, we design neural networks that can learn continuous time evolution. In Hamiltonian system \eqref{eq:HpqVT}, where the coordinates are integrated as \eqref{eq:TVint}, we can implement a time integrator to solve for $\bm p$ and $\bm q$. While ODE-net uses fourth-order Runge--Kutta method to make the neural networks structure-preserving, we need to implement an integrator that is symplectic. Therefore, we introduce Taylor-net, in which we design the symmetric Taylor series expansion and utilize the fourth-order symplectic integrator to construct neural networks that are symplectic to learn the gradients of the Hamiltonian with respect to the generalized coordinates and ultimately the temporal integral of a Hamiltonian system.

\begin{algorithm}
  \caption{Integrate \eqref{eq:TVint} by using the fourth-order symplectic integrator. Source: \citep{tong2021symplectic}.}
  \label{alg:int_net}
  \begin{algorithmic}
  \Require
  $\bm q_0,\bm p_0,t_0,t,\Delta t$,\\
  $\bm F_t^j$ in \eqref{eq:Ft} and $\bm F_k^j$ in \eqref{eq:Fk} with $j=1,2,3,4$;
  \Ensure
  $\bm q(t),\bm p(t)$
  \State
  $n = \textrm{floor}[(t-t_0)/\Delta t]$;
  \State
  for $i = 1,n$
  \State
  ~~~~$(\bm k_p^0,\bm k_q^0) = (\bm p_{i-1},\bm q_{i-1})$;
  \State
  ~~~~for $j = 1, 4$
  \State
  ~~~~~~~~$(\bm t_p^{j-1}, \bm t_q^{j-1}) =\bm F_t^j(\bm k_p^{j-1},\bm k_q^{j-1},\Delta t)$,
  \State
  ~~~~~~~~$(\bm k_p^j, \bm k_q^j) =\bm F_k^j(\bm t_p^{j-1},\bm t_q^{j-1},\Delta t)$,
  \State
  ~~~~end
  \State
  ~~~~$(\bm p_{i},\bm q_{i}) = (\bm k_p^4,\bm k_q^4)$;
  \State
  end
  \State
  $\bm q(t)=\bm q_{n},\bm p(t) = \bm p_{n}$.
  \end{algorithmic}
\end{algorithm}

For the constructed networks \eqref{eq:Tp_Taylor} and \eqref{eq:Vq_Taylor}, we integrate \eqref{eq:TVint} by using the fourth-order symplectic integrator introduced in \ref{syminte}. Specifically, we will have an input layer $(\bm q_0,\bm p_0)$ at $t = t_0$ and an output layer $(\bm q_n,\bm p_n)$ at $t = t_0 + n \textrm{d} t$. The recursive relations of $(\bm q_i,\bm p_i), i = 1,2,\cdots,n$, can be expressed by Algorithm \ref{alg:int_net}. The input function in Algorithm \ref{alg:int_net} are

\begin{equation}
\bm F_t^j(\bm p,\bm q,\textrm{d}t) =  \left(\bm p,\bm q+ c_j\bm T_p(\bm p,\bm \theta_p)\textrm{d}t\right),
\label{eq:Ft}
\end{equation}
and
\begin{equation}
\bm F_k^j(\bm p,\bm q,\textrm{d}t) =  \left(\bm p -  d_j \bm V_q(\bm q,\bm \theta_q)\textrm{d}t,\bm q\right),
\label{eq:Fk}
\end{equation}
with coefficients \eqref{coeff}.

Relationships \eqref{eq:Ft} and \eqref{eq:Fk} are obtained by replacing $\partial T(\bm p)/\partial \bm p$ and $ \partial V(\bm q)/\partial \bm q$ in the fourth-order symplectic integrator with deliberately designed neural networks $\bm T_p(\bm p,\bm \theta_p)$ and $\bm V_q(\bm q,\bm \theta_q)$, respectively. Figure \ref{fig:Sym_Taylor_net} plots a schematic diagram of Taylor-net which is described by Algorithm \ref{alg:int_net}. The input of Taylor-net is $(\bm q_0,\bm p_0)$, and the output is $(\bm q_n,\bm p_n)$. Taylor-net consists of $n$ iterations of fourth-order symplectic integrator. The input of the integrator is $(\bm q_{i-1},\bm p_{i-1})$, and the output is $(\bm q_{i},\bm p_{i})$. Within the integrator, the output of $\bm T_p$ is used to calculate $\bm q$, while the output of $\bm V_q$ is used to calculate $\bm p$, which is signified by the shoelace-like pattern in the diagram. The four intermediate variables $\bm t_p^0\cdots \bm t_p^4$ and $\bm k_q^0\cdots \bm k_q^4$ indicate that the scheme is fourth-order.

\begin{figure}
  \centering
  \includegraphics[width=0.93\linewidth]{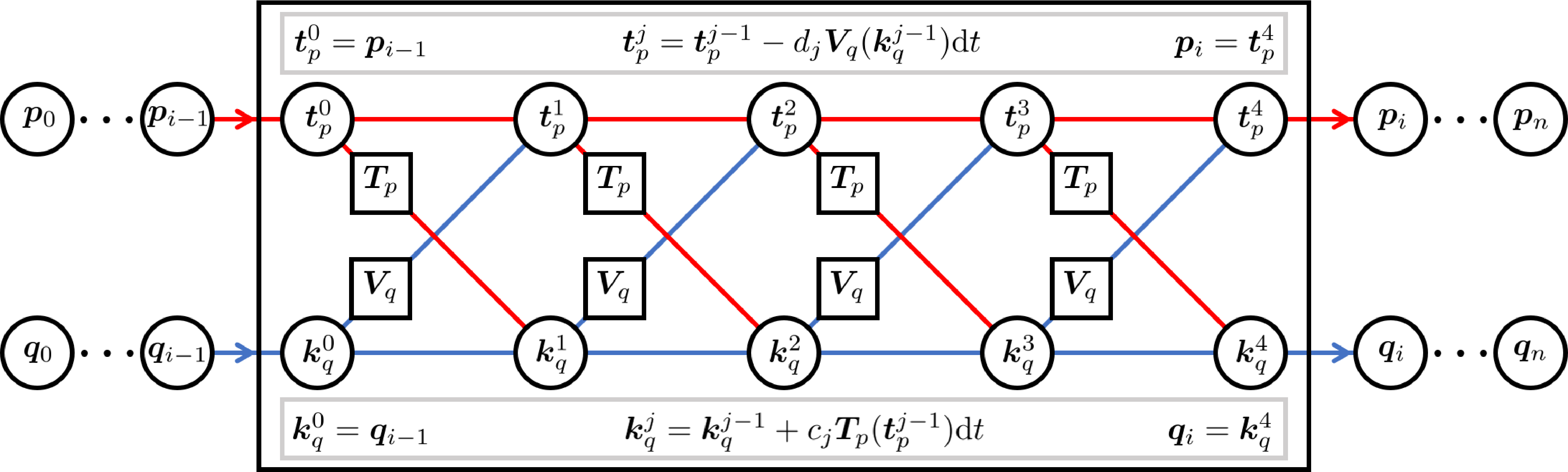}\\[0.5mm]
  \caption{The schematic diagram of Taylor-net. The input of Taylor-net is $(\bm q_0,\bm p_0)$, and the output is $(\bm q_n,\bm p_n)$. Taylor-net consists of $n$ iterations of fourth-order symplectic integrator. The input of the integrator is $(\bm q_{i-1},\bm p_{i-1})$, and the output is $(\bm q_{i},\bm p_{i})$. The four intermediate variables $\bm t_p^0\cdots \bm t_p^4$ and $\bm k_q^0\cdots \bm k_q^4$ show that the scheme is fourth-order. Source: \citep{tong2021symplectic}.}
  \label{fig:Sym_Taylor_net}
\end{figure}

By constructing the network $\bm T_p(\bm p,\bm \theta_p)$ in \eqref{eq:Tp_Taylor} that satisfies \eqref{eq:partial_T}, we show that Theorem \ref{thm:sym_Ft} holds, so the network \eqref{eq:Ft} preserves the symplectic structure of the system.

\begin{theorem}
For a given $\textrm{d}t$, the mapping $\bm F_t^j(:,:,\textrm{d}t):\mathbb{R}^{2N}\rightarrow \mathbb{R}^{2N}$ in \eqref{eq:Ft} is a symplectomorphism if and only if the Jacobian of $\bm T_p$ is a symmetric matrix, that is, it satisifies \eqref{eq:partial_T}.
 \label{thm:sym_Ft}
\end{theorem}

\begin{proof}
Let

\begin{equation}
(\bm t_p, \bm t_q) = \bm F_t^j(\bm k_p, \bm k_q,\textrm{d}t).
\end{equation}

From \eqref{eq:Ft}, we have

\begin{equation}
\begin{aligned}
&\textrm{d}\bm t_p \wedge \textrm{d}\bm t_q = \textrm{d}\bm k_p \wedge \textrm{d}\bm k_q + \\
&\frac{1}{2}\sum_{l,m=1}^N c_j\textrm{d}t  \left[\frac{\partial \bm T_p(\bm k_p,\bm \theta_p)}{\partial \bm k_p}\Bigg|_{l,m} - \frac{\partial \bm T_p(\bm k_p,\bm \theta_p)}{\partial \bm k_p}\Bigg|_{m,l}\right]\textrm{d}\bm k_p|_l \wedge \textrm{d}\bm k_q|_m.
\label{eq:dtk}
\end{aligned}
\end{equation}

Here $\bm A|_{l,m}$ refers to the entry in the $l$-th row and $m$-th column of a matrix $\bm A$, $\bm x|_l$ refers to the $l$-th component of vector $\bm x$. From \eqref{eq:dtk}, we know that $\textrm{d}\bm t_p \wedge \textrm{d}\bm t_q = \textrm{d}\bm k_p \wedge \textrm{d}\bm k_q$ is equivalent to

\begin{equation}
    \frac{\partial \bm T_p(\bm k_p,\bm \theta_p)}{\partial \bm k_p}\Bigg|_{l,m} - \frac{\partial \bm T_p(\bm k_p,\bm \theta_p)}{\partial \bm k_p}\Bigg|_{m,l}=0,\quad \forall l,m= 1,2,\cdots,N,
\end{equation}
which is \eqref{eq:partial_T}.
\end{proof}
Similar to Theorem \ref{thm:sym_Ft}, we can find the relationship between $\bm F_k^j$ and the Jacobian of $\bm V_q$. The proof of \ref{thm:sym_Fk} is omitted as it is similar to the proof of Theorem \ref{thm:sym_Ft}.
\begin{theorem}
For a given \textrm{d}t, the mapping $\bm F_k^j(:,:,\textrm{d}t):\mathbb{R}^{2N}\rightarrow \mathbb{R}^{2N}$ in \eqref{eq:Fk} is a symplectomorphism if and only if the Jacobian of $\bm V_q$ is a symmetric matrix, that is, it satisifies \eqref{eq:partial_V}.
\label{thm:sym_Fk}
\end{theorem}

Suppose that $\Phi_1$ and $\Phi_2$ are two symplectomorphisms. Then, it is easy to show that their composite map $\Phi_2\circ \Phi_1$ is also symplectomorphism due to the chain rule. Thus, the symplectomorphism of Algorithm \ref{alg:int_net} can be guaranteed by the Theorems \ref{thm:sym_Ft} and \ref{thm:sym_Fk}.

\subsection{Nonseparable Symplectic Neural Networks (NSSNNs)}
Our model aims to learn the dynamical evolution of $(\bm q, \bm p)$ in (\ref{eq:Hamilton}) by embedding (\ref{eq:overlineHamilton}) into the framework of NeuralODE \citep{chen2018neural}. We learn the nonseparable Hamiltonian dynamics (\ref{eq:Hamilton}) by constructing an augmented system (\ref{eq:overlineHamilton}), from which we can obtain the energy function $\mathcal {H}(\bm q,\bm p)$ by training the neural network $\mathcal {H}_{\theta} (\bm q, \bm p)$ with parameter $\bm \theta$ and calculate the gradient $ \bm \nabla \mathcal {H}_{\theta}(\bm q, \bm p)$ by taking the in-graph gradient.
For the constructed network $\mathcal{H}_{\theta}(\bm q, \bm p)$, we integrate (\ref{eq:overlineHamilton}) by using the second-order symplectic integrator \citep{Tao2016}. Specifically, we will have an input layer $(\bm q,\bm p, \bm x, \bm y) = (\bm q_0,\bm p_0,\bm q_0,\bm p_0)$ at $t = t_0$ and an output layer $(\bm q,\bm p, \bm x, \bm y) = (\bm q_n,\bm p_n,\bm x_n,\bm y_n)$ at $t = t_0 + n \textrm{d} t$.

\begin{algorithm}[t]
  \caption{Integrate (\ref{eq:overlineHamilton}) by using the second-order symplectic integrator. Source: \citep{xiong2020nonseparable}.}
  \label{alg:int_net2}
  \begin{algorithmic}
  \Require
  $\bm q_0,\bm p_0,t_0,t,\textrm{d}t$; $\;$ 
  $\bm \phi_1^{\delta}$, $\bm \phi_2^{\delta}$, and $\bm \phi_3^{\delta}$ in (\ref{eq:phi});
  \Ensure
  $(\hat q, \hat p, \hat x,\hat y)=(\bm q_{n}, \bm p_{n},\bm x_{n}, \bm y_{n})$
\State
  $(\bm q_0,\bm p_0,\bm x_0, \bm y_0)=(\bm q_0,\bm p_0,\bm q_0,\bm p_0)$ ; 
\State
  $n = \textrm{floor}[(t-t_0)/\textrm{d}t]$ ;
\State
~~~~for $i = 1\to n$
\State
 ~~~~~~~~$(\bm q_i,\bm p_i,\bm x_i,\bm y_i) = \bm \phi_1^{\textrm{d} t/2}\circ \bm \phi_2^{\textrm{d} t/2}\circ  \bm \phi_3^{\textrm{d} t}\circ \bm \phi_2^{\textrm{d} t/2}\circ \bm \phi_1^{\textrm{d} t/2}\circ (\bm q_{i-1},\bm p_{i-1},\bm x_{i-1},\bm y_{i-1})$;
\State
~~~~end
\end{algorithmic}
\end{algorithm}

The recursive relations of $(\bm q_i,\bm p_i,\bm x_i,\bm y_i), i = 1,2,\cdots,n$, can be expressed by Algorithm \ref{alg:int_net2}.
Figure \ref{fig:NSSNNnet}(a) shows the forward pass of NSSNN is composed of a forward pass through a
differentiable symplectic integrator as well as a backpropagation step through the model. Figure \ref{fig:NSSNNnet}(b) plots the schematic diagram of NSSNN.
For the constructed network $\mathcal{H}_{\theta}(\bm q, \bm p)$, we integrate (\ref{eq:overlineHamilton}) by using the second-order symplectic integrator \citep{Tao2016}. Specifically, The input layer of the integrator is $(\bm q,\bm p, \bm x, \bm y) = (\bm q_0,\bm p_0,\bm q_0,\bm p_0)$ at $t = t_0$ and the output layer is $(\bm q,\bm p, \bm x, \bm y) = (\bm q_n,\bm p_n,\bm x_n,\bm y_n)$ at $t = t_0 + n \textrm{d} t$. The recursive relations of $(\bm q_i,\bm p_i,\bm x_i,\bm y_i), i = 1,2,\cdots,n$, are expressed by Algorithm \ref{alg:int_net2}. Moreover, given~\eqref{eq:phi}, since $\bm x$ and $\bm y$ are theoretically equal to $\bm q$ and $\bm p$,  we can use the data set of $(\bm q, \bm p)$ to construct the data set containing variables $(\bm q, \bm p, \bm x, \bm y)$.

\begin{figure}
  \centering
  \includegraphics[width=0.95\linewidth]{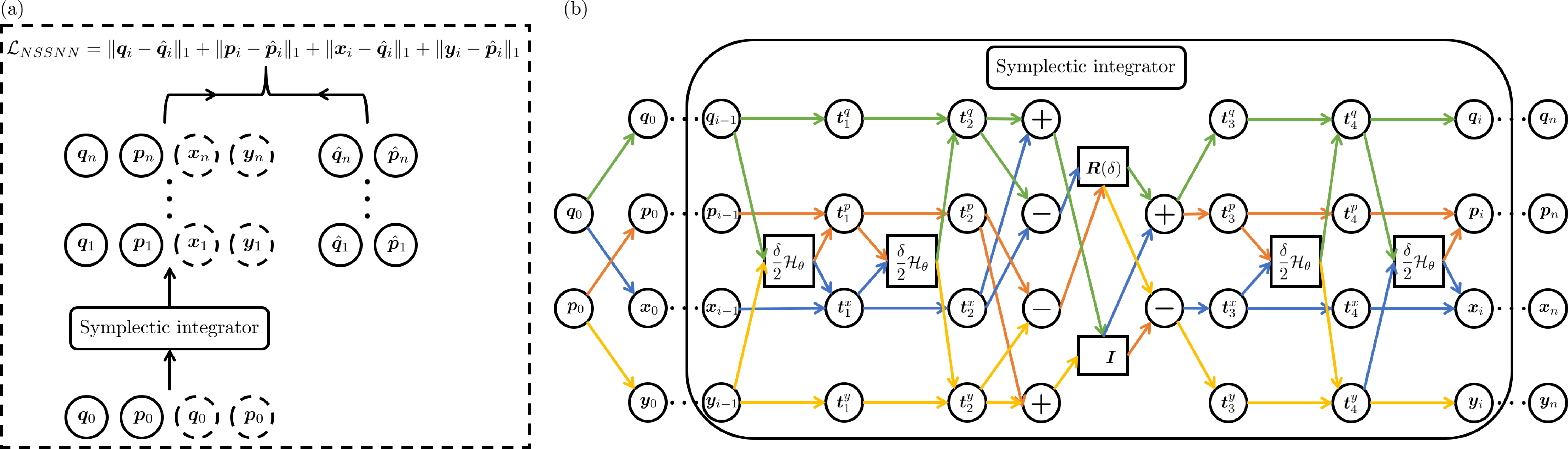}
  \caption{(a) The forward pass of an NSSNN is composed of a forward pass through a
differentiable symplectic integrator as well as a backpropagation step through the model. (b) The schematic diagram of NSSNN. Source: \citep{xiong2020nonseparable}.}
  \label{fig:NSSNNnet}
\end{figure}

In addition, by constructing the network $\mathcal{H}_{\theta}$, we show that Theorem \ref{thm:sym_phi} holds, so the networks $\bm \phi_1^{\delta},\bm \phi_2^{\delta}$, and $\bm \phi_3^{\delta}$ in (\ref{eq:phi}) preserve the symplectic structure of the system.
Suppose that $\Phi_1$ and $\Phi_2$ are two symplectomorphisms. Then, it is easy to show that their composite map $\Phi_2\circ \Phi_1$ is also symplectomorphism due to the chain rule. Thus, the symplectomorphism of Algorithm \ref{alg:int_net2} can be guaranteed by Theorem \ref{thm:sym_phi}.

\begin{theorem}
For a given $\delta$, the mapping $\bm \phi_1^{\delta}$, $\bm \phi_2^{\delta}$, and $\bm \phi_3^{\delta}$ in (\ref{eq:phi}) are symplectomorphisms.
\label{thm:sym_phi}
\end{theorem}

\begin{proof}
Let
\begin{equation}
(\bm t_j^q, \bm t_j^p, \bm t_j^x, \bm t_j^y) = \bm \phi_j^{\delta}(\bm q, \bm p , \bm x, \bm y),~~j= 1,2,3.
\end{equation}

From the first equation of (\ref{eq:phi}), we have
\begin{equation}
\begin{aligned}
&\textrm{d}\bm t_1^q \wedge \textrm{d}\bm t_1^p + \textrm{d}\bm t_1^x \wedge \textrm{d}\bm t_1^y\\
=& \textrm{d}\bm q \wedge \textrm{d}\left[\bm p - \delta \frac{\partial \mathcal{H}_{\theta}(\bm q, \bm y)}{\partial \bm q} \right]+ \textrm{d}\left[\bm x + \delta \frac{\partial \mathcal{H}_{\theta}(\bm q, \bm y)}{\partial \bm p} \right]\wedge \textrm{d}\bm y \\
=& \textrm{d}\bm q \wedge \textrm{d}\bm p+\textrm{d}\bm x \wedge \textrm{d}\bm y + \delta \left[\frac{\partial \mathcal{H}_{\theta}(\bm q, \bm y)}{\partial \bm q \partial \bm y}-\frac{\partial \mathcal{H}_{\theta}(\bm q, \bm y)}{\partial \bm y \partial \bm q}\right]\textrm{d}\bm q \wedge \textrm{d}\bm y\\
=&\textrm{d}\bm q \wedge \textrm{d}\bm p+\textrm{d}\bm x \wedge \textrm{d}\bm y.
\label{eq:dtk}
\end{aligned}
\end{equation}
Similarly, we can prove that $\textrm{d}\bm t_2^q \wedge \textrm{d}\bm t_2^p + \textrm{d}\bm t_2^x \wedge \textrm{d}\bm t_2^y =\textrm{d}\bm q \wedge \textrm{d}\bm p+\textrm{d}\bm x \wedge \textrm{d}\bm y$.  
In addition, from the third equation of (\ref{eq:phi}), we can directly deduce that $\textrm{d}\bm t_3^q \wedge \textrm{d}\bm t_3^p + \textrm{d}\bm t_3^x \wedge \textrm{d}\bm t_3^y =\textrm{d}\bm q \wedge \textrm{d}\bm p+\textrm{d}\bm x \wedge \textrm{d}\bm y$.
\end{proof}

\begin{figure}
%  	\vspace{-1.5\baselineskip}
	\centering
  \includegraphics[width=0.8\textwidth]{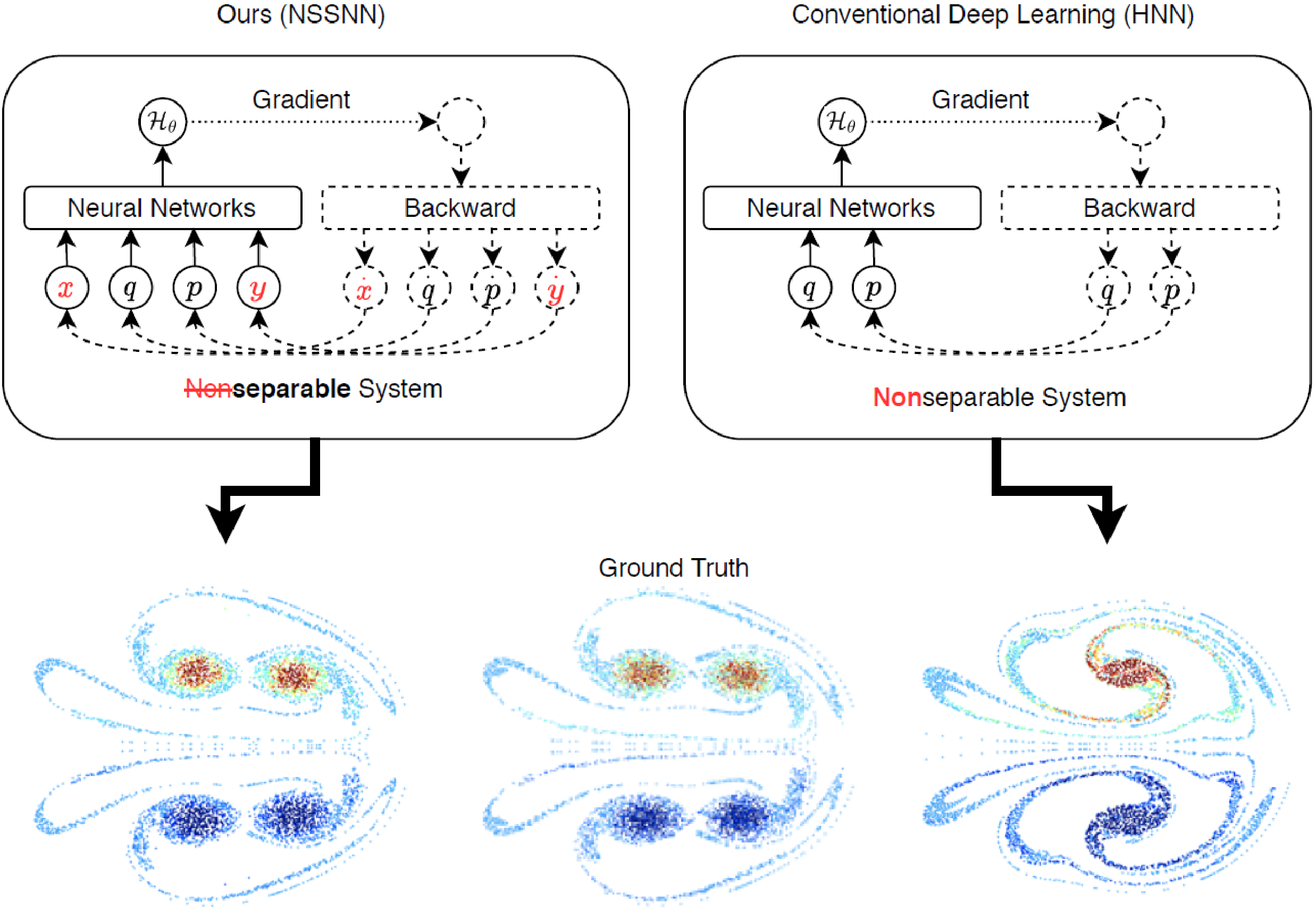}
  \caption{
  Comparison between NSSNN and HNN regarding the network design and prediction results of a vortex flow example. Source: \citep{xiong2020nonseparable}.
  %Learning the Hamiltonian form of a vortex system. We demonstrate 6000 vortex particles at $t=2$ with corresponding initial Leapfrop vortex. The vortices evolved using NSSNN are seperated nicely as the ground truth shows, while the vortices merge together using HNN.
  }.
  \label{fig:NSSNNvsHNN}
\end{figure}

We show a motivational example in Figure~\ref{fig:NSSNNvsHNN} by comparing our approach with a traditional HNN method \citep{Greydanus2019} regarding their structural designs and predicting abilities. We refer the readers to Section~\ref{subsec:vortex} for a detailed discussion.
As shown in Figure~\ref{fig:NSSNNvsHNN}, the vortices evolved using NSSNN are separated nicely as the ground truth, while the vortices merge together using HNN due to the failure of conserving the symplectic structure of a nonseparable system.
The conservative capability of NSSNN springs from our design of the auxiliary variables (red $x$ and $y$) which converts the original nonseparable system into a higher dimensional quasi-separable system where we can adopt a symplectic integrator.

\subsection{Roe Neural Networks (RoeNet)}

\begin{figure}[ht]
  \centering
  \includegraphics[width=0.4\textwidth]{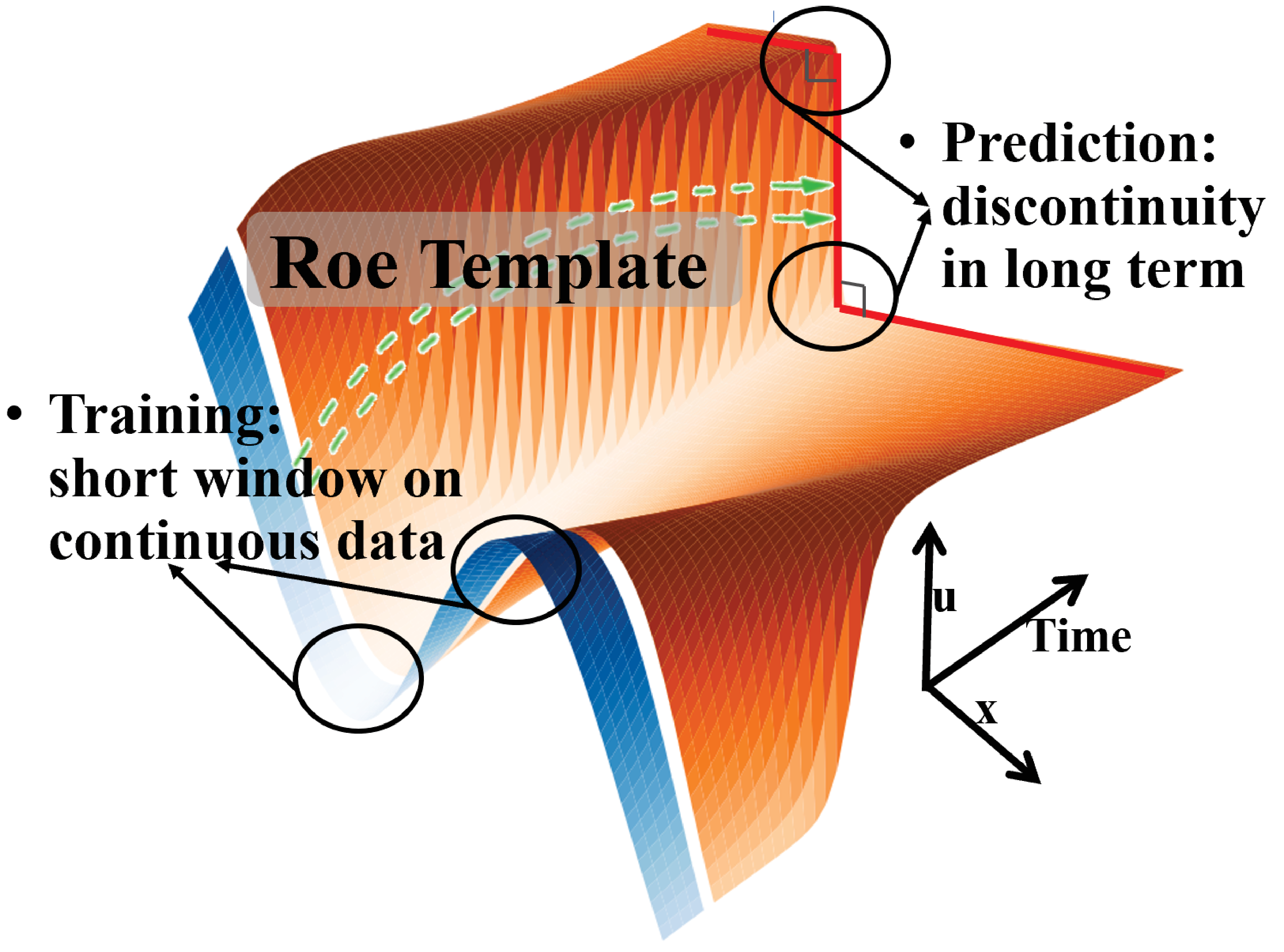}
  \caption{Schematic diagram of RoeNet to predict future discontinuity from smooth observations. The blue band shows the distribution of the training set with respect to time, and the training set does not necessarily contain discontinuous solutions to the equations. Meanwhile, the orange band represents the solutions predicted with RoeNet, which may contain discontinuous solutions. Source: \citep{tong2024roenet}.}
  \label{fig:overview}
\end{figure}

We introduce our design of the Roe template with pseudoinverse embedding, which accommodates the data processing and training over the entire learning pipeline. In particular, we present our basic ideas in Section \ref{sec:template}, a detailed description of our network architecture in Section \ref{sec:nn}.

\subsubsection{Roe template with Pseudoinverse Embedding}\label{sec:template}
Recall the one-dimensional hyperbolic conservation law described in \eqref{eq:conserv_u}, without a given $\bm F$, we learn the weak solution of \eqref{eq:conserv_u} using a neural network that incorporates the framework of a Roe solver.
For time integration of $\bm u$ in \eqref{eq:roeeq}, we need to construct the matrix functions $\bm L$ and $\bm \Lambda$. Since learning a tiny parameter space is impractical, using neural networks to approximate $\bm {L}$ and $\bm {\Lambda}$ directly in \eqref{eq:LLambda} is ineffective given that the number of learnable parameters is limited by the number of components $N_c$ of $\bm u$. To enhance the expressiveness of our model, we use neural network $\bm {L}_{\theta}$ and $\bm{\Lambda}_{\phi}$ to replace
$\bm {L}$ and $\bm {\Lambda}$ in \eqref{eq:LLambda} respectively. Similar to \eqref{eq:LLambda}, the inputs to $\bm {L}_{\theta}$ and $\bm{\Lambda}_{\phi}$ remains the same as $(\bm u_{j}^n,\bm u_{j+1}^n)$. However, the outputs of $\bm {L}_{\theta}$ and $\bm{\Lambda}_{\phi}$ are now a $N_h\times N_c$ matrix and a $N_h\times N_h$ diagonal matrix respectively, where the positive integer $N_h$ is a hidden dimension. Furthermore, we introduce the concept of pseudoinverses by replacing $\bm {L}^{-1}$ with
\begin{equation}
\bm L_{\theta}^{+} = (\bm L_{\theta}^{T} \bm L_{\theta})^{-1}\bm L_{\theta}^T.
\label{eq:pseudoinverse}
\end{equation}
Here, the transpose and inverse operations are applied to the output matrix, that is
\begin{equation}
\bm L_{\theta}^{+}(\bm u_{j}^n,\bm u_{j+1}^n) = [\bm L_{\theta}(\bm u_{j}^n,\bm u_{j+1}^n)^{T} \bm L_{\theta}(\bm u_{j}^n,\bm u_{j+1}^n)]^{-1}\bm L_{\theta}^T(\bm u_{j}^n,\bm u_{j+1}^n).
\end{equation}

Substituting $\bm {L}_{\theta}$, $\bm{\Lambda}_{\phi}$, and (\ref{eq:pseudoinverse}) into (\ref{eq:roeeq}) and \eqref{eq:LLambda} yields
\begin{equation}
\begin{aligned}
    \bm{u}_j^{n+1}=&\bm{u}_j^{n} -\frac12  \lambda_r (\bm L^{n}_{j+\frac{1}{2},\theta})^+(\bm \Lambda_{j+\frac{1}{2},\phi}^n-|\bm \Lambda_{j+\frac{1}{2},\phi}^n|)\bm L_{j+\frac{1}{2},\theta}^n(\bm u_{j+1}^n-\bm u_{j}^n)\\
    & - \frac12 \lambda_r (\bm L^{n}_{j-\frac{1}{2},\theta})^+(\bm \Lambda_{j-\frac{1}{2},\phi}^n+|\bm \Lambda_{j-\frac{1}{2},\phi}^n|)\bm L_{j-\frac{1}{2},\theta}^n(\bm u_{j}^n-\bm u_{j-1}^n),
\end{aligned}
 \label{eq:roenet}
\end{equation}
with
\begin{equation}
    \bm L_{j+\frac{1}{2},\theta}^n = \bm L_{\theta}(\bm u_{j}^n,\bm u_{j+1}^n),~~~~ \bm \Lambda_{j+\frac{1}{2},\phi}^n = \bm \Lambda_{\phi}(\bm u_{j}^n,\bm u_{j+1}^n).
    \label{eq:LL}
\end{equation}
Equation (\ref{eq:roenet})
serves as our template to evolve the system's states from $\bm{u}_j^{n}$ to $\bm{u}_j^{n+1}$ in RoeNet.

Figure \ref{fig:overview} presents a schematic diagram of RoeNet, which predicts future discontinuities from smooth observations. We note that for hyperbolic conservation laws with discontinuous solutions, RoeNet can accurately forecast long-term outcomes that are either fully or partially discontinuous. This is achievable even when the training data provided cover only a short window and contain limited information on discontinuities.

\subsubsection{Neural network architecture}\label{sec:nn}
Figure \ref{fig:RoeNet} shows an overview of our neural network architecture. In summary, RoeNet consists of $\bm L_{\theta}$ and $ \bm \Lambda_{\phi}$, two networks embedded in \eqref{eq:roenet} to serve as our template to evolve the system's states from $\bm{u}_j^{n}$ to $\bm{u}_j^{n+1}$.

\begin{figure}[ht]
  \centering
  \includegraphics[width=1.\linewidth]{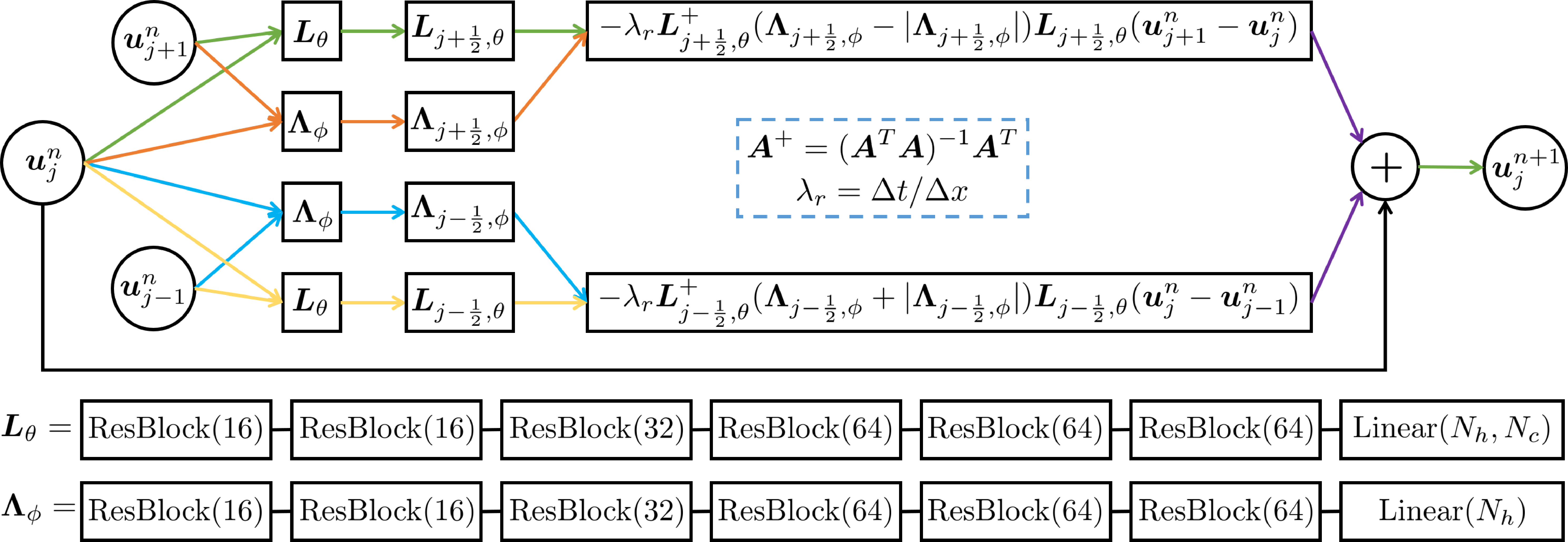}\\[0.5mm]
  \caption{The architecture of the neural network that evolves the system's states from $\bm{u}_j^{n}$ to $\bm{u}_j^{n+1}$ in RoeNet. This network takes the current conserved quantity $\bm u^n_j$ and its direct neighbors, $\bm u^n_{j-1}$ and $\bm u^n_{j+1}$, as the inputs and outputs the conserved quantity $\bm u^{n+1}_{j}$ of the next time step. The ResBlock has the same architecture as in \citep{he2016resnet}, only with the 2D convolution layers replaced by linear layers. The number in the parentheses is the dimension of each Resblock output. Source: \citep{tong2024roenet}.}
  \label{fig:RoeNet}
\end{figure}
Specifically, the network in Figure \ref{fig:RoeNet} contains two parts, each consists of a $\bm L_{\theta}$ and a $ \bm \Lambda_{\phi}$. The first part takes $\bm u_{j-1}^n$ and $\bm u_{j}^n$ as input of both $\bm L_{\theta}$ and $ \bm \Lambda_{\phi}$ and outputs $\bm L_{j-\frac{1}{2},\theta}$ through $\bm L_{\theta}$ and $\bm \Lambda_{j-\frac{1}{2},\phi}$ through $\bm \Lambda_{\phi}$. The input $\bm u_{j-1}^n$ and $\bm u_{j}^n$ is a vector $
[\bm u^{n, (1)}_{j-1}, \cdots, \bm u^{n, (N_c)}_{j-1}; \bm u^{n, (1)}_{j},\cdots, \bm u^{n, (N_c)}_{j}]$ of length $2N_c$ with $N_c$ components. The output matrix $L_{j-\frac{1}{2},\theta}$ is of size $(N_c\times N_h)$, and the other output matrix $\bm \Lambda_{j-\frac{1}{2},\phi}$ is a diagonal matrix of size $(N_h\times N_h)$. The second part takes $\bm u_j^n$ and $\bm u_{j+1}^n$ as the input for both $\bm L_{\theta}$ and $ \bm \Lambda_{\phi}$ and outputs $\bm L_{j+\frac{1}{2},\theta}$ through $\bm L_{\theta}$ and $\bm \Lambda_{j+\frac{1}{2},\phi}$ through $ \bm \Lambda_{\phi}$. The input $\bm u_j^n$ and $\bm u_{j+1}^n$ is a vector $[\bm u^{n, (1)}_{j},\cdots, \bm u^{n, (N_c)}_{j}; \bm u^{n, (1)}_{j+1}, \cdots, \bm u^{n, (N_c)}_{j+1}]$ of length $2N_c$. The output matrices $\bm L_{j+\frac{1}{2},\theta}$ and  $\bm \Lambda_{j+\frac{1}{2},\phi}$ take the same form as the output matrices in the first part. Given the four output matrices $\bm L_{j-\frac{1}{2},\theta}$, $\bm \Lambda_{j-\frac{1}{2},\phi}$, $\bm L_{j+\frac{1}{2},\theta}$, and  $\bm \Lambda_{j+\frac{1}{2},\phi}$, we combine them through (\ref{eq:roenet}) and \eqref{eq:LL} to obtain $\bm u_j^{n+1}$. Networks $\bm L_{\theta}$ and $ \bm \Lambda_{\phi}$ both consist of a chain of ResBlocks \cite{he2016resnet} with a linear layer of size $N_h\times N_c$ and $N_h$ at the end, respectively. 
The ResBlock architecture comprises two convolutional layers and one ReLU layer. 
\begin{comment}
ResNet has been proven in numerous research studies to be a neural network architecture highly suitable for deep learning and computer vision. It offers distinctive advantages in mitigating problems like gradient vanishing during network training. Our RoeNet, tailored to predict the evolution of PDEs on a uniform grid, exhibits increasing network depth as the number of time iterations grows. This network design closely mirrors the deep neural networks commonly employed in computer vision, making ResBlock an ideal selection for building our network.
\end{comment}
The $N_h$ learned parameters by $ \bm \Lambda_{\phi}$ is then transferred into a diagonal matrix of $N_h\times N_h$ with the learned parameters as its diagonal. The ResBlock has the same architecture as in \cite{he2016resnet}, only with the 2D convolution layers replaced by linear layers. Note that the number in the parentheses is the dimension of the output of each ResBlock, and the computation procedure for grid cell $j$ is applied to all grid cells. Since the computation of each node is independent of other cells except the adjacent cells, we could train them in parallel to achieve high efficiency.

In addition, we implement two ways of padding to address different boundary conditions. For periodic boundary conditions, we use the periodic padding, e.g., if $j=0$, then $\bm u_{j-1} = \bm u_{N_g}$, where $N_g$ is the number of the grid node. For Neumann boundary conditions, we use the replicate padding, e.g., if $j=0$, then $\bm u_{j-1}=\bm u_0$.

By introducing a hidden dimension $N_h$, we have increased the number of network parameters and enhanced the network's expressive capacity. However, the expansion of the parameter space could lead to multiple numerical optimal solutions during training. To address this, we employ a regularized loss function, which helps ensure that the network parameters converge to a local optimal solution. Importantly, our goal is to use the network to accurately model the evolution of PDEs over time and space; achieving a unique solution for the network parameters is not a requirement.

Algorithm \ref{alg:Roenet} summarizes the recursive relation from the input layer
\begin{equation}
    \bm u(t= 0) = [\bm u_1(t=0),\cdots,\bm u_{N_g}(t=0)]
\end{equation}
to the output layer
\begin{equation}
    \hat{\bm u}(t= T_{span}) = [\hat{\bm u}_1(t=T_{span}),\cdots,\hat{\bm u}_{N_g}(t=T_{span})],
\end{equation}
for each time step in RoeNet. Here $N_g$ is the spatial grid size and $T_{span}$ is the time span $T_{train}$ or $T_{predict}$.
As described in Algorithm \ref{alg:Roenet}, feeding $\bm u(t = 0)$, $T_{span} = T_{train}$, temporal step $\Delta t$, spatial step $\Delta x$, and the constructed networks $\bm L_{\theta}$ and $\bm \Lambda_{\phi}$ into RoeNet, we could get predicted $\hat{\bm u}(t=  T_{train})$.
Then, we choose the MSE as our loss function
\begin{equation}
    \mathcal L_{RoeNet}=\|\bm u(t=  T_{train})-\hat{\bm u}(t=  T_{train})\|_{MSE}.
\label{eq:loss}
\end{equation}
% For a rationale for using the MSE loss function, please refer to the ablation test in section \ref{sec:ablation}.

\begin{algorithm}
  \caption{Recursive relation from the input layer to the output layer in RoeNet. Here, $\bm{u}_j,~j=1,2,\cdots,N_g$ represents discretized points $\bm{u}$ in spatial coordinate. Source: \citep{tong2024roenet}.}
  \label{alg:Roenet}
  \begin{algorithmic}[1]
    \State \textbf{Inputs:} $\bm{u}_j(t=0),~j=1,2,\cdots,N_g$, $T_{span}$, $\Delta t$, $\Delta x$, $\bm{L}_{\theta}$, $\bm{\Lambda}_{\phi}$
    \State  \textbf{Outputs:}  $\hat{\bm{u}}_j(t = T_{span})=\bm{u}_j^{N_t}$
    \State  $N_t = \text{floor}(T_{span}/\Delta t)$
    \State  $\lambda_r = \Delta t/\Delta x$
    \State  $\bm{u}_j^0 = \bm{u}_j(t=0),~j=1,2,\cdots,N_g$
    \For{$n = 0$ \text{to} $N_{t-1}$}
      \State  Calculate $\bm{L}_{j\pm\frac{1}{2},\theta}^n$, $\bm{\Lambda}_{j\pm\frac{1}{2},\phi}^n$ by substituting $\bm{u}_j^n,~j=1,2,\cdots,N_g$, $\bm{L}_{\theta}$, and $\bm{\Lambda}_{\phi}$ into \eqref{eq:LL}
      \State  Calculate $\bm{u}_j^{n+1},~j=1,2,\cdots,N_g$ by substituting $\bm{u}_j^n$, $\bm{L}_{j\pm\frac{1}{2},\theta}^n$, $\bm{\Lambda}_{j\pm\frac{1}{2},\phi}^n$, and $\lambda_r$ into \eqref{eq:roenet}
    \EndFor
  \end{algorithmic}
\end{algorithm}
\subsection{Neural Vortex Method (NVM)}

To accurately and efficiently quantify fluid dynamics, we propose the novel NVM framework. This framework utilizes physics-informed neural networks to extract and translate information from the Eulerian specification of the flow field (or images of flow visualizations) into knowledge about the underlying fluid field. As detailed in Figure \ref{fig:overview2}, we integrate these networks with a vorticity-to-velocity Poisson solver to build a fully automated toolchain that extracts high-resolution Eulerian flow fields from Lagrangian inductive priors. This design addresses the challenge of learning directly from high-dimensional observations, such as images, which traditional methods struggle to convert directly into velocity and pressure fields.

\begin{figure}
  \centering
  \includegraphics[width=1.\textwidth]{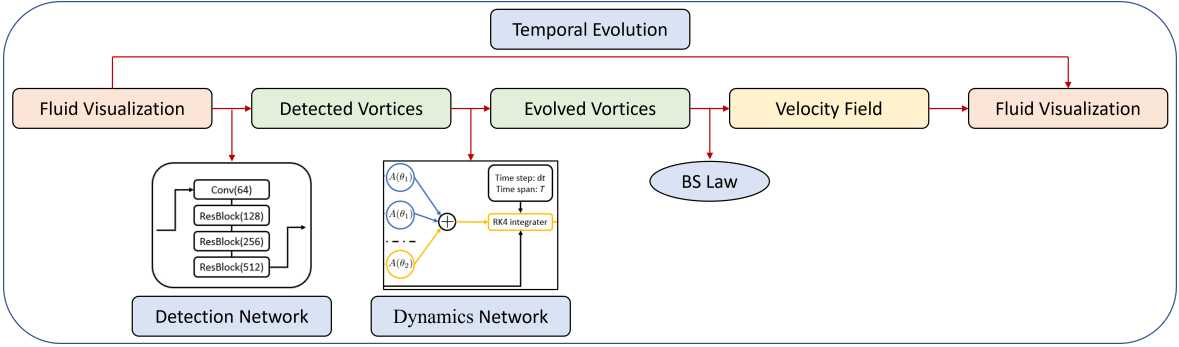}\\
  \caption{Schematic diagram of the NVM. Our system is constituted of two networks, the detection network and the dynamics network, which are embedded with a vorticity-to-velocity Poisson solver. Source: \citep{xiong2023neural}.}
  \label{fig:overview2}
\end{figure}

We construct a vortex detection network in Section \ref{subsec:detnet} to identify the positions and the vorticity of Lagrangian vortices from a grid-based velocity field, which from a mathematical perspective connects \eqref{eq:uNS} with \eqref{eq:vortex}. This approach simplifies the vorticity field to include only the detected vortices. Given the detected vortices, we then use a vortex dynamics network in Section \ref{subsec:dynet} to learn the underlying governing dynamics of these finite structures. Dynamics networks accurately model the r.h.s.of \eqref{eq:vortex} under various conditions, resolving the longstanding problem in LVM.

The training of the NVM involves two primary steps: training the detection and dynamics networks. We employ high-fidelity data from direct numerical simulation (DNS) of interactions among 2 to 6 vortices, although the model can generalize to any vorticity field with an arbitrary number of vortices. We initially train the detection network using data from randomly generated vortices and their vorticity fields, then identify vortices' positions and strengths using this trained network to facilitate the subsequent training of the dynamics network.

% To predict the future flow, we use the trained detection network to simplify the vorticity field into a field that only constituted of the identified vortices. Given the detected vortices, we then use the dynamics network to learn the underlying governing dynamics of these finite structures.

\subsubsection{Detection network}
\label{subsec:detnet}

The input of the detection network is a vorticity field of size $200\times200\times1$. As shown in Figure \ref{fig:detection}, we first feed the vorticity field into a small one-stage detection network and get the feature map of size $25\times25\times512$ (we downsampled 3 times). The detection network consists of a Conv2d-BatchNorm-ReLU combo and a 6-layer-structured ResBlock chain whose size can be adjusted dynamically to the complexity of the problem. The primary reason for downsampling is to avoid extremely unbalanced data and multiple predictions for the same vortex. We then forward the feature map to 2 branches. In the first branch, we conduct a $1\times1$ convolution to generate a probability score $\hat{p}$ of the possibility that there exists a vortex. If $\hat{p} > 0.5$, we believe there exists a vortex within the corresponding cells of the original $200\times200\times1$ vorticity field. In the second branch, we predict the relative position to the left-up corner of the cell of the feature map if the cell contains a vortex. Afterward, we set a bounding box of $10\times10$ around these
predicted vortices and use the weighted average of the positions of the cells of the original vorticity field to find the exact position of the vortex. Finally, the vortex particle strength is calculated as the sum of the value of the cells in the bounding box normalized by the cell area.

\begin{figure}
  \centering
  \includegraphics[width=1.\textwidth]{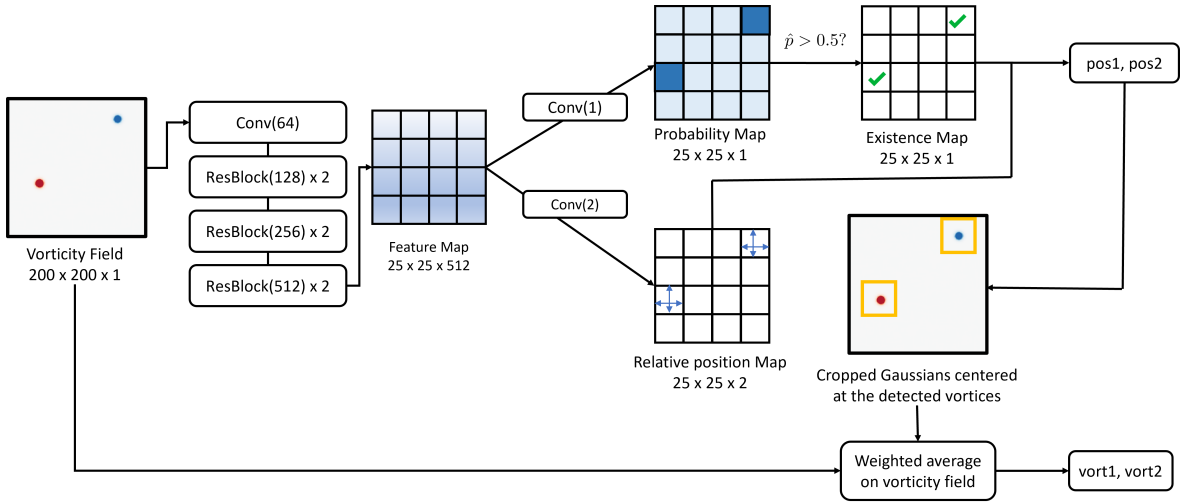}
  \caption{The architecture of the detection network. It takes the vorticity field as input and outputs the position and vortex particle strength for each vortex detected. The \textit{Conv} means the Conv2d-BatchNorm-ReLU combo, and the \textit{ResBlock} is the same as in \citep{he2016resnet}. In each \textit{ResBlock}, we use stride 2 to downsample the feature map. The Resblock chain is six-layer structured. The number in the parenthesis is the output dimension. Source: \citep{xiong2023neural}.}
  \label{fig:detection}
\end{figure}

In the training process, we penalize the wrong position detection only if the cell containing a vortex in the ground truth given by DNS is not detected. This idea is similar to the real-time object detection in \citep{redmon2016yolo}. We do not use the weighted average method to find the position in the training to ensure the detection network can produce detection results as accurately as possible. We use the focal loss \citep{lin2017focal} to further relieve the unbalanced classification problem.

We mainly use the detection network to generate training data for the dynamics network because we want to use the high-resolution data generated by the method mentioned in Section \ref{subsec:data_gen} instead of by the approximate particle method (BS law). Moreover, there are many situations where BS law is inapplicable, as discussed previously in Section \ref{LVM}. The detection network enables us to find the positions of the vortices accurately regardless of the situation.

The detection network is responsible for providing necessary information to the dynamics network. After the training, we use the well-trained detection network to detect the vortices in the initial vorticity fields and the evolved vorticity field, both generated by the method in Section \ref{subsec:data_gen}. We then apply the nearest-neighbor method to pair the vortices detected in these two fields. Figure \ref{fig:detection123} shows the case of two fields at $t=0$ and $t=0.2$. The idea of nearest-neighbor pairing can be perceived from Figure \ref{fig:detection123} (c). The sample, or these two fields, is dropped if different numbers of vortices are detected in the initial and evolved fields or if a large difference exists in the vorticity of paired vortices. The successfully detected vortices in the initial and evolved vorticity fields are passed together into the dynamics network for its training.

\begin{figure}
  \centering
  \includegraphics[width=1.\textwidth]{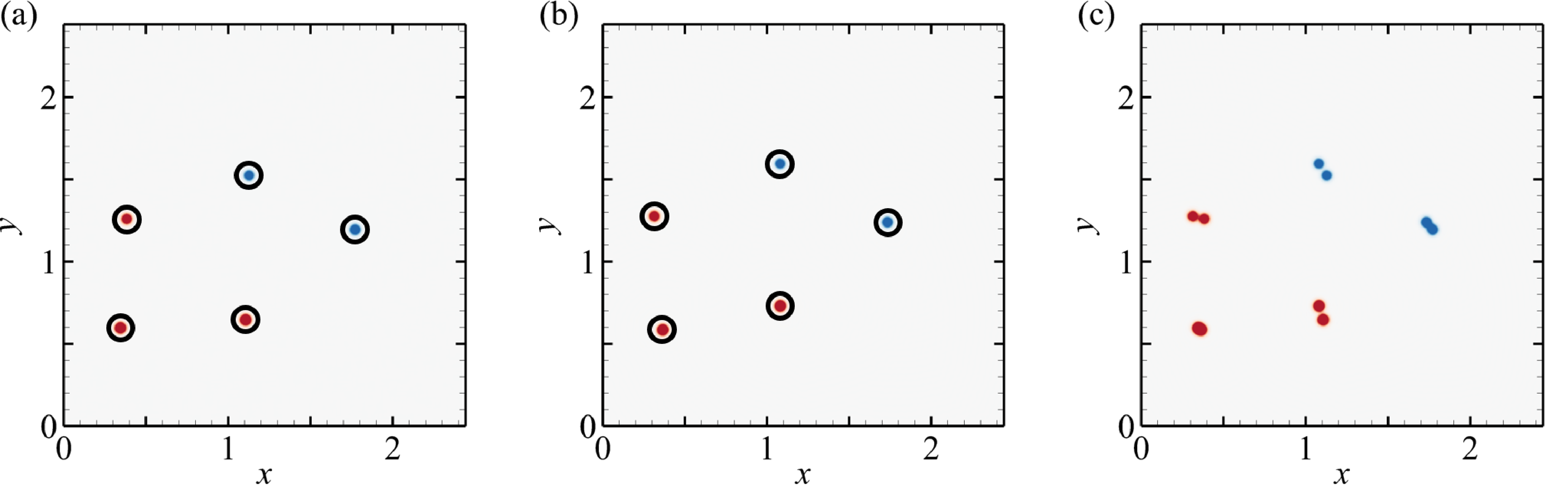}\\
  \caption{A example of vorticity contour at (a) $t=0$, (b) $t=0.2$, and (c) superposition of $t=0$ and $t=0.2$. The black circles indicate the location recognized by the detection network. The evolution from (a) to (b) is calculated by DNS. Source: \citep{xiong2023neural}.}
  \label{fig:detection123}
\end{figure}

\subsubsection{Dynamics network}
\label{subsec:dynet}

To learn the underlying dynamics of the vortices, we build a graph neural network similar to \citep{battaglia2016interaction}. We predict the velocity of one vortex due to influences exerted by the other vortices and the external force. Then we use the fourth-order Runge--Kutta integrator to calculate the position in the next timestamp. As shown in Figure \ref{fig:dynamics}, for each vortex, we use a neural network $A(\theta_1)$ to predict the influences exerted by the other vortices and add them up. Specifically, for each $i$th vortex, we consider the vortex $j (j\neq i)$. The difference of their positions can be calculated by $\textrm{diff}_{ij}=\textrm{pos}_i-\textrm{pos}_j$, and their L2 distance is $\textrm{dist}_{ij}=\|\textrm{diff}_{ij}\|_2$. The input of the $A(\theta_1)$ is the vector $(\textrm{diff}_{ij}, \textrm{dist}_{ij}, \textrm{vort}_j)$ of length 4. Here, $\textrm{pos}$ and $\textrm{vort}$ are detected by the detection network.
%The output is the influence of the vortex $j$ on the vortex $i$.
The output of $A(\theta_1)$ is a vector with the same dimension of the flow field, characterizing the induced velocity of the $j$th vortex to the $i$th vortex. In this way, we can calculate the induced velocity of each vortex $j$ ($j\neq i$) on the vortex $i$. We sum up all the induced velocities on the vortex $i$ and treat the result as the induced velocity exerted by the other vortices.

\begin{figure}
  \centering
  \includegraphics[width=1.\textwidth]{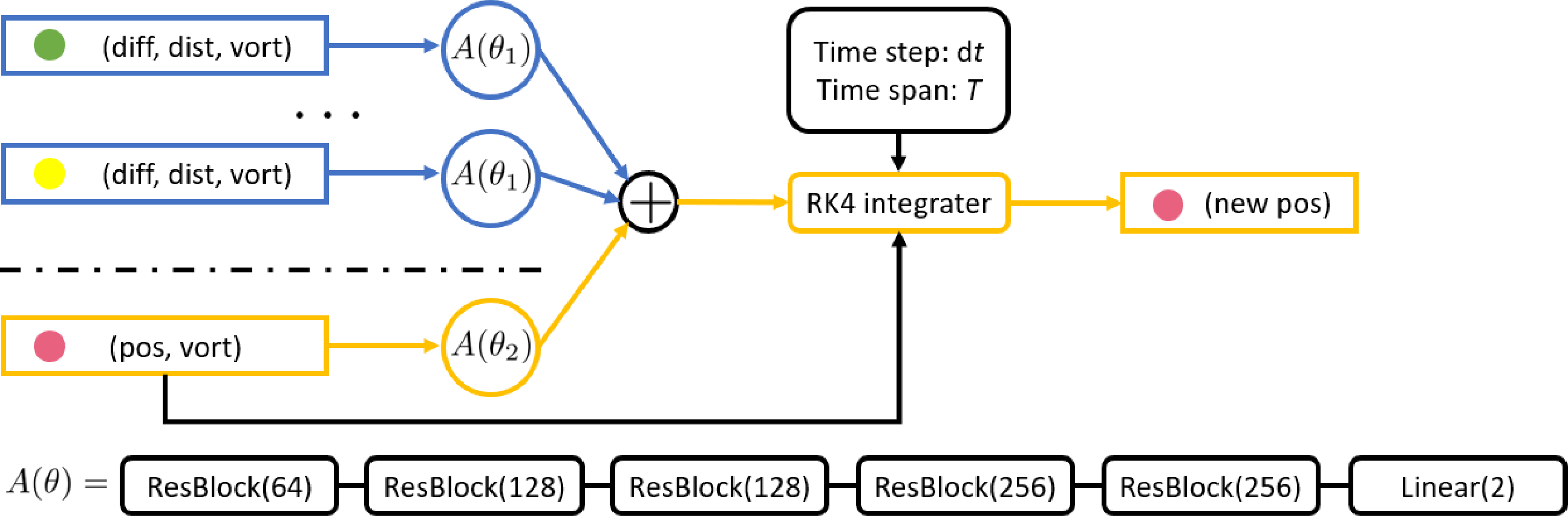}
  \caption{The architecture of the dynamics network. It takes the particle's attribution as input and outputs each vortex's position. The \textit{ResBlock} has the same architecture as in \citep{he2016resnet} with the convolution layers replaced by linear layers. The number in the parenthesis is the output dimension. Source: \citep{xiong2023neural}.}
  \label{fig:dynamics}
\end{figure}

In addition, we use another neural network $A(\theta_2)$, to predict the influence caused by the external force, which is determined by the local vorticity and the position of the vortex. The input of $A(\theta_2)$ is a vector of length 3. The output is the influence exerted by the environment on the vortex $i$, i.e., the induced velocity of the external force to $i$th vortex.

The reason we separate the induced velocity into two parts, i.e., $A(\theta_1)$ and $A(\theta_2)$, is as follows. On the one hand, the induced velocities between vortex particles are global, and exhibit a certain symmetry, i.e., the vortex particles interact with each other following the same law. In contrast, the influence of external forces on vortex particles is usually local and direct; thus, we do not need to consider the interaction between particles. The effect of the vortex stretching term in three-dimensional vortex flows or diffusion term in viscous flows is also local and should be included in network $A(\theta_2)$. Note that both the outputs of $A(\theta_1)$ and $A(\theta_2)$ are a vector with the same dimension of the flow field. Thus, we can add the two kinds of influence together, whose result is defined as the velocity of the vortex $i$. We feed the velocity into the fourth-order Runge--Kutta integrator to obtain the predicted position of vortex $i$.

In addition, in predicting the evolution of the flow field, NVM replaces the discrete BS method with a dynamics network composed of ResBlocks. We chose a 5-layer ResBloks to improve the expressiveness of the dynamics network so that we can learn dynamics of different complexity on the same network. Since the dynamics network with 5-layer ResBloks is more complex than the discrete BS method, the computational cost of NVM is higher than that of the Lagrangian vortex method. We remark that although the computational cost of ResBlocks itself is relatively large in NVM, the number of vortex particles needed to predict the evolution of the flow field using NVM is much smaller. Therefore, the overall computational cost of NVM can be greatly reduced.

\section{Results}

We present several experiments here to highlight the key advantages of our methodologies. For additional examples and ablation tests, please refer to \citep{tong2021symplectic,xiong2020nonseparable,xiong2023neural,tong2024roenet}.

\subsection{Taylor-nets}
\subsubsection{Dataset generation and training settings}\label{subsec:data}
To make a fair comparison with the ground truth, we generate our training and testing datasets by using the same numerical integrator based on a given analytical Hamiltonian. In the learning process, we generate $N_{train}$ training samples, and for each training sample, we first pick a random initial point $(\bm{q}_0,\bm{p}_0)$ (input), then use the symplectic integrator discussed in Section \ref{syminte} to calculate the value $(\bm{q}_n,\bm{p}_n)$ (target) of the trajectory at the end of the training period $T_{train}$. We do the same to generate a validation dataset with $N_{validation}=100$ samples and the same time span as $T_{train}$ and calculate the validation loss $L_{validation}$ along the training loss $L_{train}$ to evaluate the training process. In addition, we generate a set of testing data with $N_{test}=100$ samples and predicting time span $T_{predict}$ that is around 6000 times larger and calculate the prediction error $\epsilon_p$ to evaluate the predictive ability of the model.
For simplicity, we use $(\bm{\hat{p}}_n,\bm{\hat{q}}_n)$ to represent the predicted values using our trained model.

We remark that our training dataset is relatively smaller than that used by the other methods. Most of the methods, e.g. ODE-net \cite{chen2018neural} and HNN \cite{Greydanus2019}, have to rely on intermediate data in their training data to train the model. That is the dataset 
is 
\[[(\bm{q}_{0}^{(s)}, \bm{p}_{0}^{(s)}),(\bm{q}_{1}^{(s)}, \bm{p}_{1}^{(s)}),\dots,(\bm{q}_{n-1}^{(s)},\bm{p}_{n-1}^{(s)}), (\bm{q}_n^{(s)}, \bm{p}_n^{(s)})]_{s=1}^{N_{train}},\] where $(\bm{q}_{1}^{(s)},\bm{p}_{1}^{(s)})\dots,(\bm{q}_{n-1}^{(s)},\bm{p}_{n-1}^{(s)})$ are $n-1$ intermediate points collected within $T_{train}$ in between $(\bm{q}_{0}^{(s)}, \bm{p}_{0}^{(s)})$ and $(\bm{q}_n^{(s)},  \bm{p}_n^{(s)})$. On the other hand, we only use two data points per sample, the initial data point and the end point, and our dataset looks like 
\[\left[(\bm{q}_{0}^{(s)},\bm{p}_{0}^{(s)}), (\bm{q}_n^{(s)},\bm{p}_n^{(s)})\right]_{s=1}^{N_{train}},\]
which is $n-1$ times smaller the dataset of the other methods, if we do not count $(\bm{q}_{0}^{(s)},\bm{p}_{0}^{(s)})$.
Our predicting time span $T_{predict}$ is around 6000 times the training period used in the training dataset $T_{train}$ (as compared to 10 times in HNN). This leads to a 600 times compression of the training data, in the dimension of temporal evolution. Note that we fix $T_{train}$ and $T_{predict}$ in practice so that we can train our network more efficiently on GPU. One can also choose to generate training data with different $T_{train}$ for each sample to obtain more robust performance.

We use the Adam optimizer \citep{kingma2014adam}. We choose the automatic differentiation method as our backward propagation method. We have tried both the adjoint sensitivity method, which is used in ODE-net \cite{chen2018neural} and the automatic differentiation method. Both methods can be used to train the model well. However, we found that using the adjoint sensitivity method is much slower than using the automatic differentiation method considering the large parameter size of neural networks.

All $A_i$ and $B_i$ in \eqref{eq:Tp_Taylor} are initialized as $A_i, B_i \sim \mathcal{N}(0,\sqrt{2/[N*N_h*(i+1)]})$, where $N$ is the dimension of the system and $N_h$ is the size of the hidden layers. The loss function is

\begin{equation}
L_{train}=\frac{1}{N_{train}}\sum_{s=1}^{N_{train}}\|\bm{\hat{p}}_n^{(s)}-\bm{p}_n^{(s)}\|_1+\|\bm{\hat{q}}_n^{(s)}-\bm{q}_n^{(s)}\|_1.
\label{eq:loss}
\end{equation}
The validation loss $L_{validation}$ is the same as \eqref{eq:loss} but with dataset different from the training dataset. We choose $L1$ loss, instead of Mean Square Error (MSE) loss because of its better performance.

We will introduce the experimental result for an ideal pendulum system, which is defined 
\begin{equation}
   \mathcal {H}(q, p) =\frac{1}{2}p^2-\cos{(q)}.
   \label{eq:pendu}
 \end{equation}
We pick a random initial point for training $(\bm{q}_0,\bm{p}_0)\in \left[-2,2\right]\times\left[-2,2\right]$.

To show the predictive ability of our model, we pick $T_{predict}=20\pi$. We pick 15 as the sample size since we find that small $N_{train}$'s are sufficient to generate excellent results. We use 100 epochs for training, and 10 as the $step\_size$ (the period of learning rate decay), and 0.8 as $\gamma$ (the multiplicative factor of learning rate decay). The learning rate of each parameter group is decayed by $\gamma$ every $step\_size$ epochs, which prevents the model from overshooting the local minimum. The dynamic learning rate can also make our model converge faster. $M$ indicates the number of terms of the Taylor polynomial introduced in the construction of the neural networks \eqref{eq:Tp_Taylor}. Through experimentation, we find that 8 terms can represent most functions well. We choose 16 as $N_h$, the dimension of hidden layers.

\subsubsection{Predictive ability and robustness}\label{sub:compare}

\begin{figure}
  \centering
 % \psfrag{m}{\scriptsize Taylor-net}
  %  \psfrag{n}{\scriptsize HNN}
   % \psfrag{l}{\scriptsize ODE-net}
   % \psfrag{x}[c][c]{\footnotesize $t$}
   % \psfrag{y}[c][c]{\footnotesize %$\epsilon_p^{(n_t)}$}
   % \psfrag{a}[c][c]{\footnotesize (a)}
   %\psfrag{b}[c][c]{\footnotesize (b)}
   %\psfrag{c}[c][c]{\footnotesize (c)}
  \includegraphics[width=1\linewidth]{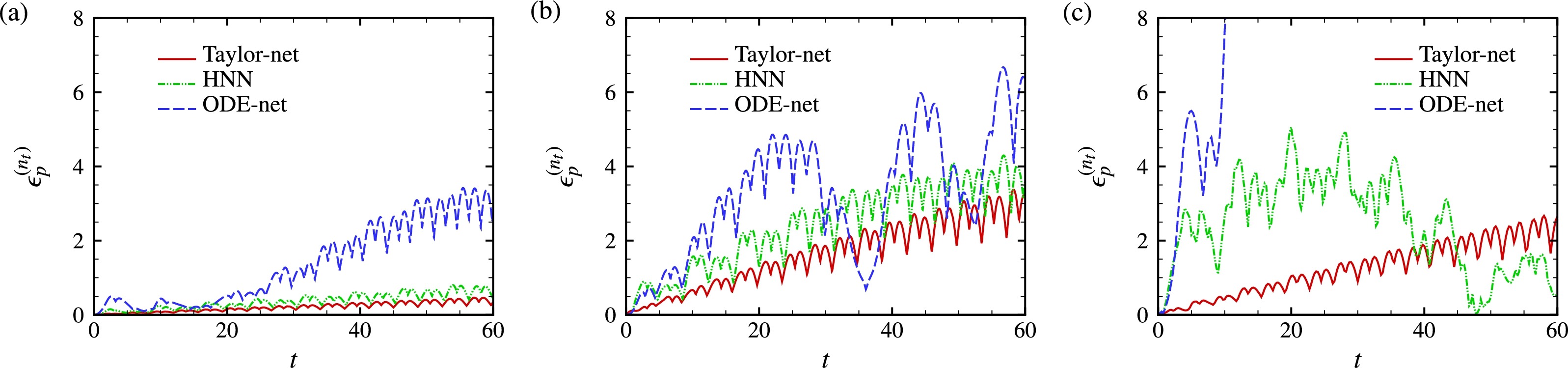}
  \caption{Prediction error $\epsilon_p^{(n_t)}$ at different $t$ from $t=0$ to $t=20\pi$ for the pendulum problem (a) without noise, (b) with noise $\sigma_1, \sigma_2 \sim \mathcal{N}(0,0.1)$, and (c) with noise $\sigma_1, \sigma_2 \sim \mathcal{N}(0,0.5)$. In the figure, $t=n_t \Delta t$, where $\Delta t=0.01$. $\epsilon_p^{(n_t)}$ is the prediction error at the $n_t^{\textrm{th}}$ predicted point among the total $N_T=T_{predict}/\Delta t$ predicted points. We use $T_{train}=0.01$, $T_{train}=0.5$ and $T_{train}=1$ to train the model in (a), (b), and (c), respectively. Source: \citep{tong2021symplectic}.}
  \label{fig:Error}
\end{figure}
\begin{table}
  \caption{Comparison of $\epsilon_p$ for the pendulum problem without noise, with noise $\sigma_1, \sigma_2 \sim \mathcal{N}(0,0.1)$, and with noise $\sigma_1, \sigma_2 \sim \mathcal{N}(0,0.5)$. Source: \citep{tong2021symplectic}.}
  \centering
  \setlength{\tabcolsep}{5mm}{
  \begin{tabular}{lccc}
  \hline
  Methods & Taylor-net & HNN & ODE-net\\
  \hline
  $\epsilon_p$, without noise & 0.213 &0.377&1.416\\
  $\epsilon_p$, with noise $\sigma_1, \sigma_2 \sim \mathcal{N}(0,0.1)$ & 1.667 &2.433&3.301\\
  $\epsilon_p$, with noise $\sigma_1, \sigma_2 \sim \mathcal{N}(0,0.5)$&1.293  &2.416 & 27.114\\
  \hline
  \end{tabular}}
  \label{tab:compa_err}
\end{table}

\begin{figure}
  \centering
 %  \psfrag{m}{\scriptsize Ground Truth}
  %  \psfrag{n}{\scriptsize Taylor-net}
  %  \psfrag{l}{\scriptsize HNN}
  %  \psfrag{o}{\scriptsize ODE-net}
    %\psfrag{e}{\scriptsize $T_{train}$}
  %  \psfrag{a}[c][c]{\footnotesize (a)}
  %  \psfrag{b}[c][c]{\footnotesize (b)}
  %  \psfrag{c}[c][c]{\footnotesize (c)}
  %  \psfrag{x}[c][c]{\footnotesize $t$}
  %  \psfrag{y}[c][c]{\footnotesize $\bm q$}
  \includegraphics[width=0.9\linewidth]{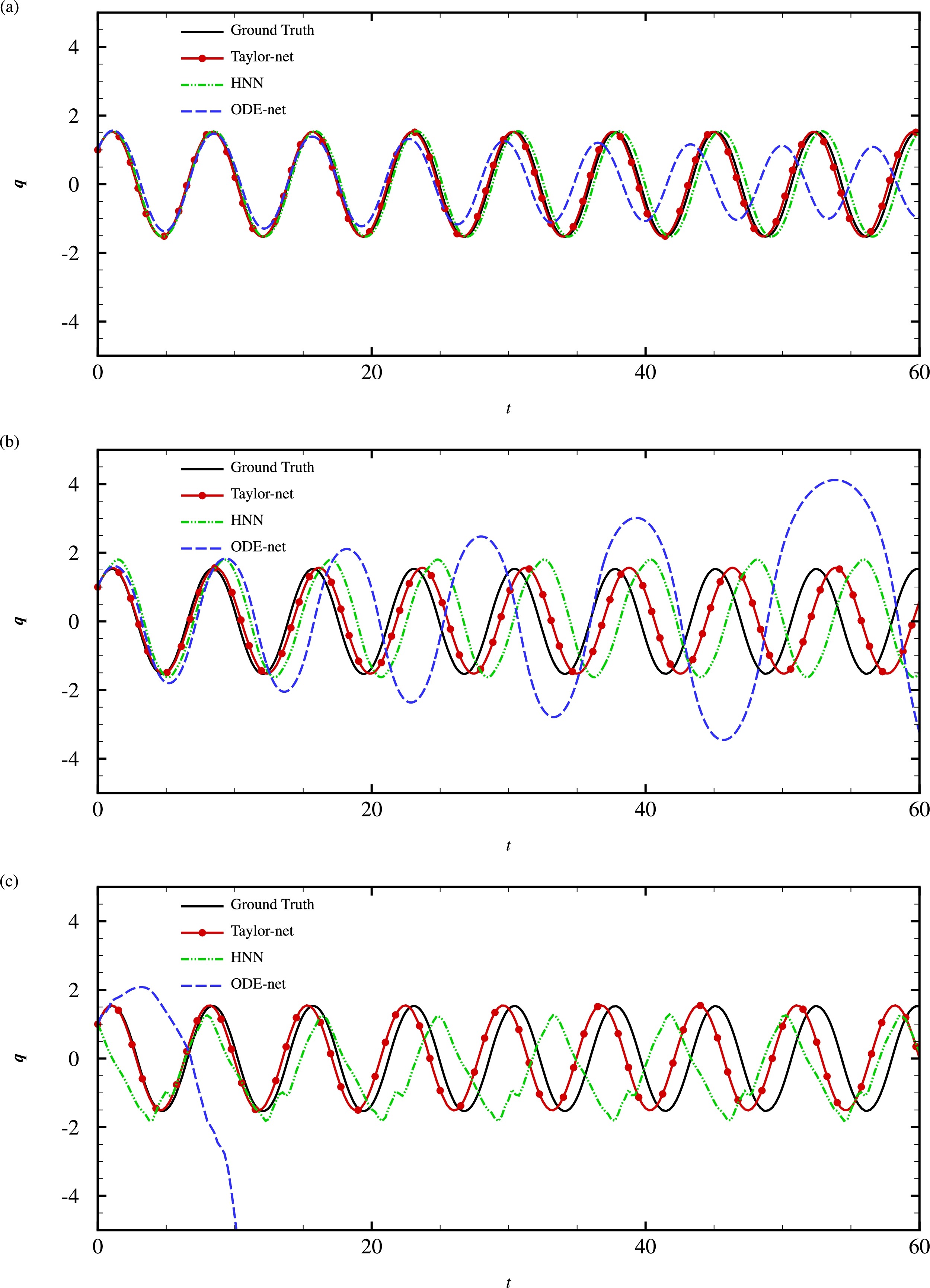}\\

  \caption{Prediction results of position $\bm q$ from $t=0$ to $t=20\pi$ for the pendulum problem using Taylor-net, HNN, and ODE-net  (a) without noise, (b) with noise $\sigma_1, \sigma_2 \sim \mathcal{N}(0,0.1)$, and (c) with noise $\sigma_1, \sigma_2 \sim \mathcal{N}(0,0.5)$. For all the models, we set the initial point as $(\bm{q}_0,\bm{p}_0)=(1,1)$. We use $T_{train}=0.01$, $T_{train}=0.5$ and $T_{train}=1$ to train the model in (a), (b), and (c), respectively. All the methods are trained until the $L_{validation}$ converges. Source: \citep{tong2021symplectic}.}
  \label{fig:prediction_q}
\end{figure}

Now, to assess how well our method can predict the future flow, we compare the predictive ability of Taylor-net with ODE-net and HNN. We apply all three methods on the pendulum problem, and let $T_{train}=0.01$ and $T_{predict}=20\pi$. We evaluate the performance of the models by calculating the average prediction error at each predicted points, defined by
\begin{equation}
  \epsilon_p^{(n_t)}=\frac{1}{N_{test}}\sum_{s=1}^{N_{test}}\|\bm{\hat{p}}^{(s, n_t)}_n-\bm{p}_n^{(s, n_t)}\|_1+\|\bm{\hat{q}}_n^{(s,  n_t)}-\bm{q}_n^{(s, n_t)}\|_1,
\end{equation}
and the average $\epsilon_p^{(n_t)}$ over $T_{predict}$ is
\begin{equation}
  \epsilon_p=\frac{1}{N_T}\sum_{n_t=1}^{N_T}\epsilon_p^{(n_t)},
\end{equation}
where $N_{test}$ represents the testing sample size specified in Section \ref{subsec:data} and $N_T=T_{predict}/\Delta t$ with $\Delta t = 0.01$. After experimentation, we find that Taylor-net has stronger predictive ability than the other two methods. The first row of Table \ref{tab:compa_err} shows the average prediction error of 100 testing samples using the three methods over $T_{predict}$ when no noise is added. The prediction error of HNN is almost double that of Taylor-net, while the prediction error of ODE-net is about 7 times that of Taylor-net. To analyze the difference more quantitatively, we made several plots to help us better compare the prediction results. Figure \ref{fig:Error} shows the plots of prediction error $\epsilon_p^{(n_t)}$ against $t=n_t \Delta t$ over $T_{predict}$ for all three methods. In Figure \ref{fig:prediction_q}, we plot the prediction of position $q$ against time period for all three methods as well as the ground truth in order to see how well the prediction results match the ground truth. From Figure \ref{fig:prediction_q} (a), we can already see that the prediction result of ODE-net gradually deviates from the ground truth as time progresses, while the prediction of Taylor-net and HNN stays mostly consistent with the ground truth, with the former being slightly closer to the ground truth. The difference between Taylor-net and HNN can be seen more clearly in Figure \ref{fig:Error} (a). Observe that the prediction error of Taylor-net is obviously smaller than that of the other two methods, and the difference becomes more and more apparent as time increases. The prediction error of ODE-net is larger than HNN and Taylor-net at the beginning of $T_{predict}$ and increases at a much faster rate than the other two methods. Although the prediction error of HNN has no obvious difference from that of Taylor-net at the beginning, it gradually diverges from the prediction error of Taylor-net.

\subsection{NSSNNs}

\subsubsection{Dataset generation and training settings} \label{sec:Training}

We use 6 linear layers with hidden size 64 to model $\mathcal H_\theta$, all of which are followed by a Sigmoid activation function except the last one. The derivatives $\partial \mathcal H_\theta / \partial \bm p$, $ \partial \mathcal H_\theta/\partial \bm q$, $ \partial \mathcal H_\theta/\partial \bm x$, $ \partial \mathcal H_\theta / \partial \bm y$ are all obtained by automatic differentiation in Pytorch \citep{paszke2019pytorch}. The weights of the linear layers are initialized by Xavier initializaiton \citep{glorot2010understanding}.

We generate the dataset for training and validation using high-precision numerical solver \citep{Tao2016}, where the ratio of training and validation datasets is $9:1$.
We set the dataset $(\bm q_0^{j},\bm p_0^{j})$ as the start input and $(\bm q^{j}, \bm p^{j})$ as the target with $j= 1,2,\cdots,N_s$, and the time span between $(\bm q_0^{j},\bm p_0^{j})$ and $(\bm q^{j}, \bm p^{j})$ is $T_{train}$. Feeding $(\bm q_0,\bm p_0) = (\bm q_0^j, \bm p_0^{j}),~t_0=0,~t = T_{train}$, and time step $\textrm{d}t$ in Algorithm \ref{alg:int_net} to get the predicted variables  $(\hat { \bm q}^j,\hat { \bm p}^j,\hat { \bm x}^j,\hat { \bm y}^j)$. Accordingly, the loss function is defined as
\begin{equation}
\mathcal L_{NSSNN}=\frac{1}{N_{b}}\sum_{j=1}^{N_{b}} \|\bm q^{(j)}-\hat{\bm q}^{(j)}\|_1+\|\bm p^{(j)}-\hat{\bm p}^{(j)}\|_1+\|\bm q^{(j)}-\hat{\bm x}^{(j)}\|_1+\|\bm p^{(j)}-\hat{\bm y}^{(j)}\|_1,
\label{eq:loss}
\end{equation}
where $N_{b}=512$ is the batch size of the training samples. We use the Adam optimizer \citep{kingma2014adam} with learning rate 0.05. The learning rate is multiplied by 0.8 for every 10 epoches.

Taking system $\mathcal H(q,p) = 0.5(q^2+1)(p^2+1)$ as an example, we carry out a series of ablation tests based on our constructed networks to find the proper parameters. Normally, we set the time span, time step and dateset size as $T= 0.01$, $\textrm{d}t = 0.01$ and $N_s=1280$. The choice of $\omega$ in (\ref{eq:overlineHamilton}) is largely flexible since NSSNN is not sensitive to the parameter $\omega$ when it is larger than a certain threshold. 
We pick the $L1$ loss function to train our network due to its better performance. In addition, we already introduced a regularization term in the symplectic integrator embedded in the network; thus, there is no need to add the regularization term in the loss function. The integral time step in the sympletic integrator is a vital parameter, and the choice of $\textrm{d}t$ largely depends on the time span $T_{train}$. In general, we should take relatively small $\textrm{d}t$ for the dataset with larger time span $T_{train}$.

\subsubsection{Spring system}
We compare five implementations that learn and predict Hamiltonian systems. The first one is NeuralODE \citep{chen2018neural}, which trains the system by embedding the network $\bm f_{\theta}\to (\textrm{d} \bm q/\textrm{d} t, \textrm{d} \bm p/\textrm{d} t)$ into the Runge-Kutta (RK) integrator. The other four, however, achieve the goal by fitting the Hamiltonian $\mathcal H_{\theta}\to \mathcal H$ based on (\ref{eq:Hamilton}). Specifically, HNN trains the network with the constraints of the Hamiltonian symplectic gradient along with the time derivative of system variables and then embeds the well-trained $\mathcal H_{\theta}$ into the RK integrator for predicting the system \citep{Greydanus2019}. The third and fourth implementations are ablation tests. One of them is improved HNN (IHNN), which embeds the well-trained $\mathcal H_{\theta}$ into the nonseparable symplectic integrator (Tao's integrator) for predicting. The other is to directly embed $\mathcal H_{\theta}$ into the RK integrator for training, which we call HRK. The fifth method is NSSNN, which embeds $\mathcal H_{\theta}$ into the nonseparable symplectic integrator for training.

\begin{figure}
  \centering
  \includegraphics[width=1.0\linewidth]{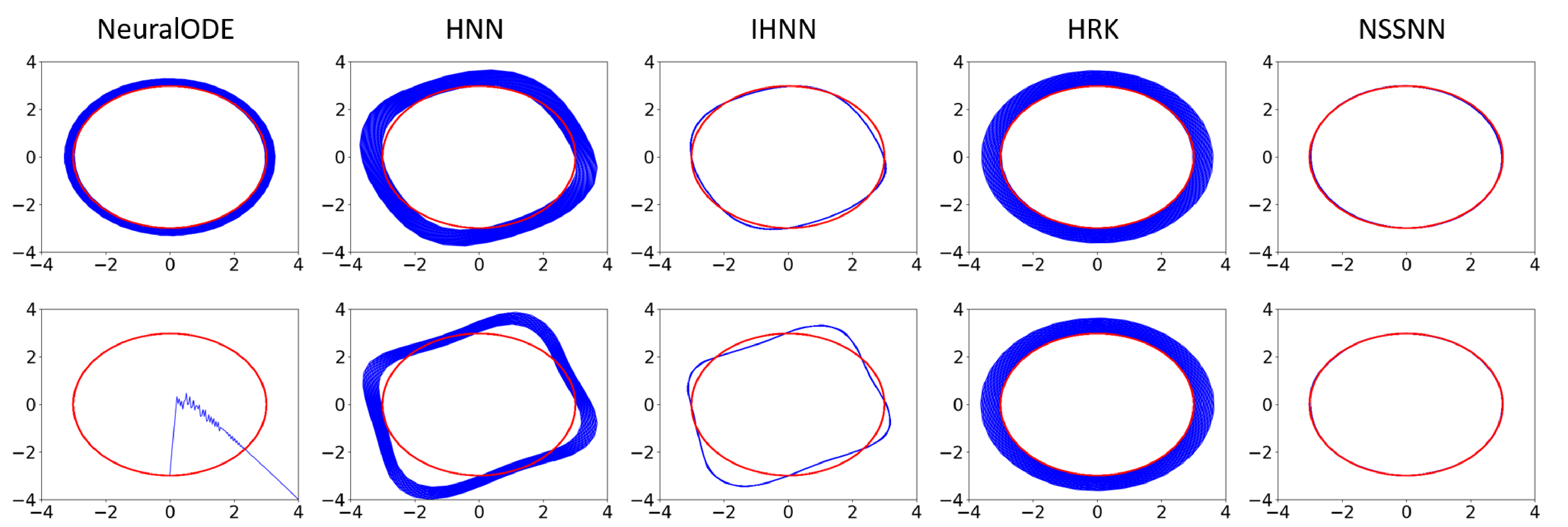}
  \caption{Comparison of prediction results of $(q,p)$ for the spring system $\mathcal H= 0.5(q^2+ p^2)$ from $t= 0$ to $t= 200$ with $(q_0,p_0) = (0,-3)$. The time span of the datasets are $T_{train} = 0.4$ (first row) and $T_{train} = 1$ (second row). The five columns are five different methods NeuralODE, HNN, IHNN, HRK, and NSSNN, respectively. The red line denotes the ground truth; the blue line denotes the prediction, which are perfectly overlapping in NSSNN. The prediction ability of HNN and IHNN improves significantly with the decreasing of $T_{train}$ of the dataset which however may be hard to obtain in the actual experimental measurements. Source: \citep{xiong2020nonseparable}.}
  \label{fig:Spring}
\end{figure}

For fair comparison, we adopt the same network structure (except that the dimension of output layer in NeuralODE is two times larger than that in the other four), the same $L1$ loss function and same size of the dataset, and the precision of all integral schemes is second order, and the other parameters keep consistent with the one in Section \ref{sec:Training}. The time derivative in the dataset for training HNN and IHNN is obtained by the first difference method
\begin{equation}
    \frac{\textrm{d} \bm q}{\textrm{d} t }\approx \frac{\bm q(T_{train})-\bm q(0)}{T_{train}}~~~~\textrm{and}~~~~\frac{\textrm{d} \bm p}{\textrm{d} t }\approx \frac{\bm p(T_{train})-\bm q(0)}{T_{train}}.
    \label{eq:dpq}
\end{equation}

Figure \ref{fig:Spring} demonstrates the differences between the five methods using a spring system $\mathcal H= 0.5(q^2+ p^2)$ with different time span $T_{train} = 0.4,~1$ and same time step $\textrm{d} t = 0.2$.
We can see that by introducing the nonseparable symplectic integrator into the prediction of the Hamiltonian system, NSSNN has a stronger long-term predicting ability than all the other methods. In addition, the prediction of HNN and IHNN lies in the dataset with time derivative; consequently, it will lead to a larger error when the given time span $T_{train}$ is large.

\subsubsection{Modeling vortex dynamics of multi-particle system}\label{subsec:vortex}
For two-dimensional vortex particle systems, the dynamical equations of particle positions $(x_j,y_j),~j = 1,2,\cdots,N_v$ with particle strengths $\Gamma_j$ can be written in the generalized Hamiltonian form as
\begin{equation}
   \Gamma_j \frac{\textrm {d} x_j}{\textrm d t}  = -\frac{\partial \mathcal{H}^p}{\partial y_j},~~~~
    \Gamma_j \frac{\textrm d y_j}{\textrm d t}  = \frac{\partial \mathcal{H}^p}{\partial x_j}
,~~~~\textrm{with}~~~~
   \mathcal H^p =\frac{1}{4\pi} \sum_{j,k=1}^{N_v} \Gamma_j\Gamma_k \log (|x_j-x_k|).
    \label{eq:Hp}
\end{equation}

By including the given particle strengths $\Gamma_j$ in Algorithm \ref{alg:int_net},
we can still adopt the method mentioned above to learn the Hamiltonian in (\ref{eq:Hp}) when there are fewer particles.
However, considering a system with $N_v\gg 2$ particles, the cost to collect training data from all $N_v$ particles might be high, and the training process can be time-consuming. Thus, instead of collecting information from all $N_v$ particles to train our model, we only use data collected from two bodies as training data to make predictions of the dynamics of $N_v$ particles.

Specifically, we assume the interactive models between particle pairs with unit particle strengths $\Gamma_j=1$ are the same, and their corresponding Hamiltonian can be represented as network $\hat{\mathcal H}_{\theta}(\bm x_j,\bm x_k)$, based on which the corresponding Hamiltonian of $N_v$ particles can be written as \citep{battaglia2016interaction,sanchez2019hamiltonian}  
\begin{equation}
    \mathcal{H}_{\theta}^{p} = \sum_{i, j = 1}^{N_v} \Gamma_j \Gamma_k \hat{\mathcal H}_{\theta}(\bm x_j,\bm x_k).
    \label{eq:Hptheta}
\end{equation}
We embed (\ref{eq:Hptheta}) into the symplectic integrator that includes $\Gamma_j$ to obtain the final network architecture.
\begin{figure}
  \centering
  \includegraphics[width=1.0\linewidth]{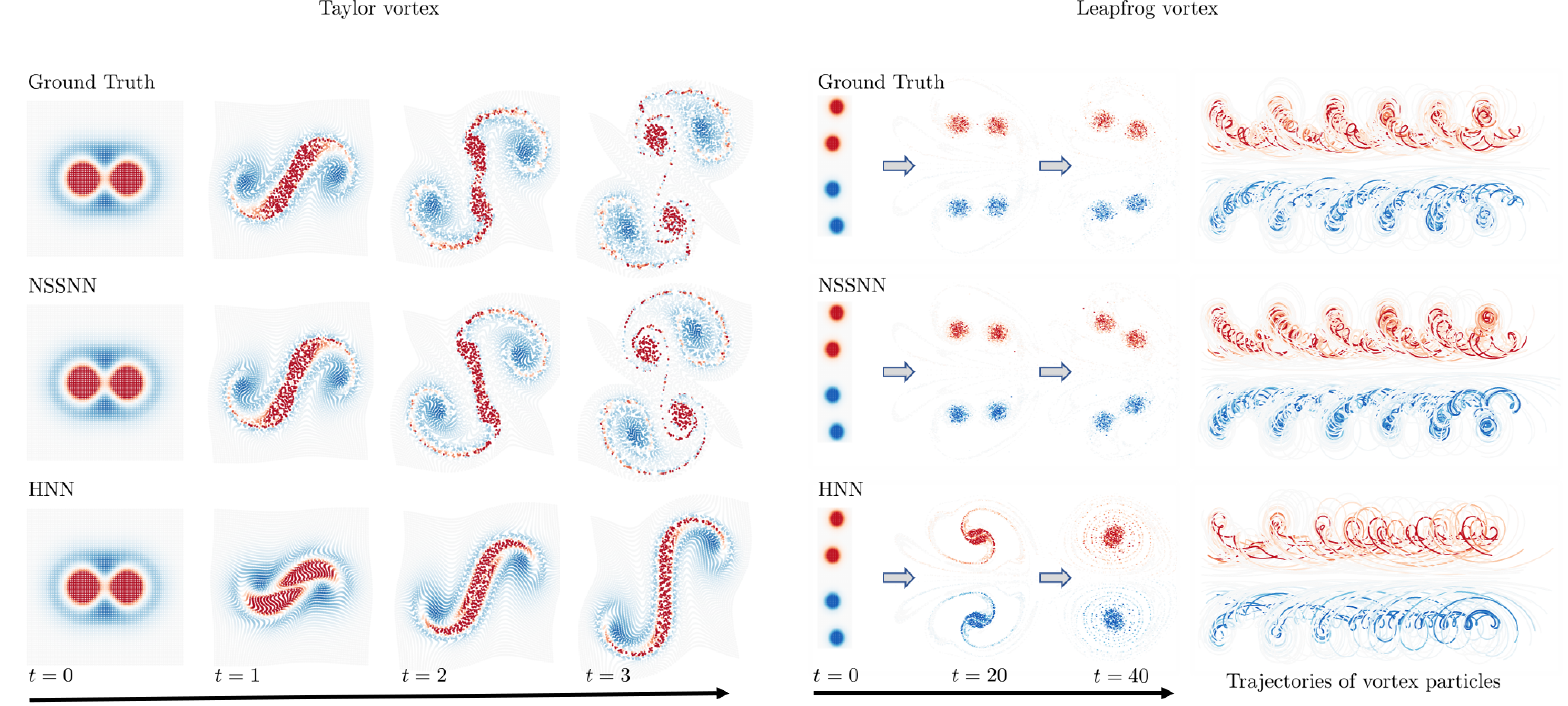}\hfill
  \caption{Taylor and Leapfrog vortex. We generate results of Taylor vortex and Leapfrop vortex using NSSNN and HNN, and compare them with the ground truth. 6000 vortex elements are used with corresponding initial vorticity conditions of Taylor vortex and Leapfrop vortex. Source: \citep{xiong2020nonseparable}.}
  \label{fig:TaylorVortex}
\end{figure}

The setup of the multi-particle problem is similar to the previous problems. The training time span is $T_{train} = 0.01$ while the prediction period can be up to $T_{predict}=40$. We use 2048 clean data samples to train our model. The training process takes about 100 epochs for the loss to converge. In Figure \ref{fig:TaylorVortex}, we use our trained model to predict the dynamics of 6000-particle systems, including Taylor and Leapfrog vortices. We generate results of Taylor vortex and Leapfrop vortex using NSSNN and HNN and compare them with the ground truth. Vortex elements are used with corresponding initial vorticity conditions of Taylor vortex and Leapfrop vortex \citep{Qu2019}. The difficulty of the numerical modeling of these two systems lies in the separation of different dynamical vortices instead of having them merging into a bigger structure. In both cases, the vortices evolved using NSSNN are separated nicely as the ground truth shows, while the vortices merge together using HNN.

\subsection{RoeNet}
\subsubsection{Dataset generation and training settings}

For our experiments, we construct datasets using either analytical solutions or numerical solutions calculated with a high-resolution finite difference method. These datasets are then divided into training and validation sets in a $9:1$ ratio. The physical quantities solved in our experiments are of order $O(1)$ and, consequently, do not require normalization.

We train the network over a time span defined as $T_{train}$ and use it to predict target values over a time span of $T_{predict}$, where $T_{predict} > T_{train}$ and $T_{predict}$ starts no earlier than $T_{train}$.

In all experiments, the Adam optimizer \citep{kingma2014adam} is employed, with a learning rate of $\gamma$ as listed in Table \ref{tab:problems}. The learning rate decays by a multiplicative factor of 0.9 every 5 to 20 epochs. This optimizer is chosen for its ability to adapt learning rates based on the gradient history of each parameter, which facilitates faster and more precise convergence compared to methods with fixed learning rates. Training is conducted with batch sizes ranging from 8 to 32, and all models undergo 100 epochs to ensure convergence. Notably, extending the number of training epochs can enhance training accuracy, reflecting a trade-off between training time and accuracy.

%We believe normalizing the training dataset would be beneficial for a more diverse dataset.
%\zhecheng{We remark that with a sufficiently small time step size and a higher-order time integrator, numerically approximated solutions can achieve high fidelity to the analytical solutions for a relatively short time span.}

\begin{table}
  \caption{Experimental set-up for RoeNet. Source: \citep{tong2024roenet}.}
  \centering
  \begin{tabular}{lcc}
  \hline
   & 1C Linear & Sod Tube \\
  \hline
  Boundary condition & Periodic & Neumann \\
  Time step $\Delta t$ & 0.02 & 0.001 \\
  Space step $\Delta x$ & 0.01 & 0.005 \\
  Training time span & 0.04 & 0.06 \\
  Predicting time span $>$ & 2 & 0.1 \\
  Data set samples & 500 & 2000 \\
  Data set generation & Analytical & Analytical \\
  Components number $N_c$ & 1 & 3 \\
  Hidden dimension $N_h$ & 1 & 64 \\
  %Learning rate $\gamma$ & 0.001 & 0.001 \\
  \hline
  \end{tabular}
\label{tab:problems}
\end{table}

\subsubsection{A simple example}\label{sec:simple}

\begin{figure}[ht]
  \centering
  \includegraphics[width=.5\linewidth]{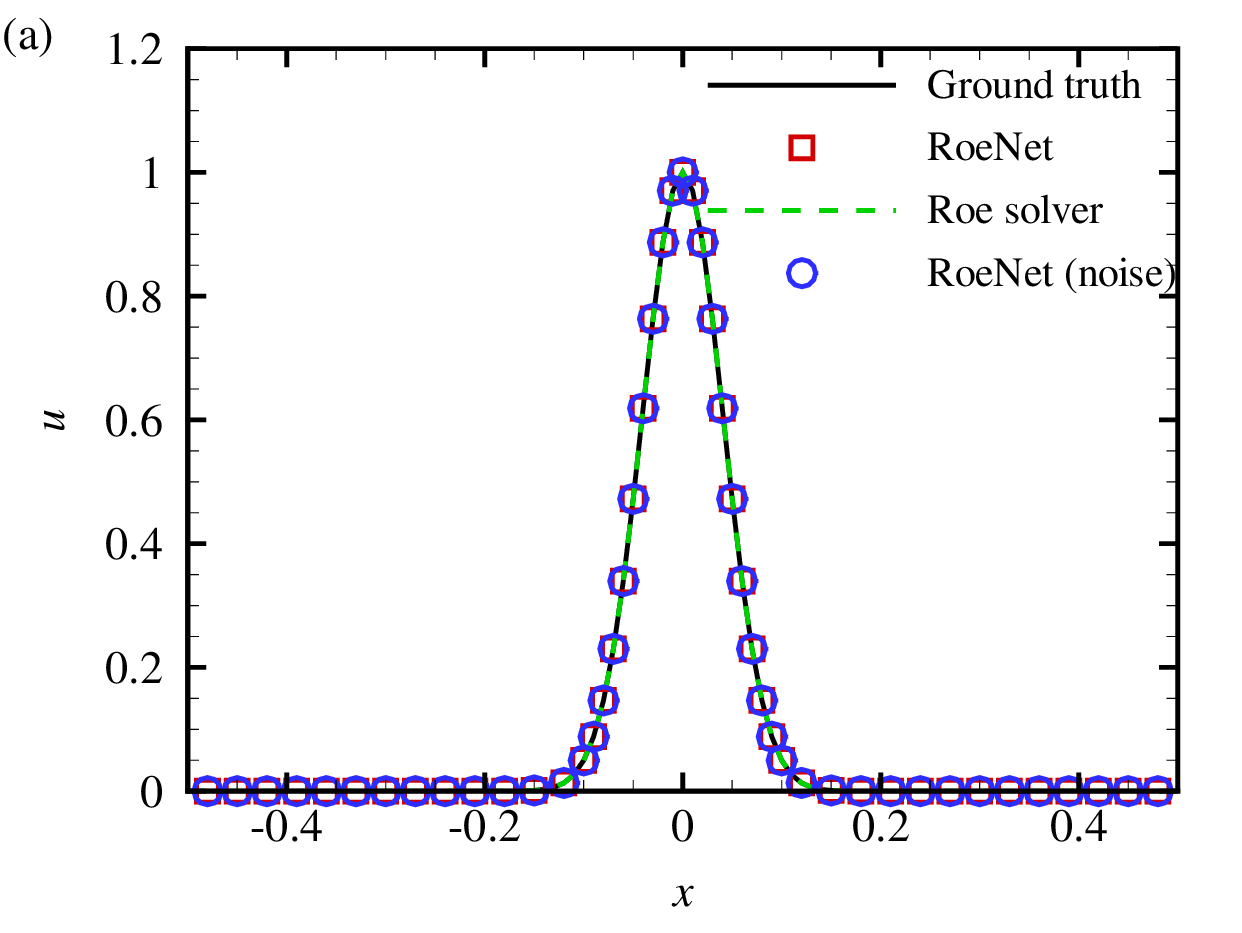}%\hfill
  \includegraphics[width=.5\linewidth]{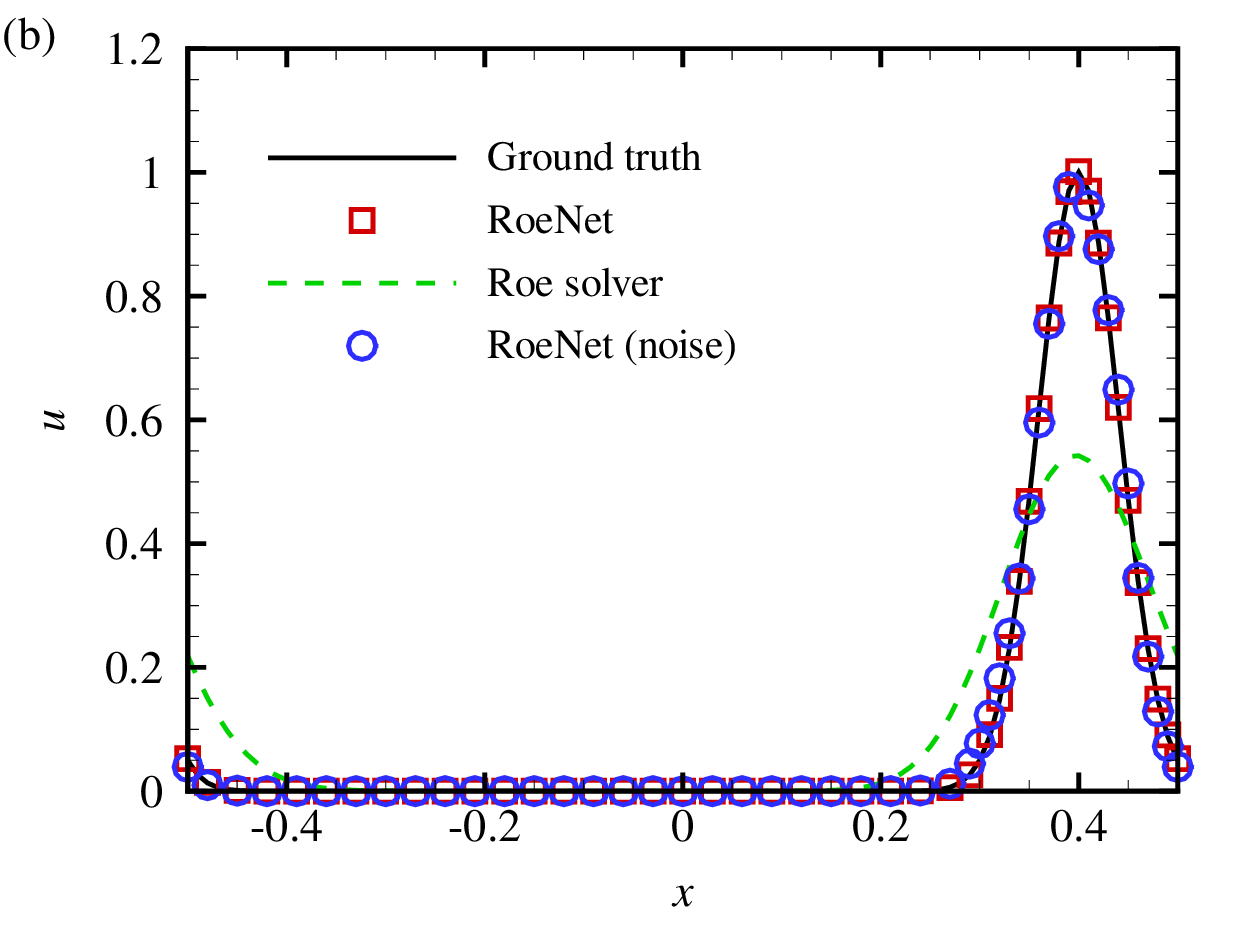}\\
  \includegraphics[width=.5\linewidth]{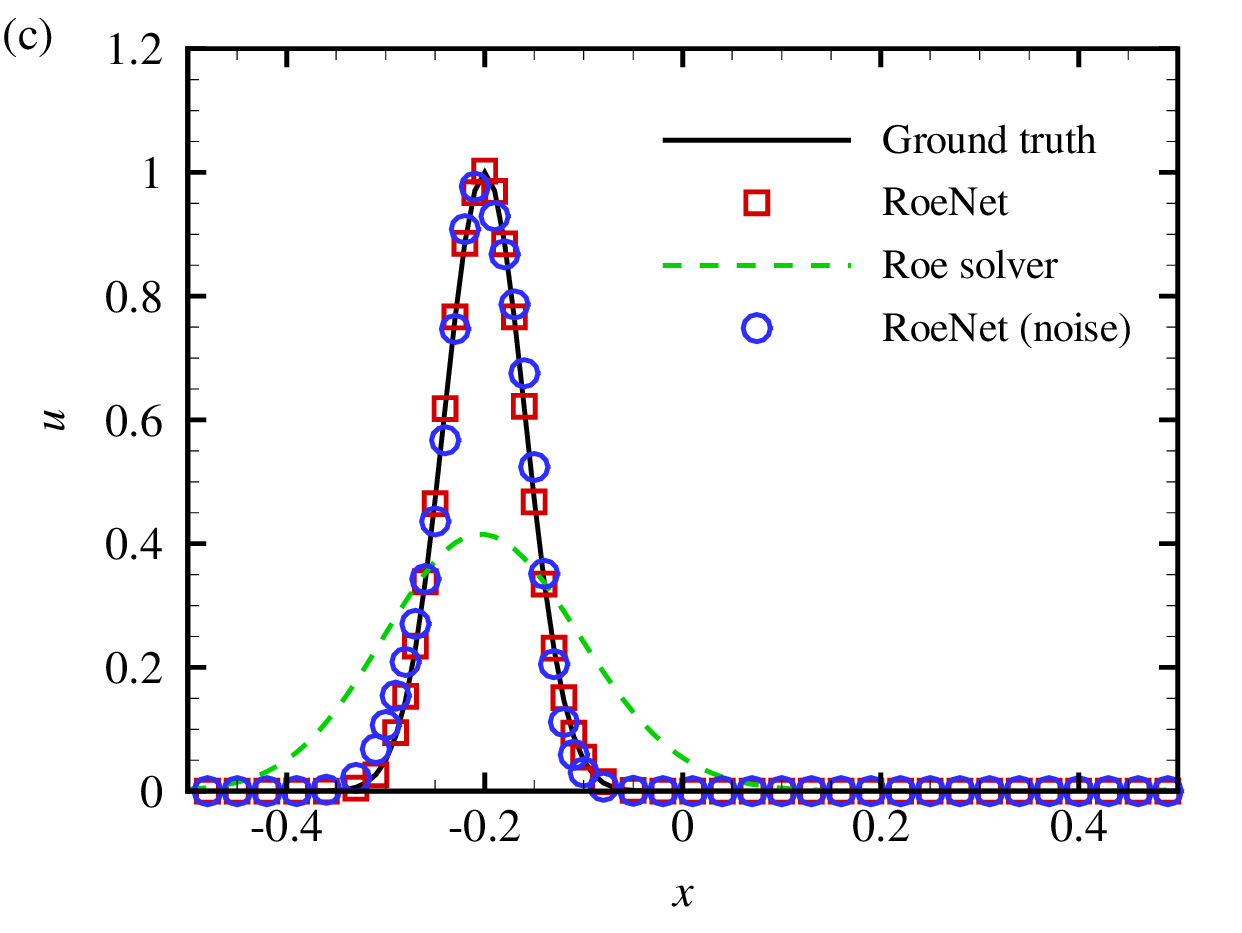}%\hfill
  \includegraphics[width=.5\linewidth]{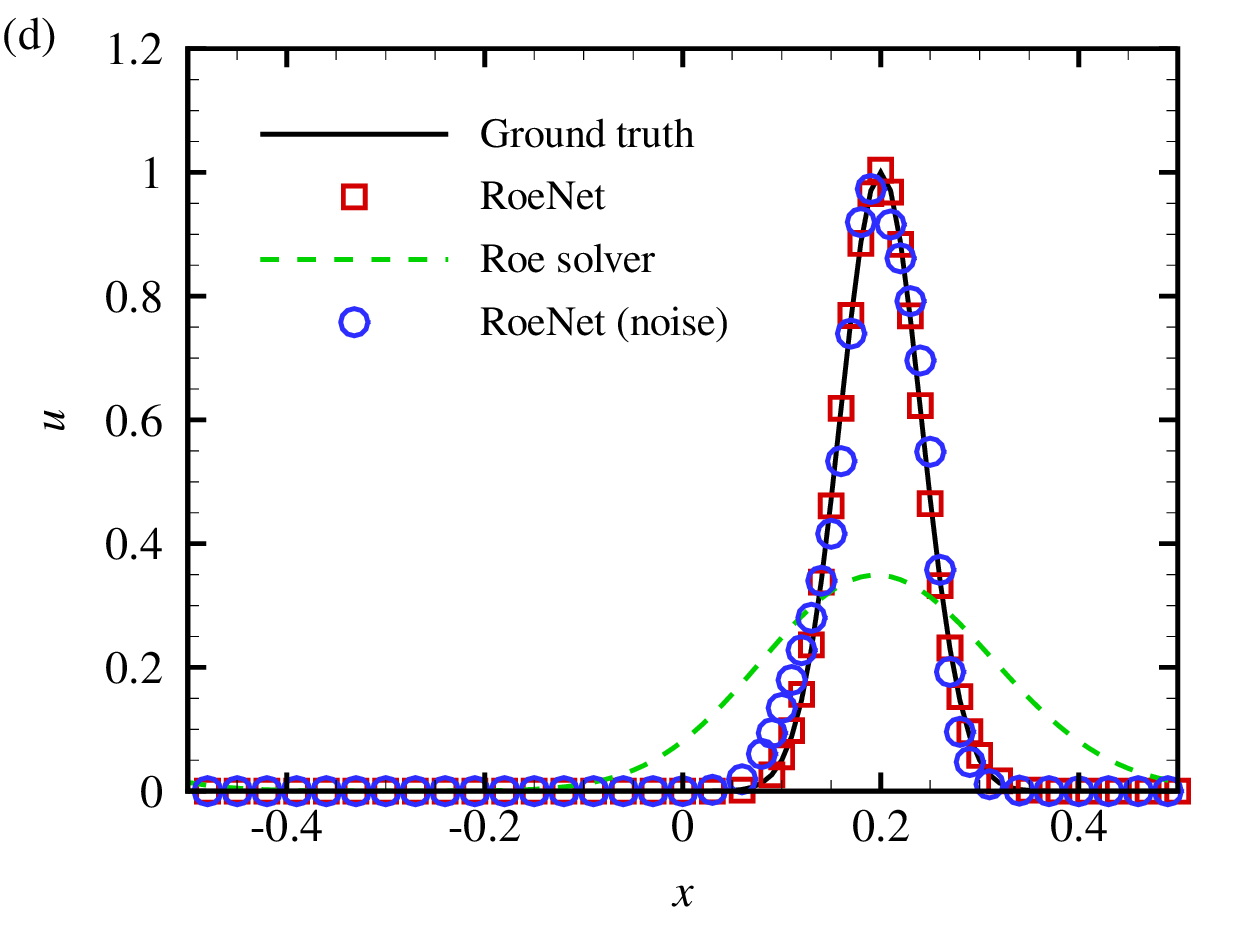}\\
  \caption{Comparison of RoeNet and Roe solver for solving a one component linear hyperbolic PDE (1C Linear in Table \ref{tab:problems}). (a) $t = 0$, (b) $t = 0.4$, (c) $t = 0.8$, (d) $t = 1.2$. The legend ``RoeNet'' and ``RoeNet (noise)'' denote the networks are trained by the clean dataset and the dataset with noise $\epsilon \sim \mathcal{N}(0,0.1)$, respectively. Source: \citep{tong2024roenet}.}
  \label{fig:trivial1C}
\end{figure}

Taking a linear hyperbolic PDE with one component (1C Linear in Table \ref{tab:problems})
\begin{equation}
\begin{dcases}
\bm  F =  u,\\
u(t=0,x) = e^{-300x^2}
\end{dcases}
\label{eq:case1}
\end{equation}
in \eqref{eq:conserv_u}
as an example, we evaluate the performance of RoeNet. 
This hyperbolic PDE models a Gaussian wave traveling along a line at constant speed. Figure \ref{fig:trivial1C} illustrates the propagation of this Gaussian wave over time, simulated using RoeNet with both clean and noisy training data sets, alongside results from the Roe solver and the analytical solution. RoeNet's predictions, regardless of noise in the training data, align closely with the analytical results throughout the entire computational time domain. In contrast, simulations using the Roe solver show rapid flattening and dissipation of the wave over time. Although the prediction error of RoeNet does accumulate gradually, this increase in numerical error is significantly slower than that observed with traditional numerical methods. As a result, RoeNet demonstrates superior performance with its more accurate predictions.

\subsubsection{Sod shock tube}\label{sec:Multi nonlinear}
We take the one-dimensional diatomic ideal gas problem to assess the performance of our model on solving multi-component Riemann problems with nonlinear flux functions (Sode Tube in Table \ref{tab:problems}). Specifically, the system is modeled by \eqref{eq:conserv_u} with
\begin{equation}
\begin{dcases}
\bm u = (\rho,\rho v, e)^T,\\
\bm  F =  [\rho v , \rho v^2+p,v(e+p)]^T,
\end{dcases}
\label{eq:idea}
\end{equation}
where $\rho$ is the density, $p$ is the pressure, $e$ is the energy, $v$ is the velocity, and the pressure $p$ is related to the conserved quantities through the equation of state $p = (\gamma -1)\left(e-0.5 \rho v^2\right)$
with a heat capacity ratio $\gamma\approx 1.4$. We apply our model to the Sod shock tube problem \cite{sod1978survey}, a one-dimensional Riemann problem in the form of \eqref{eq:conserv_u} with \eqref{eq:idea}.
The time evolution of this problem can be described by solving the mass, momentum, and energy conservation of ideal gas inside a slender tube, which leads to three characteristics, describing the propagation speed of various regions in the system \cite{sod1978survey}.
In Figure \ref{fig:nontrivial3c}, we plot the three components of the problem, at $t=0.1$. Note that due to the dissipation effects incorporated in our model, there is no sign of sonic glitch. The result shows that RoeNet exhibits higher accuracy in predicting the discontinuities of the nonlinear Riemann problem.

\begin{figure}[ht]
  \centering
  \includegraphics[width=.33\linewidth]{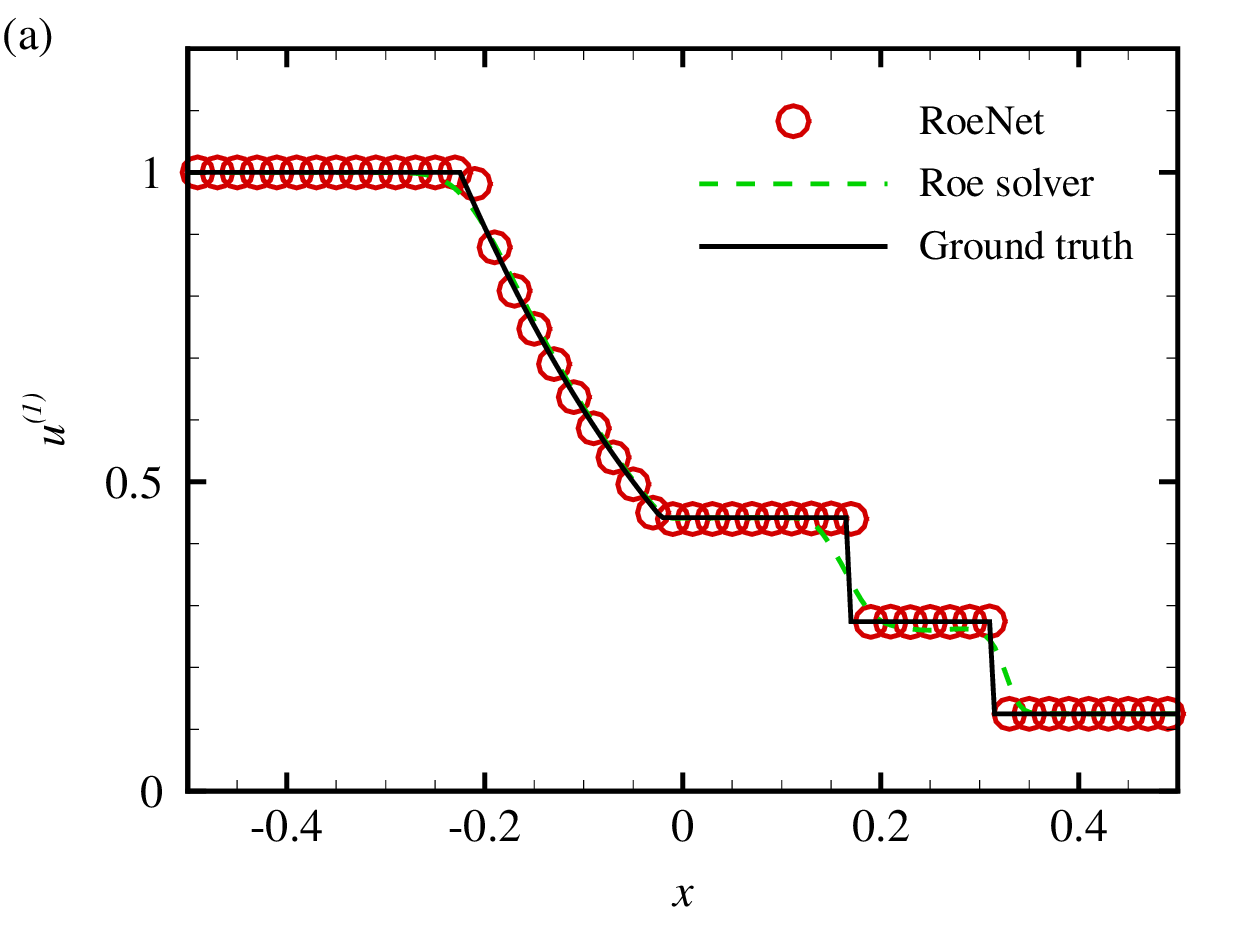}\hfill
  \includegraphics[width=.33\linewidth]{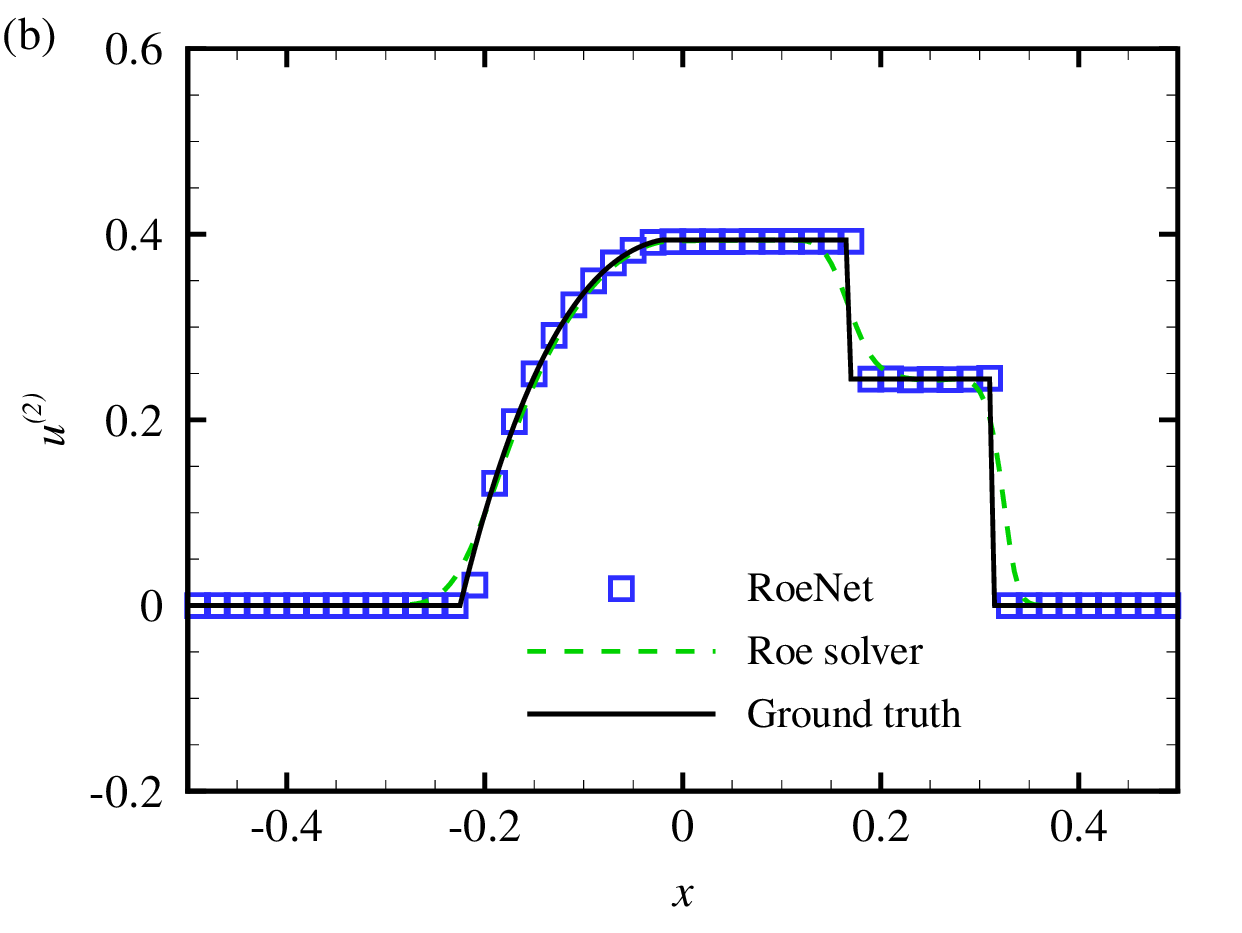}\hfill
  \includegraphics[width=.33\linewidth]{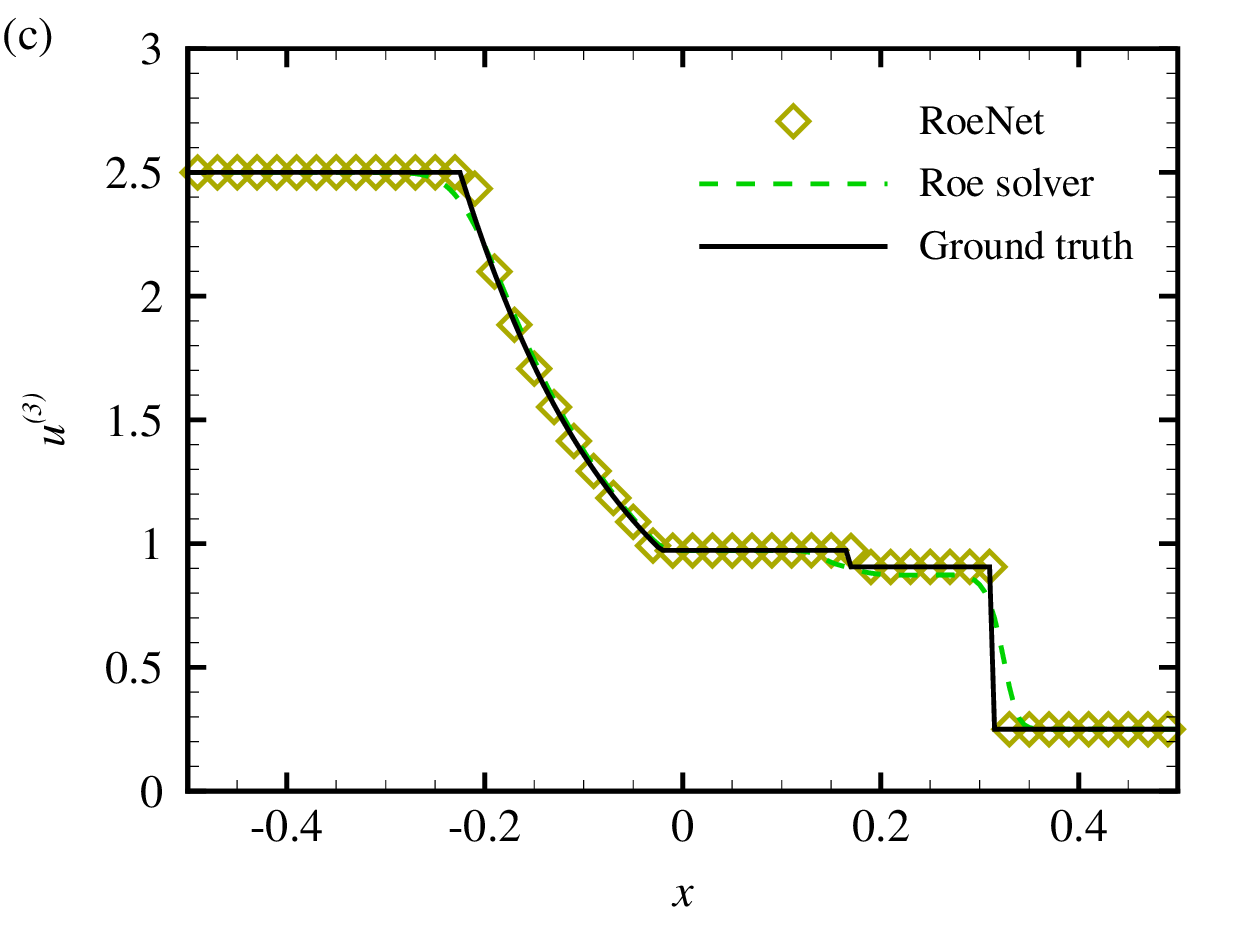}\\
  \caption{Comparison of RoeNet and Roe solver for solving a Riemann problem with three components and a nonlinear flux function (Sod Tube in Table \ref{tab:problems}). (a), (b), and (c) plot the comparison of the prediction results using RoeNet, numerical results solved with Roe solver, and the analytical solutions at $t=0.1$ of the three components $u^{(1)}$, $u^{(2)}$, and $u^{(3)}$, respectively. Source: \citep{tong2024roenet}.}
  \label{fig:nontrivial3c}
\end{figure}

\subsubsection{Comparison with other methods}\label{sec:compare}

Current neural network methods, such as Physics-Informed Neural Networks (PINNs) \cite{Raissi2017}, typically require a pre-established PDE model and continuous interaction with this model during training to adjust the loss, using complex Hessian-based optimizers like L-BFGS that often result in extended training durations. In contrast, RoeNet operates independently of any explicit equation knowledge, utilizing only the training datasets and relying on more efficient gradient-based optimizers such as SGD.

Conventional neural networks struggle to predict the emergence and evolution of discontinuous solutions without a governing equation. Our model, RoeNet, showcases a unique capability to handle tasks that traditional machine learning approaches cannot, particularly in predicting dynamics for future times not included in the training data. This is demonstrated in Figure \ref{fig:DeepXDE}, where RoeNet outperforms PINNs \cite{lu2019DeepXDE} in the simulation of the 1C Linear problem described in Section \ref{sec:simple}, providing accurate predictions for future states beyond the training scope.

\begin{figure}[ht]
  \centering
  \includegraphics[width=1.0\linewidth]{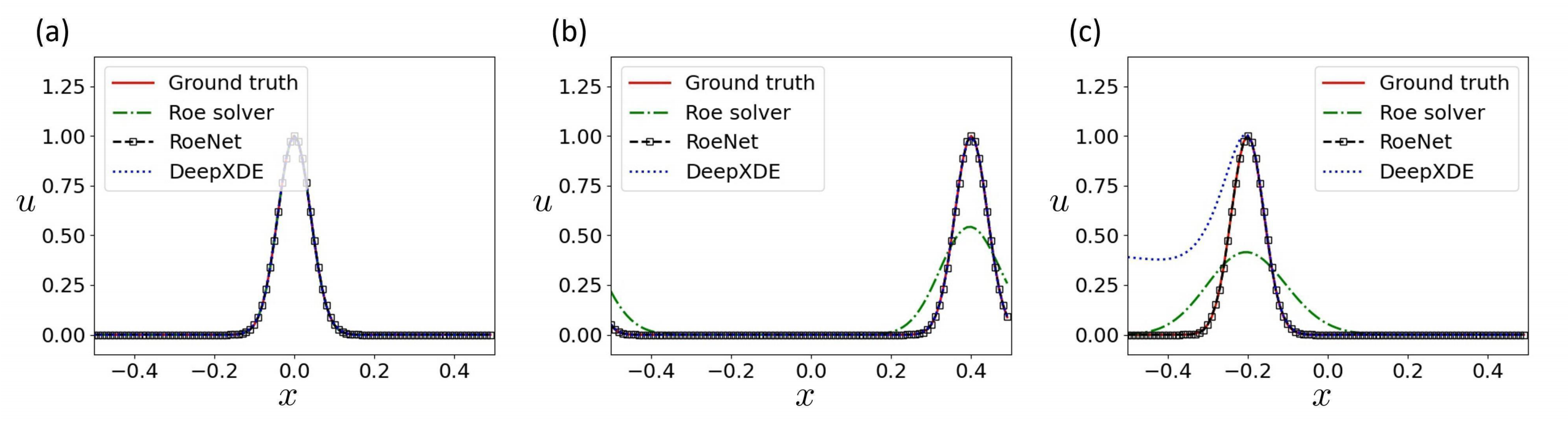}\hfill
  \caption{Comparion of the numerical results solved with Roe solver, prediction results using RoeNet and PINNs \cite{lu2019DeepXDE} at (a) $t = 0$, (b) $t = 0.4$, (c) $t = 0.8$ of the problem 1C Linear in Table \ref{tab:problems}. Source: \citep{tong2024roenet}.}
  \label{fig:DeepXDE}
\end{figure}

RoeNet, as a data-driven solver, does not require prior knowledge of the system's evolution equations, setting it apart from traditional numerical methods. It employs an optimization-based approach to construct its numerical scheme, with an optimization space that fully encompasses that of the Roe solver. This enables RoeNet to deliver more precise simulations of PDE evolution compared to conventional numerical approaches.

\subsection{NVM}
\subsubsection{Dataset generation  and training settings}\label{subsec:data_gen}

We randomly sample 2 to 6 vortices and create the initial vorticity field through convolution with a Gaussian kernel $\sim\mathcal{N}(0,0.01)$. This process is repeated 2000 times to generate $N_s = 2000$ samples. DNS is performed to solve \eqref{eq:uNS} in the periodic box using a standard pseudo-spectral method \citep{Rogallo1981}.
Aliasing errors are removed using the two-thirds truncation method with the maximum wavenumber $k_{\max}\approx N/3$.
The Fourier coefficients of the velocity are advanced in time using a second-order Adams--Bashforth method. The time step is chosen to ensure that the Courant--Friedrichs--Lewy number is less than $0.5$ for numerical stability and accuracy. To obtain accurate DNS data samples, we set the grid size as $N=1024$. Regarding the kinematic viscosity, we set $\nu=0$ and $\nu=0.001$ for different cases.
The pseudo-spectral method used in this DNS is similar to that described in
\citep{Xiong2017,Xiong2019,Xiong2020}.

We use $N_{train}= 0.9 N_s = 1600$ samples with the time span $T_{train}$ for the training of the dynamics network. The DNS dataset is generated with random initial conditions independent of the predicted vortex evolution. The time step of vortex evolution is set as $\textrm{d}t$. For the leapfrog example, we set the parameters as $T_{train}=1$ and $\textrm{d}t=0.001$.
For the turbulent flow example, we set the parameters as $T_{train}=0.001$ and $\textrm{d}t=0.001$.
For other examples, the parameters are set as $T_{train}=0.2$ and $\textrm{d}t=0.1$.
In general, the parameters are chosen within a wide range, indicating the robustness of the network.
We use the trained network to predict the vortex dynamics at time $T_{predict}$. We show that the prediction time span $T_{prediction}$ is larger than the training time span $T_{train}$ in the results section, in some cases up to tens of times of $T_{train}$.

For both the detection network and the dynamics network, we use Adam optimizer \citep{kingma2014adam} with a learning rate of 1e-3. The learning rate decays every 20 epochs by a multiplicative factor of 0.8.
For the detection network, we use a batch size of 32 and train it for 350 epochs. We use the cross entropy as the classification loss and L1 loss for position prediction. To relieve the unbalanced data problem in the detection network, we implement Focal loss \citep{lin2017focal} with $\alpha=0.4$ and $\gamma=2$. It takes 15 minutes to converge on a single Nvidia RTX 2080Ti GPU.
For the dynamics network, we use a batch size of 64 and train it for 500 epochs. We use L1 loss for position prediction. It takes 25 minutes to converge on a single Nvidia RTX 2080Ti GPU.

\subsubsection{Comparison between NVM and LVM}
\begin{figure}
  \centering
  \includegraphics[width=1.0\textwidth]{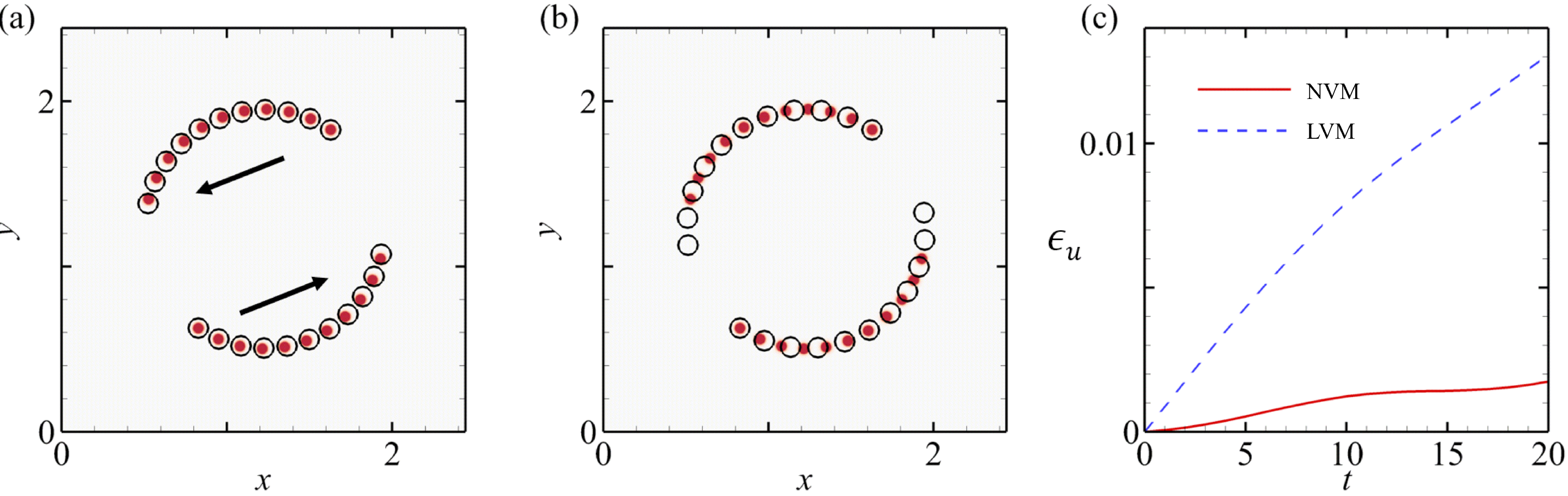}\\
  \caption{Comparison of NVM and LVM for solving NS equations in the periodic box. (a)  NVM, (b) LVM, and (c) The relative error of velocity in flow simulation. The red dots indicate the positions of 2 vortices at different time steps generated by DNS. The black circles in (a) and (b) are the prediction and simulation results of the NVM and LVMs, respectively. The black arrows indicate the directions of the motions of the 2 vortices. Source: \citep{xiong2023neural}. }
  \label{fig:uError}
\end{figure}

To demonstrate that NVM is a better approach to capturing fluid dynamics than the traditional LVM, we compare the prediction results of the NVM and the LVM for solving NS equations in the periodic box.
In the prediction, we initialize two vortex particles at $\bm X_1 = (\pi - 0.4,\pi -0.6)$ and $\bm X_2 = (\pi + 0.4,\pi +0.6)$, where the corresponding particle strength are $\Gamma_1=0.75$ and  $\Gamma_2 =0.75$.
We plot the results using the NVM and LVM and the relative error of velocity in the simulation in Figure \ref{fig:uError} (a), (b), and (c), respectively. Here, the relative error of velocity is defined as
\begin{equation}
\epsilon_u = \frac{\|\bm u_{predict}-\bm  u_{true}\|_{L^2}}{\|\bm  u_{true}\|_{L^2}},
\end{equation}
where $\bm u_{predict}$ denotes the predicted or simulated solution and $\bm  u_{true}$  denote the ground truth solution.

%The red dots indicate the positions of 2 vortices at different time steps generated by DNS, and the black circles are the prediction results of the NVM and LVM.
It is quite obvious that in Figure \ref{fig:uError} (a), the predictions made by NVM match the positions of vortices generated by DNS almost perfectly, while the predictions made by BS law in Figures \ref{fig:uError} (b) contain a large error. The divergence of the relative error of velocity is shown in Figure \ref{fig:uError} (c), which shows that the NVM outperforms traditional methods by increasing amounts as the predicting period becomes longer.

\subsubsection{Turbulent flows}

\begin{figure}
  \centering
  \includegraphics[width=0.8\textwidth]{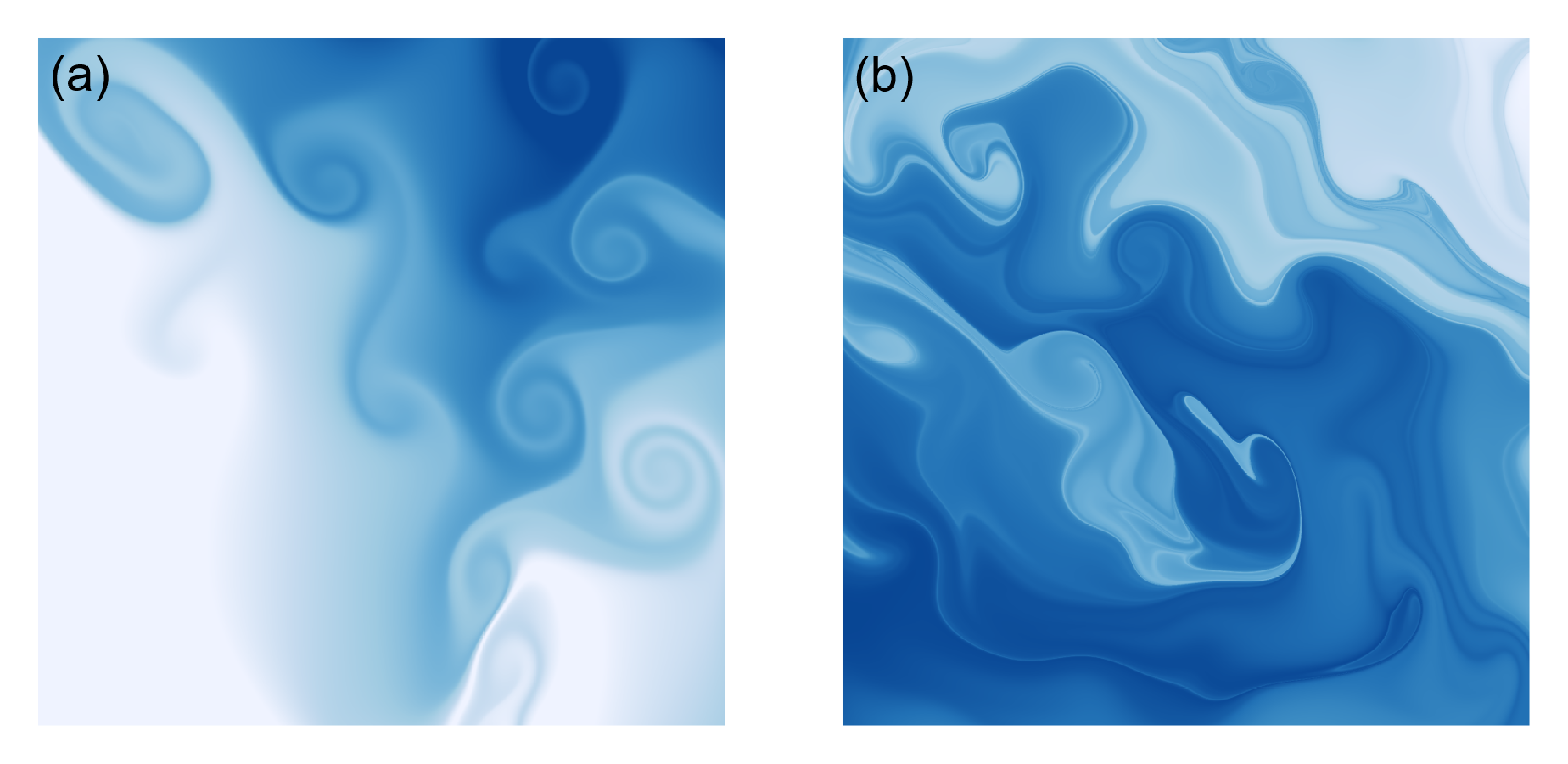}\\
  \caption{Two-dimensional Lagrangian scalar fields at $t=1$ with the initial condition $\phi = x$ and resolution $2000^2$. The evolution of the Lagrangian scalar fields is induced by (a) $O(10)$ and (b) $O(100)$ random NVM vortex particles. Source: \citep{xiong2023neural}.}
  \label{fig:turbulence}
\end{figure}
Besides simple systems, NVM is capable of predicting complicated turbulence systems.
This example's primary purpose is to illustrate our network's ability to handle more complex problems.

Figure \ref{fig:turbulence} depicts the two-dimensional Lagrangian scalar fields at $t=1$ with the initial condition $\phi = x$ and resolution $2000^2$.
The governing equation of the Lagrangian scalar fields is
\begin{equation}
    \frac{\partial \phi}{\partial t}+\bm u \cdot \bm \nabla \phi = 0.
    \label{eq:phi_L}
\end{equation}
The evolution of the Lagrangian scalar fields is induced by $O(10)$ and $O(100)$ NVM vortex particles at random positions $\backsim U(0,4)$ with random strengths $\backsim U(0,2)$.
We remark that the same trained model is used for both cases. There is no correlation between the positions and vortex particle strengths of the two sets of vortex particles.

Based on the particle velocity field from the NVM, a backward-particle-tracking method is applied to
solve \eqref{eq:phi_L}. Then the iso-contour of the Lagrangian field can be extracted as material structures in the evolution \citep{Yang2010a,Yang2011a,ZhaoYangChen2016a,ZhengYangChen2016,Zheng2019}. In Figure \ref{fig:turbulence} (a), the spiral structure \citep{Lundgren1982, Lundgren1993} of individual NVM vortex particles can be observed clearly due to the small number of NVM vortex particles. In Figure \ref{fig:turbulence} (b), the underlying field exhibits turbulent behaviors since it is generated with a large number of NVM vortex particles.

Generally, the high-resolution results shown in Figure \ref{fig:turbulence} can only be achieved by supercomputation using grid-based methods \citep{Yang2010a}, while NVM allows these to be generated on any laptop with GPU. %Although the LVM does not require such a high computational cost as grid methods, suffering from inaccuracy, it can never produce such an accurate depiction of Lagrangian fields.
We demonstrate that NVM is capable of generating an accurate depiction of complex turbulence systems with low computational costs.

\section{Conclusion}

\subsection{Summary}

This thesis introduces a novel data-driven framework, which demonstrates a significant advancement in predictive modeling for long-term forecasts by integrating physics-based priors into learning algorithms. This integration ensures intrinsic preservation of the physical structures of the systems analyzed, thereby maintaining mathematical symmetries and physical conservation laws. As a result, the models demonstrate superior performance in terms of prediction accuracy, robustness, and predictive capability, particularly in identifying patterns not present within the training dataset, despite the use of small datasets, short training periods, and small sample sizes.

In particular, we have developed four distinct algorithms, each designed to incorporate specific physics-based priors relevant to different types of nonlinear systems. These include the symplectic structure for both separable and nonseparable Hamiltonian systems, Hyperbolic Conservation Law for hyperbolic partial differential equations, and Helmholtz’s Theorem for incompressible fluid dynamics. The integration of physics-based priors not only narrows the solution space, thereby streamlining computational demands, but also enhances the reliability and validity of the predictions. Moreover, embedding these structures within neural networks significantly expands their capacity to capture and reproduce complex patterns inherent in physical phenomena, which conventional networks often fail to recognize. This expanded capability allows for a more comprehensive representation of potential physical behaviors, substantially improving the models' applicability and predictive accuracy.

\subsection{Limitations and Future Work}
We also recognize our models have several limitations. Firstly, neural networks that include an embedded integrator often require a longer training period compared to those trained on datasets with explicit time derivatives. Secondly, our method employs an explicit scheme for time evolution, which necessitates a small time step to ensure accuracy. Although a smaller time step can lead to higher discretization accuracy, this advantage must be weighed against increased training costs and the risk of gradient explosion. In our future work, we are considering the adoption of implicit formats, such as leveraging RNN structures, which may offer more stability and efficiency. In addition, our current model is designed as an end-to-end system that does not account for environmental variability. To address this issue, we will explore online learning techniques to enhance the model's adaptability in changing conditions. Lastly, To enhance the applicability of our model, a significant focus of our future research will be dedicated to developing scalable methods that can be generalized to various PDEs, aiming to achieve a versatile and universally applicable framework for various systems.

\bibliographystyle{unsrtnat}
\bibliography{refs}

\begin{thebibliography}{80}
\providecommand{\natexlab}[1]{#1}
\providecommand{\url}[1]{\texttt{#1}}
\expandafter\ifx\csname urlstyle\endcsname\relax
  \providecommand{\doi}[1]{doi: #1}\else
  \providecommand{\doi}{doi: \begingroup \urlstyle{rm}\Url}\fi

\bibitem[Tong et~al.(2021)Tong, Xiong, He, Pan, and Zhu]{tong2021symplectic}
Yunjin Tong, Shiying Xiong, Xingzhe He, Guanghan Pan, and Bo~Zhu.
\newblock Symplectic neural networks in taylor series form for hamiltonian systems.
\newblock \emph{Journal of Computational Physics}, 437:\penalty0 110325, 2021.

\bibitem[Xiong et~al.(2020)Xiong, Tong, He, Yang, Yang, and Zhu]{xiong2020nonseparable}
Shiying Xiong, Yunjin Tong, Xingzhe He, Shuqi Yang, Cheng Yang, and Bo~Zhu.
\newblock Nonseparable symplectic neural networks.
\newblock \emph{arXiv preprint arXiv:2010.12636}, 2020.

\bibitem[Xiong et~al.(2023)Xiong, He, Tong, Deng, and Zhu]{xiong2023neural}
Shiying Xiong, Xingzhe He, Yunjin Tong, Yitong Deng, and Bo~Zhu.
\newblock Neural vortex method: From finite lagrangian particles to infinite dimensional eulerian dynamics.
\newblock \emph{Computers \& Fluids}, 258:\penalty0 105811, 2023.

\bibitem[Tong et~al.(2024)Tong, Xiong, He, Yang, Wang, Tao, Liu, and Zhu]{tong2024roenet}
Yunjin Tong, Shiying Xiong, Xingzhe He, Shuqi Yang, Zhecheng Wang, Rui Tao, Runze Liu, and Bo~Zhu.
\newblock Roenet: Predicting discontinuity of hyperbolic systems from continuous data.
\newblock \emph{International Journal for Numerical Methods in Engineering}, 125\penalty0 (6):\penalty0 e7406, 2024.

\bibitem[Weinan(2021)]{weinan2021dawning}
E~Weinan.
\newblock The dawning of a new era in applied mathematics.
\newblock \emph{Notices of the American Mathematical Society}, 68\penalty0 (4):\penalty0 565--571, 2021.

\bibitem[Brunton et~al.(2020)Brunton, Noack, and Koumoutsakos]{Brunton2020}
S.~L. Brunton, B.~R. Noack, and P.~Koumoutsakos.
\newblock {Machine Learning for Fluid Mechanics}.
\newblock \emph{Annu. Rev. Fluid Mech.}, 52:\penalty0 477--508, 2020.

\bibitem[Hughes et~al.(2019)Hughes, Williamson, Minkov, and Fan]{Hughes2019}
T.~W. Hughes, I.~A.~D. Williamson, M.~Minkov, and S.~Fan.
\newblock {Wave physics as an analog recurrent neural network}.
\newblock \emph{Sci. Adv.}, 5:\penalty0 6946, 2019.

\bibitem[Sellier et~al.(2019)Sellier, Caron, and Leygonie]{Sellier2019}
J.~M. Sellier, G.~M. Caron, and J.~Leygonie.
\newblock Signed particles and neural networks, towards efficient simulations of quantum systems.
\newblock \emph{J. Comput. Phys.}, 387:\penalty0 154--162, 2019.

\bibitem[Hernandez et~al.(2020)Hernandez, Badias, Gonzalez, Chinesta, and Cueto]{hernandez2020}
Quercus Hernandez, Alberto Badias, David Gonzalez, Francisco Chinesta, and Elias Cueto.
\newblock Structure-preserving neural networks.
\newblock \emph{arXiv:2004.04653}, 2020.

\bibitem[Teicherta et~al.(2019)Teicherta, Natarajanc, der Venc, and Garikipati]{Teicherta2019}
G.~H. Teicherta, A.~R. Natarajanc, A.~Van der Venc, and K.~Garikipati.
\newblock {Machine learning materials physics: Integrable deep neural networks enable scale bridging by learning free energy functions}.
\newblock \emph{Comput. Methods Appl. Mech. Engrg.}, 353:\penalty0 201--216, 2019.

\bibitem[Regazzoni et~al.(2019)Regazzoni, Ded{\'e}, and Quarteroni]{regazzoni2019}
F~Regazzoni, L~Ded{\'e}, and A~Quarteroni.
\newblock Machine learning for fast and reliable solution of time-dependent differential equations.
\newblock \emph{J. Comput. Phys.}, 397:\penalty0 108852, 2019.

\bibitem[Raissi and Karniadakis(2018)]{Raissi2018}
M.~Raissi and G.~E. Karniadakis.
\newblock {Hidden physics models: Machine learning of nonlinear partial differential equations}.
\newblock \emph{J. Comput. Phys.}, 357:\penalty0 125--141, 2018.

\bibitem[Sirignano and Spiliopoulos(2018)]{Sirignano2018}
J.~Sirignano and K.~Spiliopoulos.
\newblock {DGM: A deep learning algorithm for solving partial differential equations}.
\newblock \emph{J. Comput. Phys.}, 375:\penalty0 686--707, 2018.

\bibitem[Raissi et~al.(2019)Raissi, Perdikaris, and Karniadakis]{Raissi2019}
M.~Raissi, P.~Perdikaris, and G.~E. Karniadakis.
\newblock {Physics-informed neural networks: A deep learning framework for solving forward and inverse problems involving nonlinear partial differential equations}.
\newblock \emph{J. Comput. Phys.}, 378:\penalty0 686--707, 2019.

\bibitem[Greydanus et~al.(2019)Greydanus, Dzamba, and Yosinski]{Greydanus2019}
S.~Greydanus, M.~Dzamba, and J.~Yosinski.
\newblock Hamiltonian neural networks.
\newblock In \emph{Conference on Neural Information Processing Systems}, pages 15379--15389, 2019.

\bibitem[Chen et~al.(2020)Chen, Zhang, Arjovsky, and Bottou]{Chen2020}
Z.~Chen, J.~Zhang, M.~Arjovsky, and L.~Bottou.
\newblock {Symplectic recurrent neural networks}.
\newblock In \emph{International Conference on Learning Representations}, 2020.

\bibitem[DiPietro et~al.(2020)DiPietro, Xiong, and Zhu]{dipietro2020sparse}
D.~DiPietro, S.~Xiong, and B.~Zhu.
\newblock {Sparse symplectically integrated neural networks}.
\newblock In \emph{Advances in Neural Information Processing Systems}, 2020.

\bibitem[Sanchez-Gonzalez et~al.(2019)Sanchez-Gonzalez, Bapst, Cranmer, and Battaglia]{sanchez2019hamiltonian}
A.~Sanchez-Gonzalez, V.~Bapst, K.~Cranmer, and P.~Battaglia.
\newblock {Hamiltonian graph networks with ODE integrators}.
\newblock \emph{arXiv:1909.12790}, 2019.

\bibitem[Battaglia et~al.(2016)Battaglia, Pascanu, Lai, and Rezende]{battaglia2016interaction}
P.~Battaglia, R.~Pascanu, M.~Lai, and D.~J. Rezende.
\newblock {Interaction networks for learning about objects, relations and physics}.
\newblock In \emph{Advances in Neural Information Processing Systems}, pages 4502--4510, 2016.

\bibitem[Jin et~al.(2020)Jin, Zhu, Karniadakis, and Tang]{Jin2020}
P.~Jin, A.~Zhu, G.~E. Karniadakis, and Y.~Tang.
\newblock {Symplectic networks: intrinsic structure-preserving networks for identifying Hamiltonian systems}.
\newblock \emph{arXiv:2001.03750}, 2020.

\bibitem[Toth et~al.(2020)Toth, Rezende, Jaegle, Racani\'ere, Botev, and Higgins]{Toth2020}
P.~Toth, D.~J. Rezende, A.~Jaegle, S.~Racani\'ere, A.~Botev, and I.~Higgins.
\newblock {Hamiltonian generative networks}.
\newblock In \emph{International Conference on Learning Representations}, 2020.

\bibitem[Zhong et~al.(2020)Zhong, Dey, and Chakraborty]{Zhong2020Symplectic}
Y.~D. Zhong, B.~Dey, and A.~Chakraborty.
\newblock {Symplectic ODE-Net: learning Hamiltonian dynamics with control}.
\newblock In \emph{International Conference on Learning Representations}, 2020.

\bibitem[Yarosky(2017)]{Yarosky2016}
D.~Yarosky.
\newblock {Error bounds for approximations with deep ReLU networks}.
\newblock \emph{Neural Netw.}, 94:\penalty0 103--114, 2017.

\bibitem[Petersen and Voigtl{\"a}nder(2018)]{Petersen2017}
P.~Petersen and F.~Voigtl{\"a}nder.
\newblock {Optimal approximation of piecewise smooth functions using deep ReLU neural networks}.
\newblock \emph{Neural Netw.}, 170:\penalty0 296--330, 2018.

\bibitem[Imaizumi and Fukumizu(2019)]{Imaizumi2019}
M.~Imaizumi and K.~Fukumizu.
\newblock Deep learning networks learn non-smooth functions effectively.
\newblock In \emph{The Institute of Statistical Mathematics}, pages 869--878. The 22nd International Conference on Artificial Intelligence and Statistics, 2019.

\bibitem[Suzuki(2019)]{Suzuki2018}
T.~Suzuki.
\newblock Adaptivity of deep relu network for learning in besov and mixed smooth besov spaces: Optimal rate and curse of dimensionality.
\newblock In \emph{The University of Tokyo}. International Conference on Learning Representations, 2019.

\bibitem[Raissi et~al.(2017)Raissi, Perdikaris, and Karniadakis]{Raissi2017}
M.~Raissi, P.~Perdikaris, and G.~E. Karniadakis.
\newblock {Inferring solutions of differential equations using noisy multi-fidelity data}.
\newblock \emph{J. Comput. Phys.}, 335:\penalty0 736--746, 2017.

\bibitem[Hornik et~al.(1989)Hornik, Stinchcombe, and Halbert]{Hornik1989}
K.~Hornik, M.~Stinchcombe, and W.~Halbert.
\newblock Multilayer feedforward networks are universal approximators.
\newblock \emph{Neural Netw.}, 2:\penalty0 359--366, 1989.

\bibitem[Zhang et~al.(2019)Zhang, Guo, and Karniadakis]{Zhang2019}
D.~Zhang, L.~Guo, and G.~E. Karniadakis.
\newblock {Learning in modal space: solving time-dependent stochastic PDEs using physics-informed neural networks}.
\newblock \emph{SIAM J. Sci. Comput.}, 42:\penalty0 A639--A665, 2019.

\bibitem[Michoski et~al.(2019)Michoski, Milosavljevic, Oliver, and Hatch]{Michoski2019}
C.~Michoski, M.~Milosavljevic, T.~Oliver, and D.~Hatch.
\newblock {Solving differential equations using deep neural networks}.
\newblock \emph{Neurocomputing}, 399:\penalty0 193--212, 2019.

\bibitem[Mao et~al.(2020)Mao, Jagtap, and Karniadakis]{Mao2020}
Z.~Mao, A.~D. Jagtap, and G.~E. Karniadakis.
\newblock Physics-informed neural networks for high-speed flows.
\newblock \emph{Comput. Method. Appl. M.}, 360:\penalty0 112789, 2020.

\bibitem[Duraisamy et~al.(2019)Duraisamy, Iaccarino, and Xiao]{duraisamy2019turbulence}
K.~Duraisamy, G.~Iaccarino, and H.~Xiao.
\newblock Turbulence modeling in the age of data.
\newblock \emph{Annu. Rev. Fluid Mech.}, 51:\penalty0 357--377, 2019.

\bibitem[Xie et~al.(2018)Xie, Franz, Chu, and Thuerey]{xie2018tempogan}
Y.~Xie, E.~Franz, M.~Chu, and N.~Thuerey.
\newblock tempogan: A temporally coherent, volumetric gan for super-resolution fluid flow.
\newblock \emph{ACM Trans. Graph.}, 37\penalty0 (4):\penalty0 1--15, 2018.

\bibitem[Chu and Thuerey(2017)]{chu2017data}
M.~Chu and N.~Thuerey.
\newblock Data-driven synthesis of smoke flows with cnn-based feature descriptors.
\newblock \emph{ACM Trans. Graph.}, 36\penalty0 (4):\penalty0 1--14, 2017.

\bibitem[Anderson et~al.(1996)Anderson, Kevrekidis, and Rico-Martinez]{anderson1996comparison}
J.~Anderson, I.~Kevrekidis, and R.~Rico-Martinez.
\newblock A comparison of recurrent training algorithms for time series analysis and system identification.
\newblock \emph{Comput. Chem. Eng.}, 20:\penalty0 S751--S756, 1996.

\bibitem[Crutchfield and McNamara(1987)]{crutchfield1987equations}
James~P Crutchfield and Bruce~S McNamara.
\newblock Equations of motion from a data series.
\newblock \emph{Complex Syst.}, 1\penalty0 (417-452):\penalty0 121, 1987.

\bibitem[Daniels and Nemenman(2015)]{daniels2015automated}
Bryan~C Daniels and Ilya Nemenman.
\newblock Automated adaptive inference of phenomenological dynamical models.
\newblock \emph{Nat. Commun.}, 6\penalty0 (1):\penalty0 1--8, 2015.

\bibitem[Wang et~al.(2017)Wang, Wu, and Xiao]{wang2017physics}
J.~Wang, J.~Wu, and H.~Xiao.
\newblock Physics-informed machine learning approach for reconstructing reynolds stress modeling discrepancies based on dns data.
\newblock \emph{Phys. Rev. Fluids}, 2\penalty0 (3):\penalty0 034603, 2017.

\bibitem[Hammond et~al.(2022)Hammond, Montomoli, Pietropaoli, Sandberg, and M.]{hammond2022machine}
J.~Hammond, F.~Montomoli, M.~Pietropaoli, R.~D. Sandberg, and V.~M.
\newblock Machine learning for the development of data-driven turbulence closures in coolant systems.
\newblock \emph{J. Turbomach.}, 144\penalty0 (8):\penalty0 081003, 2022.

\bibitem[Xu et~al.(2022)Xu, Waschkowski, Ooi, and Sandberg]{xu2022towards}
X.~Xu, F.~Waschkowski, A.~S. Ooi, and R.~D. Sandberg.
\newblock Towards robust and accurate reynolds-averaged closures for natural convection via multi-objective cfd-driven machine learning.
\newblock \emph{Int. J. Heat Mass Transf.}, 187:\penalty0 122557, 2022.

\bibitem[Mohan et~al.(2020{\natexlab{a}})Mohan, Lubbers, Livescu, and Chertkov]{mohan2020}
A.~T. Mohan, N.~Lubbers, D.~Livescu, and M.~Chertkov.
\newblock {Embedding hard physical constraints in convolutional neural networks for 3D turbulence}.
\newblock In \emph{International Conference on Learning Representations}, 2020{\natexlab{a}}.

\bibitem[Yang et~al.(2019)Yang, Zafar, Wang, and Xiao]{yang2019predictive}
X.~Yang, S.~Zafar, J.~Wang, and H.~Xiao.
\newblock Predictive large-eddy-simulation wall modeling via physics-informed neural networks.
\newblock \emph{Phys. Rev. Fluids}, 4:\penalty0 034602, 2019.

\bibitem[Raissi et~al.(2020)Raissi, Yazdani, and Karniadakis]{raissi2020hidden}
M.~Raissi, A.~Yazdani, and G.~E. Karniadakis.
\newblock Hidden fluid mechanics: Learning velocity and pressure fields from flow visualizations.
\newblock \emph{Science}, 367\penalty0 (6481):\penalty0 1026--1030, 2020.

\bibitem[Belbute-Peres et~al.(2020)Belbute-Peres, Economon, and Kolter]{belbute2020combining}
F.~Belbute-Peres, T.~Economon, and Z.~Kolter.
\newblock Combining differentiable pde solvers and graph neural networks for fluid flow prediction.
\newblock In \emph{International Conference on Machine Learning}, pages 2402--2411, 2020.

\bibitem[Lye et~al.(2020)Lye, Mishra, and Ray]{lye2020deep}
K.~Lye, S.~Mishra, and D.~Ray.
\newblock Deep learning observables in computational fluid dynamics.
\newblock \emph{J. Comput. Phys.}, 410:\penalty0 109339, 2020.

\bibitem[White et~al.(2019)White, Ushizima, and Farhat]{white2019neural}
Cristina White, Daniela Ushizima, and Charbel Farhat.
\newblock Neural networks predict fluid dynamics solutions from tiny datasets.
\newblock \emph{arXiv preprint arXiv:1902.00091}, 2019.

\bibitem[Mohan et~al.(2020{\natexlab{b}})Mohan, Lubbers, Livescu, and Chertkov]{mohan2020embedding}
Arvind~T Mohan, Nicholas Lubbers, Daniel Livescu, and Michael Chertkov.
\newblock Embedding hard physical constraints in neural network coarse-graining of 3d turbulence.
\newblock \emph{arXiv preprint arXiv:2002.00021}, 2020{\natexlab{b}}.

\bibitem[Kingma and Ba(2014)]{kingma2014adam}
D.~P. Kingma and J.~Ba.
\newblock {Adam: A method for stochastic optimization}.
\newblock \emph{arXiv:1412.6980}, 2014.

\bibitem[He et~al.(2016)He, Zhang, Ren, and Sun]{he2016resnet}
K.~He, X.~Zhang, S.~Ren, and J.~Sun.
\newblock {Deep residual learning for image recognition}.
\newblock In \emph{Proceedings of the IEEE Conference on Computer Vision and Pattern Recognitionn}, pages 770--778, 2016.

\bibitem[Chen et~al.(2018)Chen, Rubanova, Bettencourt, and Duvenaud]{chen2018neural}
R.~T.~Q. Chen, Y.~Rubanova, J.~Bettencourt, and D.~Duvenaud.
\newblock {Neural ordinary differential equations}.
\newblock In \emph{Conference on Neural Information Processing Systems}, pages 6571--6583, 2018.

\bibitem[Pontryagin(2018)]{pontryagin2018mathematical}
Lev~Semenovich Pontryagin.
\newblock \emph{Mathematical theory of optimal processes}.
\newblock Routledge, 2018.

\bibitem[Forest and Ruth(1990)]{Forest1990}
E.~Forest and R.~D. Ruth.
\newblock {Fourth-order symplectic integration}.
\newblock \emph{Physica D}, 43:\penalty0 105--117, 1990.

\bibitem[Yoshida(1990)]{Yoshida1990}
H.~Yoshida.
\newblock {Construction of higher order symplectic integrators}.
\newblock \emph{Phys. Lett. A}, 150:\penalty0 262--268, 1990.

\bibitem[Candy and Rozmus(1991)]{Candy1991}
J.~Candy and W.~Rozmus.
\newblock {A symplectic integration algorithm for separable Hamiltonian functions}.
\newblock \emph{J. Comput. Phys.}, 92:\penalty0 230--256, 1991.

\bibitem[Tao(2016)]{Tao2016}
Molei Tao.
\newblock Explicit symplectic approximation of nonseparable hamiltonians: Algorithm and long time performance.
\newblock \emph{Physical Review E}, 94\penalty0 (4):\penalty0 043303, 2016.

\bibitem[Wu et~al.(2015)Wu, Ma, and Zhou]{WuMaZhou2015}
J.~Z. Wu, H.~Y. Ma, and M.~D. Zhou.
\newblock \emph{{Vortical Flows}}.
\newblock Springer, 2015.

\bibitem[Evans(2010)]{Evans2010}
L.~C. Evans.
\newblock \emph{{Partial Differential Equations}}.
\newblock American Mathematical Society, 2 edition, 2010.

\bibitem[Roe(1981)]{Roe1981}
P.~L. Roe.
\newblock {Approximate riemann solvers, parameter vectors and difference schemes}.
\newblock \emph{J. Comput. Phys.}, 43:\penalty0 357--372, 1981.

\bibitem[Helmholtz(1858)]{Helmholtz1858}
H.~Helmholtz.
\newblock {Uber integrale der hydrodynamischen Gleichungen welche den Wirbel-bewegungen ensprechen}.
\newblock \emph{J. Reine Angew. Math}, 55:\penalty0 25--55, 1858.

\bibitem[Yang and Pullin(2010)]{Yang2010b}
Y.~Yang and D.~I. Pullin.
\newblock {On Lagrangian and vortex-surface fields for flows with Taylor--Green and Kida--Pelz initial conditions}.
\newblock \emph{J. Fluid Mech.}, 661:\penalty0 446--481, 2010.

\bibitem[Xiong and Yang(2017)]{Xiong2017}
S.~Xiong and Y.~Yang.
\newblock {The boundary-constraint method for constructing vortex-surface fields}.
\newblock \emph{J. Comput. Phys.}, 339:\penalty0 31--45, 2017.

\bibitem[Hao et~al.(2019)Hao, Xiong, and Yang]{Hao2019}
J.~Hao, S.~Xiong, and Y.~Yang.
\newblock {Tracking vortex surfaces frozen in the virtual velocity in non-ideal flows}.
\newblock \emph{J. Fluid Mech.}, 863:\penalty0 513--544, 2019.

\bibitem[Cottet and Koumoutsakos(2000)]{Cottet2000}
G.H. Cottet and P.D. Koumoutsakos.
\newblock \emph{Vortex Methods: Theory and Practice}.
\newblock Cambridge University Press, 2000.

\bibitem[Redmon et~al.(2016)Redmon, Divvala, Girshick, and Farhadi]{redmon2016yolo}
J.~Redmon, S.~Divvala, R.~Girshick, and A.~Farhadi.
\newblock {You only look once: unified, real-time object detection}.
\newblock In \emph{Proceedings of the IEEE Conference on Computer Vision and Pattern Recognitionn}, pages 779--788, 2016.

\bibitem[Lin et~al.(2017)Lin, Goyal, Girshick, He, and Doll{\'a}r]{lin2017focal}
T.~Lin, P.~Goyal, R.~Girshick, K.~He, and P.~Doll{\'a}r.
\newblock {Focal loss for dense object detection}.
\newblock \emph{IEEE Trans. Vis. Comput. Graph.}, pages 2980--2988, 2017.

\bibitem[Paszke et~al.(2019)Paszke, Gross, Massa, Lerer, Bradbury, Chanan, Killeen, Lin, Gimelshein, Antiga, et~al.]{paszke2019pytorch}
A.~Paszke, S.~Gross, F.~Massa, A.~Lerer, J.~Bradbury, G.~Chanan, T.~Killeen, Z.~Lin, N.~Gimelshein, L.~Antiga, et~al.
\newblock {Pytorch: An imperative style, high-performance deep learning library}.
\newblock In \emph{Advances in Neural Information Processing Systems}, pages 8026--8037, 2019.

\bibitem[Glorot and Bengio(2010)]{glorot2010understanding}
X.~Glorot and Y.~Bengio.
\newblock Understanding the difficulty of training deep feedforward neural networks.
\newblock In \emph{Proceedings of the Thirteenth International Conference on Artificial Intelligence and Statistics}, pages 249--256, 2010.

\bibitem[Qu et~al.(2019)Qu, Zhang, Gao, Jiang, and Chen]{Qu2019}
Z.~Qu, X.~Zhang, M.~Gao, C.~Jiang, and B.~Chen.
\newblock Efficient and conservative fluids using bidirectional mapping.
\newblock \emph{ACM Trans. Graph.}, 38:\penalty0 1--12, 2019.

\bibitem[Sod(1978)]{sod1978survey}
G.~A. Sod.
\newblock A survey of several finite difference methods for systems of nonlinear hyperbolic conservation laws.
\newblock \emph{J. Comput. Phys.}, 27:\penalty0 1--31, 1978.

\bibitem[Lu et~al.(2019)Lu, Meng, Mao, and Karniadakis]{lu2019DeepXDE}
L.~Lu, X.~Meng, Z.~Mao, and G.~E. Karniadakis.
\newblock {DeepXDE: A deep learning library for solving differential equations}.
\newblock \emph{SIAM Rev. Soc. Ind. Appl. Math.}, 63:\penalty0 208--228, 2019.

\bibitem[Rogallo(1981)]{Rogallo1981}
R.~S. Rogallo.
\newblock {Numerical experiments in homogeneous turbulence}.
\newblock In \emph{Technical Report TM81315, NASA}, 1981.

\bibitem[Xiong and Yang({2019})]{Xiong2019}
S.~Xiong and Y.~Yang.
\newblock {Construction of knotted vortex tubes with the writhe-dependent helicity}.
\newblock \emph{{Phys. Fluids}}, {31}:\penalty0 {047101}, {2019}.

\bibitem[Xiong and Yang(2020)]{Xiong2020}
S.~Xiong and Y.~Yang.
\newblock {Effects of twist on the evolution of knotted magnetic flux tubes}.
\newblock \emph{J. Fluid Mech.}, 895:\penalty0 A28, 2020.

\bibitem[Yang et~al.(2010)Yang, Pullin, and Bermejo-Moreno]{Yang2010a}
Y.~Yang, D.~I. Pullin, and I.~Bermejo-Moreno.
\newblock {Multi-scale geometric analysis of Lagrangian structures in isotropic turbulence}.
\newblock \emph{J. Fluid Mech.}, 654:\penalty0 233--270, 2010.

\bibitem[Yang and Pullin(2011)]{Yang2011a}
Y.~Yang and D.~I. Pullin.
\newblock {Geometric study of Lagrangian and Eulerian structures in turbulent channel flow}.
\newblock \emph{J. Fluid Mech.}, 674:\penalty0 67--92, 2011.

\bibitem[Zhao et~al.(2016)Zhao, Yang, and Chen]{ZhaoYangChen2016a}
Y.~Zhao, Y.~Yang, and S.~Chen.
\newblock {Evolution of material surfaces in the temporal transition in channel flow}.
\newblock \emph{J. Fluid Mech.}, 793:\penalty0 840--876, 2016.

\bibitem[Zheng et~al.(2016)Zheng, Yang, and Chen]{ZhengYangChen2016}
W.~Zheng, Y.~Yang, and S.~Chen.
\newblock {Evolutionary geometry of Lagrangian structures in a transitional boundary layer}.
\newblock \emph{Phys. Fluids}, 28:\penalty0 035110, 2016.

\bibitem[Zheng et~al.(2019)Zheng, Ruan, Yang, He, and Chen]{Zheng2019}
W.~Zheng, S.~Ruan, Y.~Yang, L.~He, and S.~Chen.
\newblock {Image-based modelling of the skin-friction coefficient in compressible boundary-layer transition}.
\newblock \emph{J. Fluid. Mech.}, 875:\penalty0 1175--1203, 2019.

\bibitem[Lundgren(1982)]{Lundgren1982}
T.~S. Lundgren.
\newblock Strained spiral vortex model for turbulent fine structure.
\newblock \emph{Phys. Fluids}, 25:\penalty0 2193--2203, 1982.

\bibitem[Lundgren({1993})]{Lundgren1993}
T.~S. Lundgren.
\newblock {A small-scale turbulence model}.
\newblock \emph{{Phys. Fluids A}}, {5}:\penalty0 {1472}, {1993}.

\end{thebibliography}

\end{document}